%% file: Arxiv.tex
\documentclass[11pt]{article}
\usepackage[top=2.54cm,left=2.54cm,right=2.54cm,bottom=2.54cm]{geometry}

\input{math_commands.tex}
\usepackage[authoryear,round]{natbib}
\usepackage[hidelinks]{hyperref}
\usepackage{url}
\usepackage{amsmath,amssymb,amsthm,mathtools,bm}
\usepackage{microtype}
\usepackage[nameinlink,capitalise,noabbrev]{cleveref}
\usepackage{bbm}
\usepackage{algorithm}
\usepackage{algorithmic}
\usepackage{subcaption}

\newtheorem{theorem}{Theorem}[section]
\newtheorem{lemma}[theorem]{Lemma}
\newtheorem{proposition}[theorem]{Proposition}
\newtheorem{corollary}[theorem]{Corollary}
\theoremstyle{definition}
\newtheorem{definition}[theorem]{Definition}
\theoremstyle{remark}
\newtheorem{remark}[theorem]{Remark}

\newcommand{\PP}{\mathbb{P}}

\newcommand{\cE}{\mathcal E}
\newcommand{\cK}{\mathcal K}
\newcommand{\cJ}{\mathcal J}
\newcommand{\cS}{\mathcal S}
\newcommand{\cA}{\mathcal A}

\newcommand{\dmin}{\Delta_{\min}}

\newcommand{\Ab}{\mathbf{A}}
\newcommand{\Bb}{\mathbf{B}}
\newcommand{\xb}{\mathbf{x}}
\renewcommand{\ab}{\mathbf{a}}
\newcommand{\bsigma}{\bm{\sigma}}
\newcommand{\bSigma}{\bm{\Sigma}}
\newcommand{\Ib}{\mathbf{I}}
\newcommand{\sah}{\mathcal{S} \times \mathcal{A} \times [H]}
\newcommand{\wb}{\mathbf{w}}
\newcommand{\bphi}{\bm{\phi}}

\newcommand{\btheta}{\bm{\theta}}
\newcommand{\VV}{\mathbb{V}}

\title{Understanding Gap-Dependent Regret for Optimism-Based Reinforcement Learning with Linear Function Approximation}

\author{
Haochen Zhang$^{1}$, Zhong Zheng$^{2}$, and Lingzhou Xue$^{1}$\\[0.5em]
$^{1}$Department of Statistics, The Pennsylvania State University\\
$^{2}$Department of Applied and Computational Mathematics and Statistics,\\ University of Notre Dame
}

\date{}

\begin{document}

\maketitle
\begin{abstract}
We study gap-dependent regret for reinforcement learning with linear function approximation. While prior works have established gap-dependent guarantees in this setting, existing analyses do not apply to algorithms that achieve the nearly minimax-optimal worst-case regret bound $\tilde{O}(d\sqrt{H^3K})$, where $d$ is the feature dimension, $H$ is the horizon length, and $K$ is the number of episodes. We bridge this gap by establishing the first gap-dependent regret bound for the nearly minimax-optimal algorithm LSVI-UCB++ \citep{he2023nearly}, with an expected regret bound $\tilde{O}(d^2H^3/\dmin+d^6H^5)$, improving the dependence on both $d$ and $H$ in the leading gap-dependent term compared with previous results. To understand the exploration cost induced by optimism, we establish a structural lower bound $\Omega(d^2H^3/\dmin)$ for a broad class of algorithms based on persistent ellipsoidal optimism. When specialized to LSVI-UCB++, this result shows that the leading dependence of our upper bound on $d$, $H$, and $\dmin$ is tight up to logarithmic factors. Beyond this algorithmic class, we establish a general gap-dependent lower bound $\Omega(dH^3\log K/\dmin)$ for arbitrary learning algorithms, showing that the logarithmic dependence on $K$ and the cubic dependence on $H$ are intrinsic to gap-dependent expected regret in linear MDPs. Together, our results substantially narrow the gap between upper and lower bounds and provide a sharper characterization of gap-dependent learning and optimism-based exploration with linear function approximation.
\end{abstract}

\section{Introduction}

Reinforcement learning (RL) \citep{1998Reinforcement} provides a framework for sequential decision-making, where an agent tries to maximize cumulative rewards through interaction with an unknown environment. A central challenge is designing efficient algorithms for large state and action spaces. Linear function approximation addresses this by representing value functions using known features.

Recently, a large body of literature has studied regret upper bounds for RL with linear function approximation. These works can be categorized into two main classes. The first class considers the linear Markov decision process (MDP) \citep{jin2020provably}, where the transition dynamics are assumed to be linear in the known features. In particular, \citet{jin2020provably} proposed the first provably efficient algorithm, LSVI-UCB, which achieves the regret bound $\tilde{O}(\sqrt{d^3 H^4 K})$. Here, $\tilde{O}(\cdot)$ hides logarithmic factors, $d$ denotes the feature dimension, $H$ is the horizon length, and $K$ is the number of episodes. Subsequently, \citet{he2023nearly} proposed the LSVI-UCB++ algorithm, which leverages adaptive weighted ridge regression and pessimistic value estimation and improves the regret bound to the nearly minimax-optimal rate $\tilde{O}(d\sqrt{H^3 K})$ \citep{zhou2021nearly}. The second class focuses on linear mixture MDPs \citep{ayoub2020model,zhou2021nearly,zhou2022computationally}, where the transition probability is modeled as a linear combination of multiple base models. \citet{ayoub2020model} proposed the UCRL-VTR algorithm with a regret bound $\tilde{O}(d\sqrt{H^4K})$. Later, \citet{zhou2021nearly} developed a near-optimal algorithm for time-inhomogeneous linear mixture MDPs. Furthermore, \citet{zhou2022computationally} proposed a near-optimal, horizon-free algorithm for time-homogeneous linear mixture MDPs.

In practice, RL algorithms often outperform their worst-case guarantees when a positive suboptimality gap exists (i.e., the optimal action at each state is better than suboptimal actions by a non-negligible margin). Gap-dependent analysis has been well studied in the tabular setting \citep{simchowitz2019non,dann2021beyond,yang2021q,xu2021fine,zheng2024gap,zhang2025regret,zhang2025q,chen2026sharp}, and it has also gained traction in RL with linear function approximation. \citet{he2021logarithmic} pioneered this direction, showing that LSVI-UCB \citep{jin2020provably} achieves an expected regret bound $\tilde{O}(d^3 H^5/\dmin)$ under linear MDPs, and UCRL-VTR \citep{ayoub2020model} attains an expected regret bound $\tilde{O}(d^2 H^5/\dmin)$ under linear mixture MDPs, where the minimum gap $\dmin$ is defined as the infimum of the positive suboptimality gap $\Delta_h(s,a)$ over all state-action-step triples $(s,a,h)$. \citet{papini2021reinforcement} established a high-probability constant-regret guarantee under the UNISOFT condition, while \citet{zhang2024achieving} obtained a high-probability constant-regret guarantee without this assumption. Both results imply expected regret of order $\tilde{O}(d^3 H^5/\dmin)$. Taken together, these results show that RL with linear function approximation can achieve polylogarithmic regret.

Despite these advances, existing gap-dependent analyses do not yet cover algorithms that achieve the nearly minimax-optimal worst-case regret bound $\tilde{O}(d\sqrt{H^3K})$, e.g., LSVI-UCB++ \citep{he2023nearly} or UCRL-VTR+ \citep{zhou2021nearly}. As a result, current gap-dependent regret bounds exhibit a loose dependence on the feature dimension $d$ and horizon length $H$, i.e., $\tilde{O}(d^3 H^5/\dmin)$ under linear MDPs or $\tilde{O}(d^2 H^5/\dmin)$ under linear mixture MDPs, failing to reflect the superior worst-case efficiency of nearly minimax-optimal algorithms. Moreover, most existing algorithms with linear function approximation, including LSVI-UCB and UCRL-VTR, rely on optimism-based exploration, which encourages exploration through uncertainty-dependent bonuses. However, it remains unclear how much gap-dependent regret is intrinsically induced by such optimism-based exploration mechanisms and whether a nearly minimax-optimal optimism-based algorithm can attain a matching upper bound. Consequently, these gaps lead to the following open question:

\begin{center}
\textit{What gap-dependent regret is unavoidable under optimism-based exploration, and can a nearly minimax-optimal optimism-based algorithm achieve a matching upper bound?}
\end{center}

Improving the dependence on feature dimension $d$ and horizon length $H$ in regret bounds is not only of theoretical interest but also has practical implications. In domains such as robotics and healthcare, agents operate in high-dimensional environments, where function approximation is essential for learning and control. For instance, in robotic manipulation and grasping tasks, the state space consists of high-dimensional continuous observations, including robot joint configurations, object positions, and environmental variables \citep{toner2023opportunities}. Similarly, healthcare applications such as optimizing cell growth conditions or sequential treatment planning involve biological state representations and long decision sequences, making function approximation critical for deployment \citep{al2024reinforcement}. In these applications, improved dependence on $d$ and $H$ provides a sharper characterization of the sample efficiency of RL algorithms in high-dimensional, long-horizon tasks.

In this paper, we answer the question above by establishing complementary upper and lower bounds. On the upper-bound side, we establish an improved gap-dependent regret bound for the nearly minimax-optimal LSVI-UCB++. Our bound achieves the best-known dependence on $d$ and $H$ in the leading gap-dependent term for linear function approximation. On the lower-bound side, we first study persistent ellipsoidal optimism, where an uncertainty-dependent ellipsoidal bonus remains active throughout learning, and establish a structural lower bound for this exploration mechanism. We then specialize the result to LSVI-UCB++, showing that the leading dependence of our upper bound on $d$, $H$, and $\dmin$ is tight up to logarithmic factors. We also establish a general lower bound characterizing the fundamental difficulty of gap-dependent learning in linear MDPs. Together, these results narrow the gap between upper and lower bounds and provide a sharper understanding of gap-dependent learning and optimism-based exploration with linear function approximation.

\textbf{Our contributions.} Our contributions are summarized as follows:

(i) \textbf{Improved gap-dependent regret bound and a novel random-suffix framework:}
We establish the first gap-dependent regret bound for the nearly minimax-optimal LSVI-UCB++ algorithm with linear function approximation. Specifically, we prove an expected regret upper bound of $\tilde{O}(d^2H^3/\dmin+d^6H^5)$. Compared with existing gap-dependent guarantees, our leading term $\tilde{O}(d^2H^3/\dmin)$ substantially improves the dependence on both the feature dimension $d$ and the horizon length $H$ (see \Cref{tab:comparison}). In particular, when $1/\dmin=\tilde{\Omega}(d^4)$, our bound improves on all existing gap-dependent regret guarantees in Table~\ref{tab:comparison}. To obtain this bound, we develop a novel random-suffix framework for gap-dependent regret analysis in linear MDPs, presented in \Cref{sec:proof-sketch}. The framework localizes regret and variance analysis to random trajectory suffixes beginning at the first realized positive-gap step. Our numerical experiments in Appendix~\ref{app:experiment} suggest approximately logarithmic regret growth of LSVI-UCB++ and show lower empirical regret than the baselines considered.

(ii) \textbf{A structural lower bound for persistent ellipsoidal optimism and a matching lower bound for LSVI-UCB++:} To understand whether the leading dependence of our upper bound can be improved, we study a broad class of algorithms based on persistent ellipsoidal optimism defined in Definition~\ref{def:persistent-ellipsoid}. We establish a structural lower bound for this algorithmic class in \Cref{ellipsoidlower}, showing that, for sufficiently large episode budgets, every algorithm in the class can incur expected regret of order $\Omega(d^2H^3/\dmin)$. We further verify in Appendix \ref{subsec:persistent-lsvi-ucbpp} that LSVI-UCB++ satisfies the required structural conditions, yielding an $\Omega(d^2H^3/\dmin)$ lower bound that matches the leading $\tilde{O}(d^2H^3/\dmin)$ term of our upper bound up to logarithmic factors.

(iii) \textbf{Improved general gap-dependent regret lower bound:}
We also establish an improved gap-dependent expected regret lower bound $\Omega\left(dH^3\log K/\dmin\right)$ for arbitrary learning algorithms in linear MDPs, as shown in \Cref{lower}. Compared with the prior lower bound $\Omega(dH/\dmin)$ \citep{he2021logarithmic}, our result shows that the logarithmic dependence on the number of episodes $K$ in the expected regret is unavoidable. More importantly, it reveals an intrinsic horizon dependence of order $H^3$. This provides a stronger benchmark for gap-dependent regret in linear MDPs and shows that the cubic dependence on $H$ in our upper bound is unavoidable even beyond optimism-based algorithms. The remaining gap between this general lower bound and the leading gap-dependent term of our upper bound is therefore a factor of $d$, up to logarithmic factors.
\begin{table}[ht]
\centering
\fontsize{10pt}{14pt}\selectfont
\begin{tabular}{|l|l|}
\hline
\multicolumn{1}{|c|}{\textbf{Algorithm}} &
\multicolumn{1}{c|}{\textbf{$\mathbb{E}[\text{Regret}(T)]$}} \\
\hline
LSVI-UCB & $\tilde{O}(d^3 H^5 / \dmin)$ \citep{he2021logarithmic} \\
\hline
UCRL-VTR & $\tilde{O}(d^2 H^5 / \dmin)$ \citep{he2021logarithmic}\\
\hline
LSVI-UCB & $\tilde{O}(d^3 H^5 / \dmin)$ \citep{papini2021reinforcement}\\
\hline
Cert-LSVI-UCB & $\tilde{O}(d^3 H^5 / \dmin)$ \citep{zhang2024achieving}\\
\hline
LSVI-UCB++ & $\tilde{O}(d^2 H^3/ \dmin +d^6H^5 )$ (Ours) \\
\hline
\end{tabular}
\caption{Comparison of gap-dependent expected regret bounds for different RL algorithms with linear function approximation. We use $T$ to denote the total number of steps, $d$ as the feature dimension, $H$ as the horizon length, and $\dmin$ as the minimum gap. The notation $\tilde{O}(\cdot)$ hides logarithmic factors.}
\label{tab:comparison}
\end{table}

(iv) \textbf{Technical novelty:}
Our three main results rely on distinct technical ideas.

\textbf{Upper bound.}
The proof of \Cref{regret} develops a first-positive-gap and random-suffix analysis. Lemma~\ref{expectedrmain} reduces the expected regret to the first realized positive-gap step of each suboptimal episode. To analyze the resulting random suffixes, Lemma~\ref{sumbonus} introduces a surrogate covariance sequence for partial bonus sums, while Lemmas~\ref{lemma:transition} and \ref{lemma:transition1} control the associated value-estimation errors. These tools lead to the random-suffix variance bound in Lemma~\ref{LEMMA: TOTAL-ESTIMATE-VARIANCE} and the self-bounding argument in Lemma~\ref{lemma:global-error}, yielding the leading $\tilde O(d^2H^3/\dmin)$ regret.

\textbf{Structural lower bound for persistent ellipsoidal optimism.}
For \Cref{ellipsoidlower}, we construct a hard linear MDP with action-specific unknown transition directions across $\Theta(H)$ steps (Appendix \ref{subsec:persistent-hard-family}). Full-domain Bellman calibration bounds the total weighted squared regression error along these directions (Appendix \ref{subsec:persistent-joint-error}), forcing the learner to either repeatedly select suboptimal actions or accurately estimate many transition directions (Appendix \ref{subsec:persistent-exploration-estimation}). A finite-family testing argument then yields the $\Omega(d^2H^3/\dmin)$ regret lower bound (Appendix \ref{subsec:persistent-information-completion}).

\textbf{General lower bound.}
For \Cref{lower}, we construct a hard linear MDP with statistically difficult alternatives across feature directions and steps. In each pair, the same action is suboptimal under one alternative but optimal under the other. Sampling it sufficiently often incurs regret under the former, while insufficient sampling prevents reliable identification and incurs regret under the latter. A change-of-measure argument formalizes this tradeoff and yields the lower bound $\Omega(dH^3\log K/\dmin)$.

\section{Preliminaries}
\label{prelim}
\textbf{Notation.} For any $C \in \mathbb{N}_+$, we write $[C] := \{1, 2, \ldots, C\}$. We denote by $\mathbb{I}[x]$ the indicator function of the event $x$. For a vector $\xb \in \mathbb{R}^d$ and a positive definite matrix $\bSigma \in \mathbb{R}^{d \times d}$, we use $\|\xb\|_2$ to denote the Euclidean norm, and $\|\xb\|_{\bSigma} := \sqrt{\xb^\top \bSigma \xb}$. For any $a , b \in \R$ with $a < b$, $x \in \R$, let $[x]_{[a,b]}$ denote the truncation function $a\cdot \mathbb{I}[x \leq a] + x \cdot \mathbb{I}[a < x < b] + b \cdot \mathbb{I}[b \leq x]$.

\textbf{Episodic Markov decision processes.}
An episodic MDP is denoted by $\mathcal{M}:=(\mathcal{S}, \mathcal{A}, H, \mathbb{P}, r)$, where $\mathcal{S}$ is a possibly infinite measurable state space, $\mathcal{A}$ is the finite set of actions, $H$ is the number of steps in each episode, $\mathbb{P}:=\{\mathbb{P}_h\}_{h=1}^H$ is the transition kernel so that $\mathbb{P}_h(\cdot \mid s, a)$ characterizes the distribution over the next state given the state-action pair $(s,a)$ at step $h$, and $r:=\{r_h\}_{h=1}^H$ is the collection of reward functions. Following \citet{he2023nearly}, we assume that the reward functions
$r_h(s,a)\in[0,1]$ are deterministic and known to the learner for any $(s,a,h)\in\mathcal S\times\mathcal A\times[H]$.
	
In each episode, an adversary selects the initial state $s_1$. At each step $h \in[H]$, an agent observes a state $s_h \in \mathcal{S}$, picks an action $a_h \in \mathcal{A}$, receives the reward $r_h = r_h(s_h,a_h)$ and then transitions to the next state $s_{h+1}$. The episode terminates at the terminal state $s_{H+1}$ after $H$ steps.

For convenience, we denote $\mathbb { P }_{s,a,h}f = \mathbb{E}_{s'\sim \mathbb{P}_h(\cdot|s,a)}f(s')$, $\mathbbm { 1 }_s f = f(s)$ and  $\mathbb{V}_{s,a,h}(f) = \mathbb{P}_{s,a,h}f^2-(\mathbb{P}_{s,a,h}f)^2$ for any function $f: \mathcal{S} \rightarrow \mathbb{R}$ and state-action-step triple $(s,a,h)$.

\textbf{Policy and value functions.}
	A policy $\pi$ is a collection of $H$ functions $\{\pi_h: \mathcal{S} \rightarrow \Delta^\mathcal{A}\}_{h \in[H]}$, where $\Delta^\mathcal{A}$ is the set of probability distributions over $\mathcal{A}$. A policy is deterministic if for any $s\in\mathcal{S}$,  $\pi_h(s)$ concentrates all the probability mass on an action $a\in\mathcal{A}$. In this case, we denote $\pi_h(s) = a$. 

 Denote state value functions by $V_h^\pi(s):=\sum_{h^{\prime}=h}^H \mathbb{E}_{(s_{h^{\prime}},a_{h^{\prime}})\sim(\mathbb{P}, \pi)}\left[r_{h^{\prime}}(s_{h^{\prime}},a_{h^{\prime}}) \left. \right\vert s_h = s\right]$
and action value functions by $Q_h^\pi(s, a):=r_h(s,a)+\sum_{h^{\prime}=h+1}^H\mathbb{E}_{(s_{{h^\prime}},a_{h^\prime})\sim\left(\mathbb{P},\pi\right)}\left[ r_{h^{\prime}}(s_{h^{\prime}},a_{h^{\prime}}) \left. \right\vert s_h=s, a_h=a\right]$. Since the action space and the horizon are finite, there exists an optimal policy $\pi^{\star}$ that achieves the optimal value $V_h^{\star}(s)=\sup _\pi V_h^\pi(s)=V_h^{\pi^\star}(s)$ for all $(s,h) \in \mathcal{S} \times [H]$  \citep{azar2017minimax}. Then for any $(s,a,h) \in \mathcal{S}\times\mathcal{A}\times [H]$, the Bellman equation can be expressed as ($V_{H+1}^{\pi}(s)=V_{H+1}^{\star}(s)=0$):
	\begin{equation}\label{eq_Bellman}
		\left\{\begin{array} { l } 
			{ V _ { h } ^ { \pi } ( s ) = \mathbb{E}_{a'\sim \pi _ { h } ( s )}[Q _ { h } ^ { \pi } ( s , a' ) }] \\
			{ Q _ { h } ^ { \pi } ( s , a ) : =  r _ { h }(s,a) + \mathbb { P } _ { s,a,h } V _ { h + 1 } ^ { \pi } }
		\end{array}  \text { and } \left\{\begin{array}{l}
			V_h^{\star}(s)=\max _{a' \in \mathcal{A}} Q_h^{\star}(s, a') \\
			Q_h^{\star}(s, a):=r_h(s,a)+\mathbb{P}_{s,a,h} V_{h+1}^{\star}.
		\end{array}\right.\right.
	\end{equation}
For any algorithm over $K$ episodes, let $\pi^{k}$ be the policy used in the $k$-th episode, and $s_1^{k}$ be the corresponding initial state. The regret  over $T=HK$ steps is $\mbox{Regret}(T) := \sum_{k=1}^{K}(V_1^\star - V_1^{\pi^{k}})(s_1^{k}).$

\textbf{Suboptimality gap.} For any given MDP, we define the suboptimality gap as follows.
\begin{definition}\label{def_sub}
    For any $(s,a,h)$, the suboptimality gap is defined as $\Delta_h(s,a) := V_h^\star(s) - Q_h^\star(s,a).$
\end{definition}
\Cref{eq_Bellman} implies that for any $ (s,a,h)$, $\Delta_h(s,a) \geq 0$. Then, it is natural to define the minimum gap, which is the infimum of the strictly positive suboptimality gaps.
\begin{definition}\label{def_minsub}
    We define the \textbf{minimum gap} as $\dmin:=\inf_{(s,a,h) \in\mathcal{S} \times \mathcal{A} \times[H],\Delta_h(s,a)>0}\Delta_h(s,a)$.
\end{definition}
We assume that $\dmin > 0$ in this paper. Definitions \ref{def_sub} and \ref{def_minsub} and the assumption of $\dmin > 0$ are standard in the literature \citep{simchowitz2019non, dann2021beyond, yang2021q, xu2021fine, he2021logarithmic,zhang2024achieving,zheng2024gap,zhang2025q,zhang2025regret}.

\textbf{Linear Markov decision processes.} In this work, we focus on the linear Markov decision process \citep{jin2020provably,he2021logarithmic,he2023nearly}, which is formally defined as follows:
\begin{definition}\label{assumption:linear-MDP}
An episodic MDP $\mathcal{M}:=(\mathcal{S}, \mathcal{A}, H, \mathbb{P}, r)$ is a linear MDP with a known feature mapping $\bphi: \cS \times \cA \rightarrow \R^d$, if for any $h \in [H]$, there exist $d$ unknown signed measures $\btheta_h = (\btheta_h^1,...,\btheta_h^d)$ over $\mathcal{S}$, such that for any $(s,a) \in \cS \times \cA$, we have
$\PP_h(\cdot|s,a) = \left\langle \bphi(s,a), \btheta_h(\cdot)\right\rangle.$
\end{definition}
For simplicity, we assume that $\|\phi(s,a)\|_2 \le 1$ for all $(s,a)\in\mathcal S\times\mathcal A$ and
$\|\btheta_h(B)\|_2 \le \sqrt d$ for all measurable $B\subseteq\mathcal S$ and $h\in[H]$.
For linear MDPs, we have the following property:
\begin{proposition}[Proposition 3.3 of \cite{he2023nearly}]\label{prop:linearq}
For any policy $\pi$, there exist $\{\wb_h^\pi\}_{h=1}^H$ such that for any state-action-step triple $(s,a,h)\in \cS \times \cA \times[H]$, we have $\PP_{s,a,h} V_{h+1}^\pi = \langle \bphi(s,a), \wb_h^\pi\rangle$. 
\end{proposition}

\section{Theoretical Guarantees}
\subsection{Algorithm Review}
We begin by reviewing a slightly modified version of LSVI-UCB++ \citep{he2023nearly} (see Appendix D for the two modifications) in \Cref{algorithm1}. LSVI-UCB++ reduces the problem of learning the optimal value functions to a sequence of linear regression problems. Based on Proposition \ref{prop:linearq}, \Cref{algorithm1} constructs the estimator $\hat{\wb}_{k,h}$ by solving the following weighted ridge regression
\begin{align}
    \hat{\wb}_{k,h}&= \argmin\nolimits_{\wb\in \R^d}\lambda\|\wb\|_2^2 +\sum\nolimits_{i=1}^{k-1}\bar{\sigma}_{i,h}^{-2}\left(\wb^{\top}\bphi(s_h^i,a_h^i)-V_{k,h+1}(s_{h+1}^i)\right)^2.\notag
\end{align} 
Here, the adjusted estimated variance $\bar{\sigma}_{k,h}^2$ is set as $\bar\sigma_{k,h}^2= \max\{\sigma_{k,h}^2, H,2d^3H^2\|\bphi(s_h^k,a_h^k)\|_{\bSigma_{k,h}^{-1}}\}$, where the estimated variance $\sigma_{k,h}^2$ is defined in \Cref{eq:variance} below.
 
With exploration bonuses, Line 11 of \Cref{algorithm1} constructs an optimistic estimate $Q_{k,h}(s,a)$ as
\begin{align*}
     Q_{k,h}(s,a) \approx r_h(s,a)+\hat{\wb}_{k,h}^{\top}\bphi(s,a)+\beta\|\bphi(s,a)\|_{\bSigma_{k,h}^{-1}}.
 \end{align*}
Here, $Q_{k,h}(s,a)$ is an upper bound on $Q_h^\star(s,a)$ with high probability when the hyperparameter $\beta$ is chosen as $\tilde{\Theta}(\sqrt{d})$. The remaining components in Line 11 ensure its boundedness and monotonicity.

The algorithm also uses a pessimistic estimate $\check{Q}_{k,h}(s,a)$, which serves as a lower bound for $Q_{h}^\star(s,a)$ with high probability. By \Cref{eq_Bellman} and Line 17 of \Cref{algorithm1}, this estimate allows us to control the value estimation error $V_{k,h}(s) - V_h^\star(s)$ via the error $V_{k,h}(s) - \check{V}_{k,h}(s)$. The pessimistic estimate $\check{Q}_{k,h}(s,a)$ is constructed in a manner analogous to the optimistic one. We first obtain the vector $\check{\wb}_{k,h}$ by solving the following weighted ridge regression:
\begin{align}
    \check{\wb}_{k,h}&= \argmin\nolimits_{\wb\in \R^d}\lambda\|\wb\|_2^2+\sum\nolimits_{i=1}^{k-1}\bar{\sigma}_{i,h}^{-2}\left(\wb^{\top}\bphi(s_h^i,a_h^i)-\check{V}_{k,h+1}(s_{h+1}^i)\right)^2.\notag
\end{align} 
We then compute $\check{Q}_{k,h}$ as follows:
$\check{Q}_{k,h}(s,a) \approx r_h(s,a)+\check{\wb}_{k,h}^{\top}\bphi(s,a)-\bar{\beta}\|\bphi(s,a)\|_{\bSigma_{k,h}^{-1}},$ where the hyperparameter $\bar{\beta}$ can be chosen as $\tilde{\Theta}(\sqrt{d^3H^2})$.
Finally, LSVI-UCB++ constructs $\sigma_{k,h}^2$ as follows:
\begin{align}
\sigma_{k,h}^2=\bar{\VV}_{s_h^k,a_h^k,h}V_{k,h+1}+E_{k,h}+D_{k,h}+H.\label{eq:variance}
\end{align}
In (\ref{eq:variance}), $\bar{\VV}_{s_h^k,a_h^k,h}V_{k,h+1}$ denotes the estimated variance of the value function and is defined as
\begin{align*}
    \left[\tilde{\wb}_{k,h}^{\top}\bphi(s_h^k,a_h^k)\right]_{[0,H^2]} - \left(\left[\hat{\wb}_{k,h}^{\top}\bphi(s_h^k,a_h^k)\right]_{[0,H]}\right)^2.
\end{align*}
Here,
\begin{align*}
\tilde{\wb}_{k,h}&:=\argmin\nolimits_{\wb\in \R^d}\lambda\|\wb\|_2^2 +\sum\nolimits_{i=1}^{k-1}\bar{\sigma}_{i,h}^{-2}\left(\wb^{\top}\bphi(s_h^i,a_h^i)-V^2_{k,h+1}(s_{h+1}^i)\right)^2
\end{align*}
is the solution to the weighted ridge regression problem for the squared value function. In addition,
\begin{align*}
    E_{k,h}&=\min \big\{\tilde{\beta}\|\bphi(s_h^k,a_h^k)\|_{\bSigma_{k,h}^{-1}},H^2\big\} +\min \big\{2H\bar \beta \|\bphi(s_h^k,a_h^k)\|_{\bSigma_{k,h}^{-1}},H^2\big\},\\
    D_{k,h}&=\min\big\{4d^3H^2\big(\hat{\wb}_{k,h}^{\top}\bphi(s_h^k,a_h^k)-\check{\wb}_{k,h}^{\top}\bphi(s_h^k,a_h^k) +2\bar{\beta}\|\bphi(s_h^k,a_h^k)\|_{\bSigma_{k,h}^{-1}}\big),d^3H^3\big\}.
\end{align*}
Here, $E_{k,h}$ bounds the error between the estimated variance $\bar{\VV}_{s_h^k,a_h^k,h}V_{k,h+1}$ and the true variance $\VV_{s_h^k,a_h^k,h}V_{k,h+1}$ of $V_{k,h+1}$, and $D_{k,h}$ bounds the error between  $\VV_{s_h^k,a_h^k,h}V_{k,h+1}$ and $\VV_{s_h^k,a_h^k,h}V_{h+1}^\star$. 

\subsection{Gap-Dependent Regret Upper Bound}

We now present our main result: the first gap-dependent regret upper bound for a nearly minimax-optimal algorithm in RL with linear function approximation. In \Cref{algorithm1}, for any confidence parameter $\delta \in (0,1)$, define the confidence radii $\beta = \beta(\delta)$, $\bar \beta = \bar\beta(\delta)$, and $\tilde\beta=\tilde\beta(\delta)$ as
\begin{align}
&\beta(\delta)=\Theta\Big(H\sqrt{d\lambda}+\sqrt{d \log^2\left(1+dT/(\delta\lambda)\right)}\Big),\bar{\beta}(\delta)=\Theta\Big(H\sqrt{d\lambda} +\sqrt{d^3H^2\log^2\left(dT/(\delta\lambda)\right)}\Big),\notag \\
&\tilde{\beta}(\delta)= \Theta \Big(H^2\sqrt{d\lambda} +\sqrt{d^3H^4\log^2\left(dT/(\delta\lambda)\right)}\Big),
\label{radii}
\end{align}
where the hidden constants are chosen to be sufficiently large universal constants.
\begin{theorem}
\label{regret}
For any linear MDP, if we set $\lambda =1/H^2$ and the confidence radii as specified in \Cref{radii} with $\delta = 1/(19T)$, then the expected regret of Algorithm~1 satisfies
\begin{align*}
\mathbb{E} \left[ \operatorname{Regret}(T) \right]
\leq
O\left(d^2H^3 \iota_2^3/\dmin +d^6H^5\iota_2^3\right),
\end{align*}
where $\iota_2 = \log\left(1+dHK\right)$.
\end{theorem}
\begin{algorithm}[H]
    \caption{LSVI-UCB++}
    \begin{algorithmic}[1]\label{algorithm1}
   \REQUIRE Regularization parameter $\lambda>0$, confidence radii
   $\beta, \bar\beta,\tilde{\beta}>0$. Episode number $K\in\mathbb{N}_+$.
   \STATE Initialize $k_{\text{mid}},k_\text{last} \leftarrow 0$, and for each step $h\in[H]$, set $\bSigma_{0,h},\bSigma_{1,h}= \lambda \Ib_{d}$. 
   \STATE For each step $h\in[H]$ and state-action $(s,a)\in \cS \times \cA$, set $ Q_{0,h}(s,a)= H, \check{Q}_{0,h}(s,a) = 0$.
\FOR{episode $k=1,\ldots,K$}
\STATE Receive the initial state $s_1^k$. 
\STATE Set $V_{k,H+1}(s)=\check V_{k,H+1}(s)=0$ for all $s\in\mathcal S$.
    \FOR{step $h=H,\ldots,1$}       
        \STATE 
    $\hat{\wb}_{k,h}=\bSigma_{k,h}^{-1}\sum_{i=1}^{k-1}\bar\sigma_{i,h}^{-2}\bphi(s_h^i,a_h^i)V_{k,h+1}(s_{h+1}^{i})$. 
        \STATE $\check{\wb}_{k,h}=\bSigma_{k,h}^{-1}\sum_{i=1}^{k-1}\bar\sigma_{i,h}^{-2}\bphi(s_h^i,a_h^i)\check{V}_{k,h+1}(s_{h+1}^{i})$.
        \STATE $\widetilde{\mathbf{w}}_{k,h}
= \mathbf{\Sigma}_{k,h}^{-1}
\sum_{i=1}^{k-1}\bar{\sigma}_{i,h}^{-2}
\boldsymbol{\phi}(s_h^i,a_h^i)V_{k,h+1}^2(s_{h+1}^i).$
        \STATE {\textbf{if} there exists $h'\in[H]$ such that $\det(\bSigma_{k,h'})\ge 2\det(\bSigma_{k_{\text{last}},h'})$, \textbf{then} for any $(s,a)\in\cS\times \cA$,}
        \STATE $\quad Q_{k,h}(s,a)=\min\big\{r_h(s,a)+\hat{\wb}_{k,h}^{\top}\bphi(s,a)+ \beta\|\bphi(s,a)\|_{\bSigma_{k,h}^{-1}} ,{Q}_{k-1,h}(s,a),H\big\}$, 
         \STATE $\quad \check{Q}_{k,h}(s,a)=\max\big\{r_h(s,a)+\check{\wb}_{k,h}^{\top}\bphi(s,a)- \bar{\beta}\|\bphi(s,a)\|_{\bSigma_{k,h}^{-1}} ,\check{Q}_{k-1,h}(s,a),0\big\}$.
        \STATE $\quad$Set $k_{\text{mid}} \leftarrow k$.
        \STATE \textbf{else}
        \STATE $\quad Q_{k,h}(s,a)=Q_{k-1,h}(s,a),\ \check{Q}_{k,h}(s,a)=\check{Q}_{k-1,h}(s,a), \forall (s,a)\in\cS\times \cA$.
        \STATE \textbf{end if}
        \STATE $V_{k,h}(s)=\max_{a}Q_{k,h}(s,a),\ \check{V}_{k,h}(s)=\max_{a}\check{Q}_{k,h}(s,a), \forall s\in\cS$.
    \ENDFOR
    \STATE Set the last updating episode $k_{\text{last}} \leftarrow k_{\text{mid}}$.
    \FOR{step $h=1,\ldots,H$}
    \STATE Take $a_h^k =  \pi_h^k(s_h^k) = \argmax_{a} Q_{k,h}(s_h^k,a)$, with ties broken by the smallest indexed action.\\
    \STATE Set the estimated variance $\sigma_{k,h}^2$ as in \Cref{eq:variance}. 
     \STATE $\bar\sigma_{k,h}^2= \max\big\{\sigma_{k,h}^2, H,2d^3H^2\|\bphi(s_h^k,a_h^k)\|_{\bSigma_{k,h}^{-1}}\big\}$.
    \STATE $\bSigma_{k+1,h}=\bSigma_{k,h}+\bar\sigma_{k,h}^{-2}\bphi(s_h^k,a_h^k)\bphi(s_h^k,a_h^k)^{\top}$. 
    \STATE Receive the next state $s_{h+1}^k$.
    \ENDFOR
\ENDFOR
    \end{algorithmic}
\end{algorithm}
The full proof is deferred to Appendix~\ref{regretproof}. Compared with prior gap-dependent expected regret upper bounds for linear function approximation \citep{he2021logarithmic,papini2021reinforcement,zhang2024achieving}, our result improves the dependence on both the feature dimension $d$ and the horizon length $H$ in the leading gap-dependent term. In particular, while existing bounds scale as $\tilde{O}(d^3H^5/\dmin)$ for LSVI-UCB or $\tilde{O}(d^2H^5/\dmin)$ for UCRL-VTR (see \Cref{tab:comparison}), our bound improves the leading gap-dependent term to $\tilde{O}(d^2H^3/\dmin)$. The experiments in Appendix~\ref{app:experiment} further show lower empirical regret than other baselines and suggest approximately logarithmic regret growth.

\subsection{A Structural Lower Bound for Persistent Ellipsoidal Optimism}
\label{sec:ellipsoidal-lower}

We next study the exploration cost induced by persistent ellipsoidal optimism. We establish a structural lower bound for this exploration mechanism and then specialize the result to LSVI-UCB++, obtaining a lower bound that matches the leading gap-dependent term in \Cref{regret} up to logarithmic factors.

We consider persistent ellipsoidal optimistic algorithms defined in Definition~\ref{def:persistent-ellipsoid}. These algorithms use predictably weighted ridge regression to estimate Bellman targets and form optimistic candidates by adding an ellipsoidal bonus $\beta\|\phi(s,a)\|_{\Lambda_{u,h}^{-1}}$ whenever updating value estimates at episode $u$. At such an episode, the step-$h$ covariance matrix is $\Lambda_{u,h}=\lambda I_d+\sum_{\ell<u}\alpha_{\ell,h}\phi(s_h^\ell,a_h^\ell)\phi(s_h^\ell,a_h^\ell)^\top$, where $\alpha_{\ell,h}>0$ is the weight of the step-$h$ sample from episode $\ell$. Episode $k$ uses the value estimates from the latest update, allowing for both every-episode and rare-switching updates.

The class captures a broad family of LSVI-style optimistic algorithms, allowing either unweighted or predictably weighted regression and either every-episode or rare-switching updates. Representative examples include LSVI-UCB and LSVI-UCB++, corresponding respectively to unweighted regression with every-episode updates and predictable variance-based weighting with rare switching. More generally, the framework covers other variants satisfying the same structural conditions, while the underlying combination of ridge regression and ellipsoidal optimism also appears in classical optimistic linear-bandit methods such as OFUL \citep{abbasi2011improved}.

We now quantify the exploration cost shared by this algorithmic class. Our structural lower bound shows that, for every algorithm in the class and a sufficiently large episode budget, there exists a linear MDP on which the algorithm incurs expected regret of at least order $d^2H^3/\dmin$.
\begin{theorem}[Informal]
\label{ellipsoidlower}
For any feature dimension $d\geq15$, horizon $H\geq16$, minimum gap $\dmin\in(0,H/16]$, and any persistent ellipsoidal optimistic algorithm satisfying Definition~\ref{def:persistent-ellipsoid}, if the episode number $K$ satisfies the sufficient-budget condition in \Cref{thm:persistent-structural-lower-bound}, then there exists a $d$-dimensional linear MDP with minimum gap $\dmin$ such that
\[
\mathbb E[\operatorname{Regret}(T)]
\geq
\Omega\!\left(\frac{d^2H^3}{\dmin}\right).
\]
\end{theorem}
The formal statement and complete proof are given in Theorem~\ref{thm:persistent-structural-lower-bound} and Appendix~\ref{sec: ellipsoidal-lower-bound}, respectively. The theorem requires the episode budget $K$ to exceed an explicit threshold that depends polynomially on the problem parameters and on the algorithmic quantities governing the regression weights, regularization, and confidence radius. The proof constructs a hard linear MDP with multiple action-specific unknown transition directions across $\Theta(H)$ steps (Appendix \ref{subsec:persistent-hard-family}). At each such step, the learner must either continue selecting sufficiently many suboptimal actions or estimate many unknown transition directions to sufficiently high accuracy (Appendix \ref{subsec:persistent-exploration-estimation}); a finite-family testing argument then converts this exploration--estimation tradeoff into the stated lower bound (Appendix \ref{subsec:persistent-information-completion}).

Theorem~\ref{ellipsoidlower} identifies the $d^2H^3/\dmin$ dependence as an unavoidable exploration cost for the considered class of persistent ellipsoidal optimistic algorithms. The quadratic dependence on $d$ is consistent with gap-dependent results for ellipsoidal optimism in linear bandits, such as OFUL \citep{abbasi2011improved}, whose leading inverse-gap term is also quadratic in $d$ up to logarithmic factors. In linear MDPs, this exploration cost propagates across the horizon, yielding the additional $H^3$ dependence. We further verify in \Cref{subsec:persistent-lsvi-ucbpp} that LSVI-UCB++ satisfies the required conditions, implying an $\Omega(d^2H^3/\dmin)$ lower bound for LSVI-UCB++ itself. This matches the leading $d^2H^3/\dmin$ term in our upper bound up to logarithmic factors.

\subsection{A General Gap-Dependent Regret Lower Bound}
\label{seclower}

Beyond the algorithm class considered above, we next establish a gap-dependent regret lower bound that applies to arbitrary learning algorithms. This result characterizes an intrinsic difficulty of gap-dependent learning in linear MDPs, independently of the particular exploration mechanism.
\begin{theorem}
\label{lower}
For any feature dimension $d\geq 7$, horizon $H\geq 8$, minimum gap $\dmin\in(0,H/4]$, episode number $K\geq(dH^3/\dmin^2)^2$, and any learning algorithm, there exists a $d$-dimensional linear MDP with minimum gap $\dmin$ such that
\[
\mathbb{E}[\operatorname{Regret}(T)]\geq cdH^3\log K/\dmin,\]
where $c>0$ is a universal constant.
\end{theorem}
The proof is deferred to Appendix~\ref{lowerproof}. Compared with the prior lower bound $\Omega(dH/\dmin)$ \citep{he2021logarithmic}, our result provides two substantial improvements. First, it establishes at finite $K$ that the logarithmic dependence on the number of episodes is unavoidable in gap-dependent expected regret. Second, it reveals an intrinsic cubic dependence on the horizon $H$. In particular, the $H^3$ dependence matches that of our upper bound in \Cref{regret}, showing that the improved horizon dependence achieved by LSVI-UCB++ is tight even beyond the class of optimism-based algorithms. The remaining gap between this general lower bound and the leading gap-dependent term of our upper bound is therefore a factor of $d$ in the feature dimension, up to logarithmic factors.

We also note that some prior works \citep{papini2021reinforcement,zhang2024achieving} establish high-probability gap-dependent regret bounds of order $\tilde{O}(d^3H^5\log(1/\delta)/\dmin)$ that do not explicitly depend on $\log K$. This does not contradict \Cref{lower}, since our result concerns expected regret. Converting such high-probability guarantees into expected regret typically requires choosing a confidence level $\delta$ that depends on $K$, thereby introducing an additional $\log K$ factor. For example, setting $\delta=1/K$  (see Remark~5.2 of \cite{zhang2024achieving}) yields $\mathbb{E}[\operatorname{Regret}(T)]\le \tilde{O}(d^3H^5\log K/\dmin).$

\section{Proof Sketch of \texorpdfstring{\Cref{regret}}{Theorem 3.1}}
\label{sec:proof-sketch}
In this section, we present the random-suffix framework underlying the proof of \Cref{regret}. The use of the first positive-gap step to localize regret has precedents in tabular RL \citep{dann2021beyond}. For linear MDPs, we develop a novel random-suffix analysis that jointly controls cumulative exploration bonuses and variance. For each episode in which a positive gap is realized, we take the first such step as the suffix boundary. We express the expected regret in terms of value differences at these boundaries and use the bonus and variance bounds to obtain the improved dependence on $H$.

For each episode $k$, define $\tau_k:=\inf\{h\in[H]:\Delta_h(s_h^k,a_h^k)>0\}$, with the convention $\tau_k=H+1$ if no positive gap is encountered. Let $\cK_\tau:=\{k\in[K]:\tau_k\leq H\}$ and $L:=|\cK_\tau|$. We further define $\cJ_\tau:=\{(k,h):k\in\cK_\tau,\tau_k\leq h\leq H\}$ and, for each $h\in[H]$, $\cK_h:=\{k\in[K]:\tau_k\leq h\}$.

We begin with a first-positive-gap decomposition that expresses the expected regret in terms of value differences at the random suffix boundaries.

\begin{lemma}
\label{lemma:sketch-first-gap}
For LSVI-UCB++ with $K$ episodes and $T=HK$ steps, the expected regret satisfies
\begin{align*}
\mathbb E[\operatorname{Regret}(T)]
=
\mathbb E\left[\sum\nolimits_{k\in\cK_\tau}
(V_{\tau_k}^\star-V_{\tau_k}^{\pi^k})(s_{\tau_k}^k)\right].
\end{align*}
\end{lemma}
The proof is given in Appendix~\ref{regretproof}. Since all realized gaps before $\tau_k$ are zero, the performance-difference identity implies that the expected cumulative gap over the remaining suffix equals $(V_{\tau_k}^\star-V_{\tau_k}^{\pi^k})(s_{\tau_k}^k)$. Thus, each episode with a realized positive-gap action is charged only once at its first such step, which forms the basis of our random-suffix analysis. 

The following lemma controls the cumulative optimistic error at these first positive-gap steps.
\begin{lemma}[Informal]
\label{lemma:sketch-global-error}
Let $\iota_1:=\log\!\left(1+dHK/\delta\right)$. Then with high probability,
\begin{align}
&\sum\nolimits_{k\in\cK_\tau}
\left(
Q_{k,\tau_k}(s_{\tau_k}^k,a_{\tau_k}^k)
-
Q_{\tau_k}^{\pi^k}(s_{\tau_k}^k,a_{\tau_k}^k)
\right)
\leq
O\left(
d^2H^3\iota_1^3/\dmin
+
d^6H^5\iota_1^3
\right).
\label{eq:sketch-global-final}
\end{align}
\end{lemma}
The formal statement is given in Lemma~\ref{lemma:global-error}. By optimism, the left-hand side upper bounds $\sum_{k\in\cK_\tau}(V_{\tau_k}^\star-V_{\tau_k}^{\pi^k})(s_{\tau_k}^k)$. Therefore, combining Lemmas~\ref{lemma:sketch-first-gap} and \ref{lemma:sketch-global-error}, together with the trivial bound $\operatorname{Regret}(T)\leq T$, gives \Cref{regret}. We next explain the key ideas behind Lemma~\ref{lemma:sketch-global-error}.

\subsection{Proof Sketch of \texorpdfstring{Lemma~\ref{lemma:sketch-global-error}}{Lemma 4.2}}
For simplicity, denote the left-hand side of \Cref{eq:sketch-global-final} by $X$. For every $k\in\cK_\tau$, the definition of $\tau_k$ gives $\Delta_{\tau_k}(s_{\tau_k}^k,a_{\tau_k}^k)\geq\dmin$. Since $\pi^k$ is greedy with respect to $Q_{k,h}$ and $Q_{k,h}$ is optimistic, we have $Q_{k,\tau_k}(s_{\tau_k}^k,a_{\tau_k}^k)-Q_{\tau_k}^{\pi^k}(s_{\tau_k}^k,a_{\tau_k}^k)\geq\dmin$. Summing over $k\in\cK_\tau$ therefore gives
\begin{align}
\dmin L\leq X.
\label{eq:sketch-global-lower}
\end{align}
We next derive an upper bound on $X$. By the Bellman equation and the confidence bound for the optimistic estimate, for every $(k,h)$ we have (here we define $Q_{k,H+1}(s,a):=Q_{H+1}^{\pi^k}(s,a):=0$)
\begin{align}
(Q_{k,h}&-Q_h^{\pi^k})(s_h^k,a_h^k)
\leq
(Q_{k,h+1}-Q_{h+1}^{\pi^k})(s_{h+1}^k,a_{h+1}^k)\notag\\
&\qquad+\big(\mathbb P_{s_h^k,a_h^k,h}-\mathbbm{1}_{s_{h+1}^k}\big)
(V_{k,h+1}-V_{h+1}^{\pi^k})
+4\min\{\beta\|\bphi(s_h^k,a_h^k)\|_{\bSigma_{k,h}^{-1}},H\}.
\label{eq:sketch-tele}
\end{align}
Unlike the usual full-trajectory analysis \citep{he2023nearly}, we sum \Cref{eq:sketch-tele} only from the random starting point $\tau_k$. Summing over $h=\tau_k,\ldots,H$ and $k\in\cK_\tau$ gives $X\leq M_\tau+B_\tau$, where $$M_\tau:=\sum_{(k,h)\in\cJ_\tau}(\mathbb P_{s_h^k,a_h^k,h}-\mathbbm{1}_{s_{h+1}^k})(V_{k,h+1}-V_{h+1}^{\pi^k})$$ is a martingale term and $$B_\tau:=\sum_{h=1}^H\sum_{k\in\cK_h}4\min\{\beta\|\bphi(s_h^k,a_h^k)\|_{\bSigma_{k,h}^{-1}},H\}$$ is the corresponding partial sum of bonuses. The martingale term can be bounded by $M_\tau\leq O(\sqrt{H^3L\log(KH/\delta)})$ with high probability. The main difficulty is therefore to control $B_\tau$.

The next lemma provides a partial-sum bonus bound that is specifically designed for the episode subsets $\cK_h$ generated by the first-positive-gap decomposition.

\begin{lemma}[Informal]
\label{lemma:sketch-sumbonus}
Set $\lambda=1/H^2$ and $\iota=\log(1+K/(d\lambda H))$. For any $h\in[H]$ and $\beta',C\geq1$, with high probability,
\begin{align*}
&\sum\nolimits_{k\in\cK_h}\min\{\beta'\|\bphi(s_h^k,a_h^k)\|_{\bSigma_{k,h}^{-1}},C\}\leq
10d^4H^3C\iota
+
10\beta'd^4H^2\iota
+
2\beta'\sqrt{d\iota\sum\nolimits_{k\in\cK_h}(\sigma_{k,h}^2+H)}.
\end{align*}
\end{lemma}

The formal statement and proof are given in Lemma~\ref{sumbonus} in the appendix. This lemma requires a new argument because, after restricting the episodes to $\cK_h$, the subsequence of covariance matrices $\{\bSigma_{k,h}:k\in\cK_h\}$ no longer satisfies the one-step recursive update needed for a standard elliptical-potential argument. To recover this structure, write $\cK_h=\{k_1<\cdots<k_{|\cK_h|}\}$ and construct a surrogate covariance sequence $\{\bSigma_i'\}$ using only these selected episodes and recursively defined surrogate variance weights. We show by induction that $\bSigma_i'\preceq\bSigma_{k_i,h}$ and $\bar\sigma'_{i,h}\geq\bar\sigma_{k_i,h}$. Consequently, each original bonus can be upper bounded by the corresponding surrogate bonus. Since the surrogate sequence again admits a one-step recursive update, the standard elliptical-potential argument becomes applicable. This surrogate construction is what allows the variance-weighted bonus analysis of LSVI-UCB++ to be localized to the random subsets $\cK_h$.

Applying Lemma~\ref{lemma:sketch-sumbonus} with $\beta'=\beta$ and $C=H$ at every step and using Cauchy--Schwarz gives $B_\tau\leq O\left(d^4H^5\iota+\beta d^4H^3\iota+\beta\sqrt{dH\iota\sum_{(k,h)\in\cJ_\tau}(\sigma_{k,h}^2+H)}\right)$. It therefore remains to control the estimated variances over the same random suffix set $\cJ_\tau$.

Our next key lemma develops a random-suffix analogue of the variance analysis used in the worst-case regret bound of LSVI-UCB++.
\begin{lemma}[Informal]
\label{lemma:sketch-total-variance}
With high probability,
\begin{align}
\sum\nolimits_{(k,h)\in\cJ_\tau}\sigma_{k,h}^2
\leq O\left(H^2L+d^{10}H^{9}\iota_1^3\right).
\label{eq:sketch-total-variance}
\end{align}
\end{lemma}
The formal statement is given in Lemma~\ref{LEMMA: TOTAL-ESTIMATE-VARIANCE}, with the complete proof provided in Appendix~\ref{regretproof}. Here, we briefly outline the main ideas of the proof. According to the definition of $\sigma_{k,h}^2$, we need to handle components involving the estimated variance $\bar{\VV}_{s_h^k,a_h^k,h}V_{k,h+1}$, the variance-estimation error $E_{k,h}$, the correction term $D_{k,h}$, and the additive term $H$.

For the estimated variance $\bar{\VV}_{s_h^k,a_h^k,h}V_{k,h+1}$, we approximate it by the true variance $\VV_{s_h^k,a_h^k,h}V_{k,h+1}$, where the approximation error is controlled by $E_{k,h}$. We then compare $\VV_{s_h^k,a_h^k,h}V_{k,h+1}$ with $\VV_{s_h^k,a_h^k,h}V_{h+1}^{\pi^k}$. The partial sums of the resulting errors can be bounded using Lemma~\ref{lemma:transition}. In particular, the partial sum of $\VV_{s_h^k,a_h^k,h}V_{h+1}^{\pi^k}$ over the random suffixes is bounded by $\tilde{O}\!\left(H^2L+H^3\right)$ through a suffix version of the total-variance argument combined with martingale concentration. The partial sums of $E_{k,h}$ can be controlled directly using Lemma~\ref{lemma:sketch-sumbonus}. For $D_{k,h}$, we approximate its main component by $\mathbb P_{s_h^k,a_h^k,h}(V_{k,h+1}-\check V_{k,h+1})$, where the remaining error can again be bounded by the bonus. We further establish a recursive bound for the partial sum of $\mathbb P_{s_h^k,a_h^k,h}(V_{k,h+1}-\check V_{k,h+1})$ over $(k,h)\in\cJ_\tau$ in Lemma~\ref{lemma:transition1}. Combining these bounds gives a self-bounding inequality for $\sum_{(k,h)\in\cJ_\tau}\sigma_{k,h}^2$, whose solution yields \Cref{eq:sketch-total-variance}.

Using Lemma~\ref{lemma:sketch-total-variance} in the bonus bound, together with the martingale bound and $\beta=O(\sqrt d\,\iota_1)$, gives
\begin{align}
X\leq O(d^6H^5\iota_1^3+d\sqrt{H^3\iota_1^3L}).
\label{eq:sketch-global-upper}
\end{align}
Finally, \Cref{eq:sketch-global-lower} implies $L\leq X/\dmin$. Substituting this relation into \Cref{eq:sketch-global-upper} gives $X\leq O(d^6H^5\iota_1^3+d\sqrt{H^3\iota_1^3X/\dmin})$. Solving this inequality proves Lemma~\ref{lemma:sketch-global-error}.

The improved $H^3$ dependence therefore follows from combining the first-positive-gap decomposition with the random-suffix variance analysis: the former reduces the analysis to $L$ relevant suffixes, while the latter controls the sum of the true policy variances by $\tilde O(H^2L+H^3)$. Together with $\dmin L\leq X$, this yields the leading term $\tilde O(d^2H^3/\dmin)$ and completes the proof sketch.

\section{Conclusion}

We establish the first gap-dependent regret guarantee for the nearly minimax-optimal LSVI-UCB++, with a leading term $\tilde{O}(d^2H^3/\dmin)$. We also develop a structural lower bound for persistent ellipsoidal optimism that, when specialized to LSVI-UCB++, matches the leading term of our upper bound up to logarithmic factors. Beyond this class, we prove a general lower bound $\Omega(dH^3\log K/\dmin)$ for arbitrary learning algorithms, showing that the $H^3$ dependence and logarithmic dependence on $K$ are intrinsic. Resolving the remaining gap between the leading gap-dependent term of our upper bound and the general lower bound is an important direction for future work.

\bibliography{main.bib}
\bibliographystyle{plainnat}

\newpage
\appendix
The appendix is organized as follows. Appendix~\ref{related} reviews related work, Appendix~\ref{app:experiment} presents numerical experiments, and Appendices~\ref{keylemma}--\ref{property} provide technical preliminaries. Appendix~\ref{regretproof} proves the upper bound in \Cref{regret}, Appendix~\ref{sec: ellipsoidal-lower-bound} proves the structural lower bound in \Cref{ellipsoidlower}, and Appendix~\ref{lowerproof} proves the general lower bound in \Cref{lower}.
\section{Related Work}

\label{related}
\textbf{Near-Optimal RL.}
In tabular RL, algorithms are categorized into model-based and model-free approaches. Model-based methods estimate the transition and reward models from data and plan based on the learned models, while model-free methods directly maintain value function estimates and act greedily. A large body of work has focused on model-based algorithms \citep{agarwal2020model, agrawal2017optimistic, auer2008near, azar2017minimax, dann2019policy, zanette2019tighter, zhang2024settling, zhang2021reinforcement, zhou2023sharp}. Notably, \citet{zhang2024settling} proposed an algorithm achieving a regret bound of $\tilde{O}(\min\{\sqrt{SAH^2T},\, T\})$, which matches the information-theoretic lower bound. Model-free approaches have also been studied~\citep{jin2018q, li2023breaking, menard2021ucb, yang2021q, zhang2020almost, zhang2025regret}, and several works~\citep{zhang2020almost, menard2021ucb, li2023breaking, zhang2025regret} achieved the near-optimal regret $\tilde{O}(\sqrt{SAH^2T})$. Several works also focus on federated RL, including \citet{zheng2023federated, labbi2024federated, zheng2024federated, zhang2025regret}, with the last three attaining near-optimal regret bounds.

The literature on RL with linear function approximation can be divided based on the structural assumptions imposed on the MDP. One line of work considers linear MDPs \citep{yang2019sample, jin2020provably, wei2021learning, wagenmaker2022first, he2023nearly, zhanghorizon}. \citet{yang2019sample} provided the first sample-efficient algorithm under a generative model. Subsequently, \citet{jin2020provably} proposed the first provably efficient algorithm, LSVI-UCB, which achieves a regret bound $\tilde{O}(\sqrt{d^3 H^3 T})$ without access to a generative model. More recently, \citet{he2023nearly} developed the LSVI-UCB++ algorithm, improving the regret to the nearly minimax-optimal bound $\tilde{O}(d\sqrt{H^3 K})$~\citep{zhou2021nearly}. Another line of work studies linear mixture MDPs \citep{jia2020model, ayoub2020model, modi2020sample, zhou2021nearly, zhou2021provably, zhou2022computationally}. \citet{jia2020model} and \citet{ayoub2020model} proposed the UCRL-VTR algorithm for episodic MDPs, achieving a regret bound of $\tilde{O}(d\sqrt{H^4K})$. Subsequently, \citet{zhou2021nearly} developed a near-optimal algorithm for time-inhomogeneous linear mixture MDPs by proposing a Bernstein-type concentration inequality for self-normalized martingales. Furthermore, \citet{zhou2022computationally} proposed a near-optimal, horizon-free algorithm for time-homogeneous linear mixture MDPs. \citet{zhang2023optimal} provided the near-optimal horizon-free sample complexity in the reward-free time-homogeneous setting. \cite{liang2026provably} provide a provably efficient analysis of federated RL with linear function approximation.

\textbf{Gap-Dependent RL.}
In tabular RL, early works established logarithmic regret guarantees for undiscounted and irreducible MDPs, including finite-time results \citep{auer2007logarithmic, tewari2008optimistic}. Subsequent works developed sharper instance-dependent guarantees and extended gap-dependent analyses to episodic settings \citep{jaksch2010near, ok2018exploration, simchowitz2019non, dann2021beyond, yang2021q, xu2021fine, zheng2024gap, chen2026sharp, zhang2025q}. For model-based algorithms, \citet{simchowitz2019non}, \citet{dann2021beyond}, and \citet{chen2026sharp} obtain fine-grained gap-dependent regret bounds. For model-free algorithms, \citet{yang2021q} provided the first gap-dependent regret bound for UCB-Hoeffding~\citep{jin2018q}, which was subsequently refined by the AMB algorithm of \citet{xu2021fine}. Later, \citet{zheng2024gap} further improved the result by reanalyzing the UCB-Advantage algorithm \citep{zhang2020almost} and the Q-EarlySettled-Advantage algorithm \citep{li2023breaking}. \citet{zhang2025gapdependent, zhang2025regret} extend gap-dependent analysis to federated $Q$-learning settings. More recently, \citet{zhang2025q} provided the first fine-grained gap-dependent regret upper bound for UCB-Hoeffding.

For RL with linear function approximation, gap-dependent regret bounds have also been studied. For online RL, \citet{he2021logarithmic} provided the first gap-dependent regret bound, showing that LSVI-UCB \citep{jin2020provably} achieves $\tilde{O}(d^3 H^5 / \dmin)$ in linear MDPs, and UCRL-VTR~\citep{ayoub2020model} achieves $\tilde{O}(d^2 H^5 / \dmin)$ in linear mixture MDPs. Subsequent works~\citep{papini2021reinforcement,zhang2024achieving} further improved these results, establishing high-probability constant-regret guarantees for LSVI-UCB under the additional UNISOFT condition and for Cert-LSVI-UCB. \cite{zhao2025logarithmic} studied logarithmic regret in a different setting with adversarial reward corruptions. \citet{nguyen2023instance} provided gap-dependent guarantees for offline RL with linear function approximation.

\section{Numerical Experiments}
\label{app:experiment}

We provide synthetic simulations to illustrate the empirical behavior of a practical implementation of LSVI-UCB++, to compare it with representative gap-dependent baselines, and to examine whether the observed growth is compatible with the polylogarithmic dependence on $K$ in \Cref{regret}.

\textbf{Synthetic linear MDPs.}
We consider two families of finite-horizon linear MDPs. In the random family, we use four problem scales
\[
(S,A,H,d)\in\{(5,3,2,3),(5,5,5,3),(8,6,5,5),(10,10,10,8)\}.
\]
For each scale, the feature vector $\bphi(s,a)\in\mathbb R^d$ is sampled from $\mathrm{Dir}(1,\ldots,1)$, so that $\bphi_j(s,a)\ge 0$ and $\sum_{j=1}^d\bphi_j(s,a)=1$. The transition kernel is generated from base kernels $\btheta_{h,j}\in\Delta(\cS)$ via
\[
\mathbb P_h(s'\mid s,a)=\langle \bphi(s,a),\btheta_h(s')\rangle,
\]
where each $\btheta_{h,j}$ is independently sampled from $\mathrm{Dir}(1,\ldots,1)$. Independently of the transition construction, each known reward $r_h(s,a)$ is sampled independently from $\mathrm{Unif}[0,1]$. At the beginning of each episode, the initial state is drawn uniformly over the state space.

To isolate the effect of the minimum positive gap, we also use a controlled-gap family with $(S,A,H,d)=(10,10,10,8)$ and $\dmin\in\{0.2,0.4,0.6,0.8\}$. Let $s_0$ be the initial state and partition the other nine states into a rewarding class $\cS_+$ of four states and a zero-reward class $\cS_-$ of five states. Using the standard basis $e_1,\ldots,e_8$ of $\mathbb R^8$, we set
\[
\bphi(s_0,A-1)=e_1,\qquad \bphi(s_0,a)=e_2\ \ (a=0,\ldots,A-3),\qquad \bphi(s_0,A-2)=e_5,
\]
and set $\bphi(s,a)=e_3$ for $s\in\cS_+$ and $\bphi(s,a)=e_4$ for $s\in\cS_-$. This construction uses five active one-hot coordinates and is zero-padded to dimension $d=8$. At step $h=1$, all state--action pairs have reward $0$. At steps $h=2,\ldots,H$, states in $\cS_+$ have reward $1$, while $s_0$ and states in $\cS_-$ have reward $0$. For every $h\in[H]$, states in $\cS_+$ transition uniformly within $\cS_+$ and states in $\cS_-$ transition uniformly within $\cS_-$. At steps $h=2,\ldots,H$, every action at $s_0$ transitions uniformly into $\cS_-$. Thus, for every state--action--step triple with $(s,h)\neq(s_0,1)$, both the reward and transition kernel are fully specified, and the gap is zero.

Let $U_+$ and $U_-$ denote the uniform distributions over $\cS_+$ and $\cS_-$, respectively, and let $p_\star=0.60+\varepsilon$, where $\varepsilon\sim\mathrm{Unif}[-0.025,0.025]$ is sampled independently for each MDP instance. At step $h=1$, action $A-1$ at $s_0$ transitions into $\cS_+$ with total probability $p_\star$, actions $0,\ldots,A-3$ do so with total probability $p_\star-\dmin/(H-1)$, and action $A-2$ does so with total probability $p_\star-\Delta_{\mathrm{bad}}/(H-1)$, where $\Delta_{\mathrm{bad}}=1.2$; conditional on either class, the next state is uniform within that class.

To make the linear-MDP structure explicit, we set
\[
\btheta_{1,1}=p_\star U_+ +(1-p_\star)U_-,\qquad
\btheta_{1,2}=\left(p_\star-\frac{\dmin}{H-1}\right)U_+ +\left(1-p_\star+\frac{\dmin}{H-1}\right)U_-,
\]
\[
\btheta_{1,5}=\left(p_\star-\frac{\Delta_{\mathrm{bad}}}{H-1}\right)U_+ +\left(1-p_\star+\frac{\Delta_{\mathrm{bad}}}{H-1}\right)U_-,\qquad
\btheta_{1,3}=U_+,\qquad
\btheta_{1,4}=U_-.
\]
For $h\ge2$, we set
\[
\btheta_{h,1}=\btheta_{h,2}=\btheta_{h,4}=\btheta_{h,5}=U_-,\qquad
\btheta_{h,3}=U_+,
\]
with all remaining coordinates equal to the zero measure. Hence $\mathbb P_h(\cdot\mid s,a)=\langle\bphi(s,a),\btheta_h(\cdot)\rangle$ for every $(s,a,h)$. Moreover, for every measurable $B\subseteq\cS$, at most five coordinates of $\btheta_h(B)$ are nonzero and each has magnitude at most one, so $\|\btheta_h(B)\|_2\le\sqrt5\le\sqrt d$.

Consequently, each of the nine noninitial states has positive probability of being reached after the first action. Since $H=10$, $p_\star\in[0.575,0.625]$, and $\dmin\leq0.8<\Delta_{\mathrm{bad}}$, all specified transition probabilities lie in $[0,1]$. The value gaps at $(s_0,h=1)$ are exactly $\dmin$ for actions $0,\ldots,A-3$ and $\Delta_{\mathrm{bad}}$ for action $A-2$. All other state--action--step triples have zero gap, so the minimum positive gap is exactly $\dmin$ for every generated instance. The random common shift changes the MDP instance without changing any action gap.

For each panel, we run each algorithm for $K=20{,}000$ episodes, average over ten independently generated MDP instances, and use paired rollout seeds across algorithms on each instance. Since the MDP is known in simulation, regret is evaluated exactly by dynamic programming as
\[
\sum_{k=1}^K \bigl(V_1^\star - V_1^{\pi^k}\bigr)(s_1^k),
\]
which matches the regret definition in Section~\ref{prelim}. The figures report the mean curve together with the 10th--90th percentile interval.

\textbf{Algorithms and parameters.}
We compare LSVI-UCB++ with LSVI-UCB~\citep{jin2020provably} and Cert-LSVI-UCB~\citep{zhang2024achieving}. To ensure that all methods use the same reward information, all three algorithms use the known reward function and estimate only the transition-dependent continuation-value term. Thus, for LSVI-UCB, the regression target is $V_{k,h+1}(s_{h+1}^i)$ and the optimistic value is computed as
\[
\Lambda_{k,h}=\lambda I+\sum_{i<k}\bphi(s_h^i,a_h^i)\bphi(s_h^i,a_h^i)^\top,\qquad
w_{k,h}=\Lambda_{k,h}^{-1}\sum_{i<k}\bphi(s_h^i,a_h^i)V_{k,h+1}(s_{h+1}^i),
\]
\[
Q_{k,h}(s,a)=\min\big\{r_h(s,a)+\langle w_{k,h},\bphi(s,a)\rangle
+\beta\|\bphi(s,a)\|_{\Lambda_{k,h}^{-1}},\,H\big\}.
\]
We use $\lambda=1$ and $\beta=Hd$ for LSVI-UCB. For the practical implementation of LSVI-UCB++, we set $\lambda=H^{-2}$, $\beta=\bar\beta=\sqrt d$, $\widetilde\beta=\sqrt{dH}$, and use determinant-triggered rare switching with threshold $2$. For Cert-LSVI-UCB, we use $\lambda=1$ and the confidence-radius form of \citet{zhang2024achieving} with multiplicative coefficient $4\times 10^{-4}$.

\begin{figure}[ht]
    \centering
    \begin{subfigure}[b]{0.495\textwidth}
        \centering
        \includegraphics[width=\linewidth]{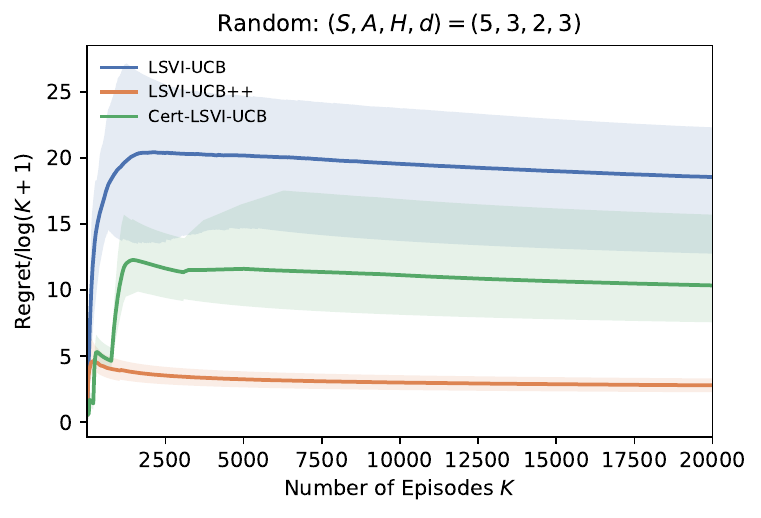}
        \caption{$(S,A,H,d)=(5,3,2,3)$}
    \end{subfigure}
    \hfill
    \begin{subfigure}[b]{0.495\textwidth}
        \centering
        \includegraphics[width=\linewidth]{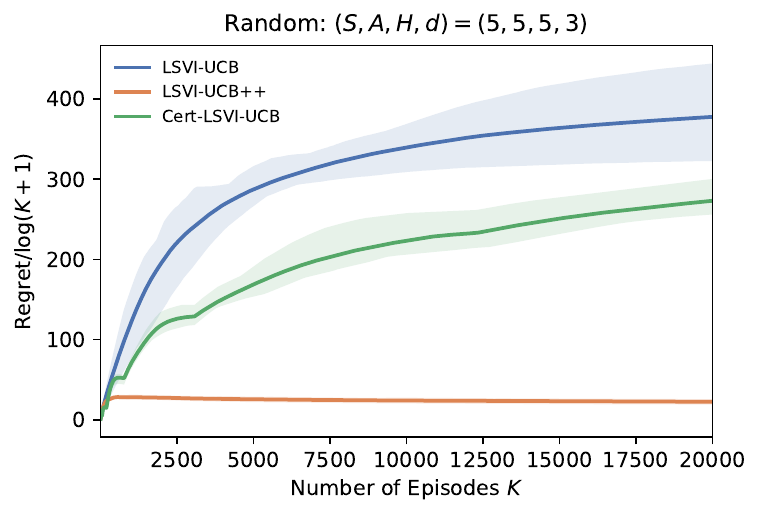}
        \caption{$(S,A,H,d)=(5,5,5,3)$}
    \end{subfigure}
    \begin{subfigure}[b]{0.495\textwidth}
        \centering
        \includegraphics[width=\linewidth]{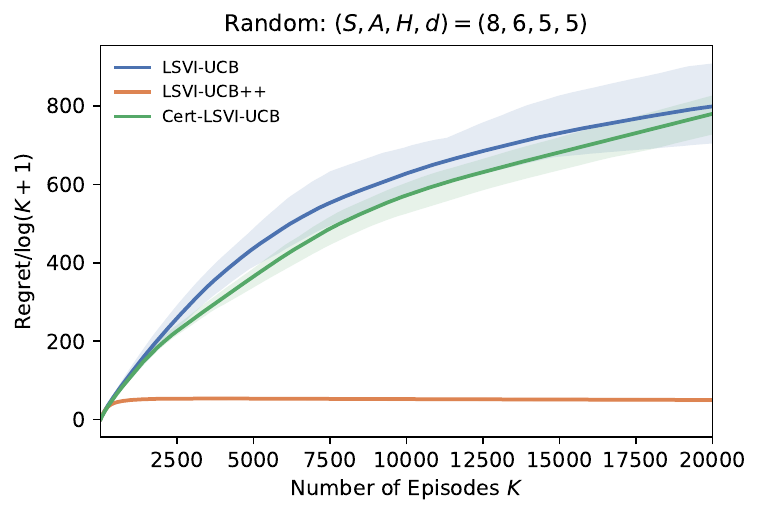}
        \caption{$(S,A,H,d)=(8,6,5,5)$}
    \end{subfigure}
    \hfill
    \begin{subfigure}[b]{0.495\textwidth}
        \centering
        \includegraphics[width=\linewidth]{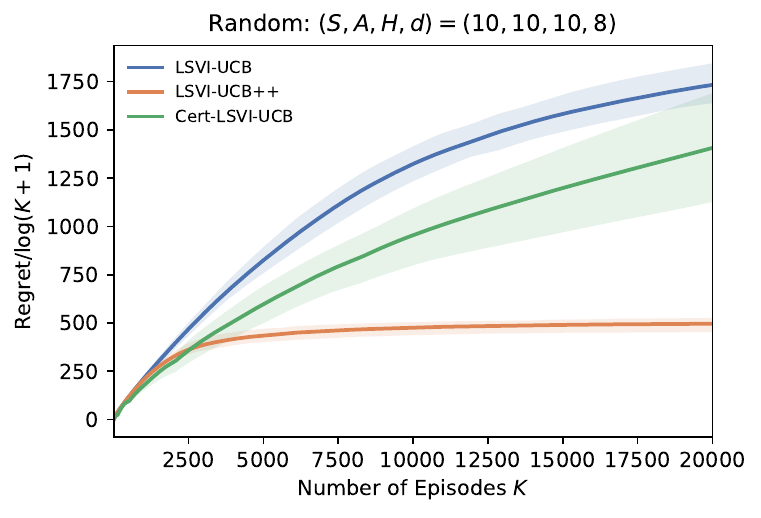}
        \caption{$(S,A,H,d)=(10,10,10,8)$}
    \end{subfigure}
    \caption{Log-normalized regret $\operatorname{Regret}(T)/\log(K+1)$ on random synthetic linear MDPs. Each curve is averaged over ten MDP instances, with shaded 10th--90th percentile intervals.}
    \label{fig:random_regret_comparison}
\end{figure}

\begin{figure}[ht]
    \centering
    \begin{subfigure}[b]{0.495\textwidth}
        \centering
        \includegraphics[width=\linewidth]{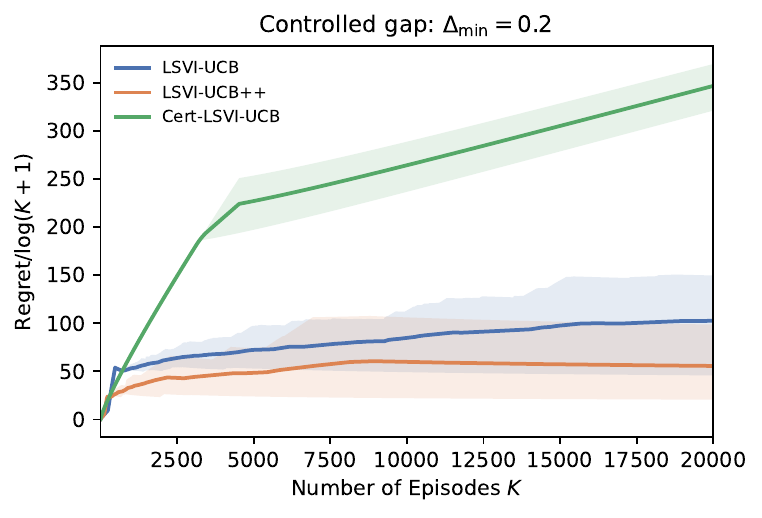}
        \caption{$\dmin=0.2$}
    \end{subfigure}
    \hfill
    \begin{subfigure}[b]{0.495\textwidth}
        \centering
        \includegraphics[width=\linewidth]{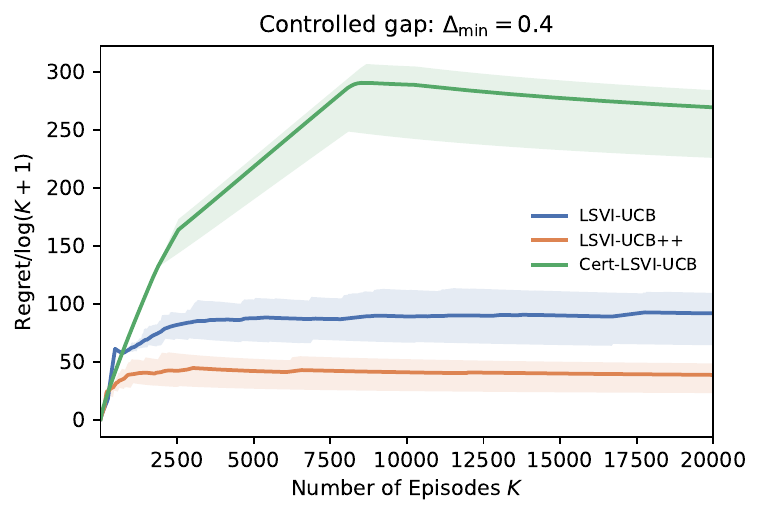}
        \caption{$\dmin=0.4$}
    \end{subfigure}
    \begin{subfigure}[b]{0.495\textwidth}
        \centering
        \includegraphics[width=\linewidth]{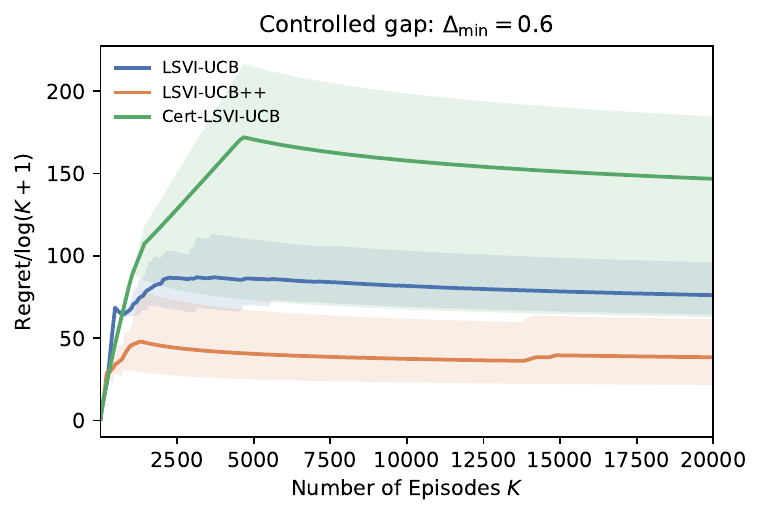}
        \caption{$\dmin=0.6$}
    \end{subfigure}
    \hfill
    \begin{subfigure}[b]{0.495\textwidth}
        \centering
        \includegraphics[width=\linewidth]{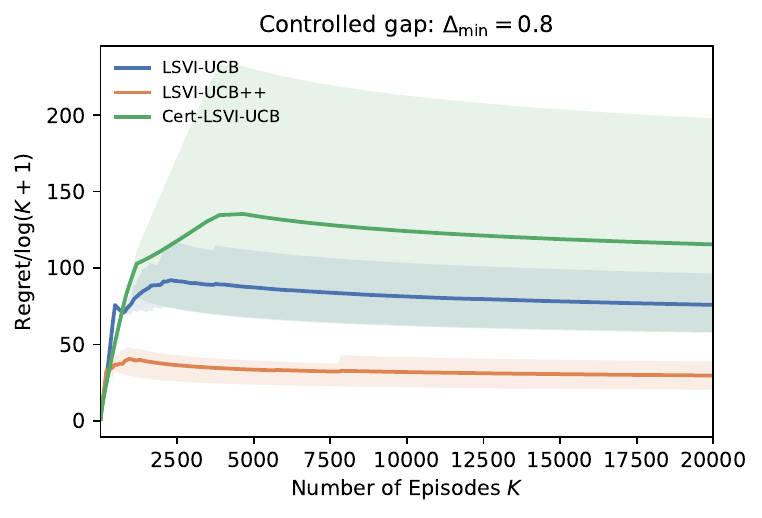}
        \caption{$\dmin=0.8$}
    \end{subfigure}
    \caption{Log-normalized regret $\operatorname{Regret}(T)/\log(K+1)$ on controlled-gap linear MDPs with $(S,A,H,d)=(10,10,10,8)$. Each curve is averaged over ten MDP instances, with shaded 10th--90th percentile intervals.}
    \label{fig:gap_regret_comparison}
\end{figure}

\textbf{Results.}
\Cref{fig:random_regret_comparison} shows that LSVI-UCB++ obtains lower empirical regret than both baselines across the four random linear MDP scales. Its normalized curves $\operatorname{Regret}(T)/\log(K+1)$ become nearly flat toward the end of the tested range, suggesting approximately logarithmic regret growth over this finite range. This trend is qualitatively compatible with the polylogarithmic dependence on $K$ established in \Cref{regret}; it does not imply that the higher-order logarithmic factors in the upper bound are absent.

\Cref{fig:gap_regret_comparison} further examines the dependence on the minimum positive gap. LSVI-UCB++ achieves lower empirical regret than both baselines under the reported parameter settings. Over the later part of the tested range, its regret decreases as $\dmin$ increases from $0.2$ to $0.8$, consistent with the gap-dependent intuition that larger action gaps make suboptimal actions easier to eliminate.

\section{Auxiliary Lemmas}
\label{keylemma}
In this section, we introduce several auxiliary lemmas that will be used to support our proof.
\begin{lemma}[Azuma-Hoeffding inequality]
\label{Hoeffding}
Suppose {$\left\{X_k\right\}_{k=0}^\infty$} is a martingale and $|X_k-X_{k-1}|\leq c_k$, $\forall k\in\mathbb{N}_+$, almost surely. Then for any $N \in \mathbb{N}_+$ and $\epsilon >0$, it holds that:
$$\mathbb{P}\left(X_N-X_{0}\geq \epsilon\right) \leq \exp \left(-\frac{\epsilon^2}{2\sum_{k=1}^{N}c_k^2}\right),$$
and
$$\mathbb{P}\left(|X_N-X_{0}|\geq \epsilon\right) \leq 2\exp \left(-\frac{\epsilon^2}{2\sum_{k=1}^{N}c_k^2}\right).$$
\end{lemma}

\begin{lemma}[Freedman's inequality {\cite[Theorem~1.6]{freedman1975tail}}]
\label{Freedman}
Suppose $\left\{X_k\right\}_{k=0}^\infty$ is a martingale with respect to the filtration $\left\{\mathcal F_k\right\}_{k=0}^\infty$ and $|X_k-X_{k-1}|\leq c$, $\forall k\in\mathbb N_+$, almost surely. Define
\[
\operatorname{Var}_N:=\sum_{k=1}^N\mathbb E\left[(X_k-X_{k-1})^2\mid\mathcal F_{k-1}\right].
\]
Then, for any $\epsilon,v>0$, it holds that
\[
\mathbb P\left(\exists N\in\mathbb N_+:X_N-X_0\geq\epsilon\ \text{and}\ \operatorname{Var}_N\leq v\right)\leq\exp\left(-\frac{\epsilon^2}{2(v+c\epsilon)}\right).
\]
Consequently, for any $N\in\mathbb N_+$, $v>0$, and $\delta\in(0,1)$, it holds that
\[
\mathbb P\left(X_N-X_0>2\sqrt{v\log\left(\frac{1}{\delta}\right)}+2c\log\left(\frac{1}{\delta}\right),\ \operatorname{Var}_N\leq v\right)\leq\delta.
\]
\end{lemma}

\begin{proof}
The first claim follows from Theorem~1.6 of \cite{freedman1975tail}. To prove the consequence, let
\[
q:=\log\left(\frac{1}{\delta}\right),\qquad \epsilon:=2\sqrt{vq}+2cq.
\]
By Freedman's inequality,
\begin{align*}
    \mathbb P\left(X_N-X_0>\epsilon,\operatorname{Var}_N\leq v\right)&\leq\mathbb P\left(\exists n\in\mathbb N_+:X_n-X_0>\epsilon\ \text{and}\ \operatorname{Var}_n\leq v\right)\\
    &\leq\exp\left(-\frac{\epsilon^2}{2(v+c\epsilon)}\right).
\end{align*}
Moreover,
\[
\epsilon^2-2q(v+c\epsilon)=2vq+4cq\sqrt{vq}\geq0,
\]
which implies
\[
\frac{\epsilon^2}{2(v+c\epsilon)}\geq q.
\]
Therefore,
\[
\mathbb P\left(X_N-X_0>\epsilon,\ \operatorname{Var}_N\leq v\right)\leq e^{-q}=\delta.
\]
\end{proof}

\begin{lemma}[Lemma 12 of \cite{abbasi2011improved}]\label{lemma:det}
    If $\Ab, \Bb\in \mathbb R^{d \times d}$ are two positive definite matrices such that $\Ab \succeq \Bb$, then for any $\xb \in \mathbb R^d$, $\|\xb\|_{\Ab} \leq \|\xb\|_{\Bb}\cdot \sqrt{\det(\Ab)/\det(\Bb)}$.
\end{lemma}
    
\begin{lemma}[Lemma 11 of  \cite{abbasi2011improved}]\label{Lemma:abba}
        Let $\{\xb_k\}_{k=1}^{K}$ be a sequence of vectors in $\mathbb R^d$. Let matrix $\bSigma_0$ be a $d \times d$ positive definite matrix and define $\bSigma_k=\bSigma_0+\sum_{i=1}^{k} \xb_i\xb_i^{\top}$, then we have
        \begin{align}
            \sum_{i=1}^{k} \min\left\{1,\xb_i^{\top} \bSigma_{i-1}^{-1} \xb_i\right\}\leq 2 \log \left(\frac{\det{\bSigma_k}}{\det{\bSigma_0}}\right).\notag
        \end{align}
        In addition, if $\|\xb_i\|_2\leq L$ holds for all $i\in [K]$, then 
        \begin{align}
            \sum_{i=1}^{k} \min\left\{1,\xb_i^{\top} \bSigma_{i-1}^{-1} \xb_i\right\}\leq 2 \log \left(\frac{\det{\bSigma_k}}{\det{\bSigma_0}}\right)\leq 2\left(d\log\left((\text{trace}(\bSigma_0)+kL^2)/d\right)-\log \det \bSigma_0\right).\notag
        \end{align}
    \end{lemma}
    
\begin{lemma}[Lemma 4.4 of \cite{zhou2022computationally}]\label{lemma:keysum:temp}
        Let $\{\sigma_k\}_{k \geq 1}$ be a sequence of non-negative numbers, $\alpha, \gamma, \lambda>0$, $\{\ab_k\}_{k \geq 1} \subset \mathbb R^d$ and $\|\ab_k\|_2 \leq A$. Let $\{\bar\sigma_k\}_{k \geq 1}$ and $\{\hat\bSigma_k\}_{k \geq 1}$ be (recursively) defined as follows: $\hat\bSigma_1 = \lambda\Ib_d$,\ 
        \begin{align}
            \forall k \geq 1,\ \bar\sigma_k = \max\{\sigma_k, \alpha, \gamma\|\ab_k\|_{\hat\bSigma_k^{-1}}^{1/2}\},\ \hat\bSigma_{k+1} = \hat\bSigma_k + \ab_k\ab_k^\top/\bar\sigma_k^2.\notag
        \end{align}
        Let $\iota_0 = \log(1+KA^2/(d\lambda\alpha^2))$. Then we have
        \begin{align}
            \sum_{k=1}^K\min\left\{1, \|\ab_k\|_{\hat\bSigma_k^{-1}}\right\} &\leq 2d \iota_0 +2\gamma^2d\iota_0+ 2\sqrt{d \iota_0}\sqrt{\sum_{k=1}^K(\sigma_k^2 + \alpha^2)}\notag.
        \end{align}
    \end{lemma}

\section{Minor Modifications to LSVI-UCB++}
\label{app:lsvi-ucbpp-modifications}

Our implementation of LSVI-UCB++ makes two minor modifications to the pseudocode of \citet{he2023nearly}. The first clarifies the timing of the rare-switching update, while the second relaxes the adjusted variance normalization. We describe both modifications below and explain why the results from \citet{he2023nearly} invoked in our analysis remain valid.

\textbf{Rare-switching update schedule.} In the original pseudocode of \citet{he2023nearly}, the last updating episode $k_{\mathrm{last}}$ is reset to the current episode $k$ immediately after the determinant-doubling criterion is triggered inside the backward loop over $h=H,\ldots,1$. Taken literally, this changes the reference covariance matrices within the same backward pass: after the first step is updated, the criterion is subsequently compared against $\Sigma_{k,h'}$ itself and therefore cannot remain triggered. To implement the episode-level rare-switching update intended in \citet{he2023nearly}, Algorithm~\ref{algorithm1} introduces an auxiliary variable $k_{\mathrm{mid}}$. Whenever the updating criterion is triggered, we set $k_{\mathrm{mid}}\leftarrow k$, while keeping $k_{\mathrm{last}}$ fixed throughout the entire backward pass. Only after all steps have been processed do we set $k_{\mathrm{last}}\leftarrow k_{\mathrm{mid}}$. Consequently, when an episode triggers the determinant-doubling criterion, the optimistic and pessimistic value functions are updated consistently at all steps using the same previous updating episode $k_{\mathrm{last}}$ as reference. This modification only makes explicit the episode-level update schedule used in the rare-switching analysis of \citet{he2023nearly}. Hence, the rare-switching and covering-number arguments used in their analysis continue to apply.

\textbf{Relaxed adjusted variance normalization.} The second modification concerns the adjusted variance used in the weighted ridge regression. The original algorithm defines
\[
\bar{\sigma}_{k,h}=\max\left\{\sigma_{k,h},H,2d^3H^2\|\phi(s_h^k,a_h^k)\|_{\Sigma_{k,h}^{-1}}^{1/2}\right\},
\]
whereas Algorithm~\ref{algorithm1} uses
\[
\bar{\sigma}_{k,h}^2=\max\left\{\sigma_{k,h}^2,H,2d^3H^2\|\phi(s_h^k,a_h^k)\|_{\Sigma_{k,h}^{-1}}\right\}.
\]
All other components of the weighted regression and variance estimator remain unchanged. Let $u_{k,h}:=\|\phi(s_h^k,a_h^k)\|_{\Sigma_{k,h}^{-1}}$. Our normalization guarantees
\[
\bar{\sigma}_{k,h}^2\ge \sigma_{k,h}^2,\qquad \bar{\sigma}_{k,h}^2\ge H,\qquad \bar{\sigma}_{k,h}^2\ge 2d^3H^2u_{k,h}.
\]
Consequently,
\[
\bar{\sigma}_{k,h}^{-1}\le H^{-1/2},\qquad \frac{2Hu_{k,h}}{\bar{\sigma}_{k,h}^2}\le \frac{1}{d^3H}.
\]
These are the properties of the adjusted variance needed in the self-normalized concentration arguments used by \citet{he2023nearly}. Moreover, $\bar{\sigma}_{k,h}^2\ge\sigma_{k,h}^2$ preserves the required conditional-variance domination. Therefore, the arguments underlying Lemmas B.1 and B.5 of \citet{he2023nearly} remain valid under our normalization with the same orders of the confidence radii.

The relaxed normalization is useful for our gap-dependent analysis because the corresponding weighted elliptical-potential argument can be applied with $\alpha=\sqrt H$ and $\gamma=\sqrt{2d^3H^2}$, which yields the sharper partial bonus bound developed below. The rare-switching and covering-number arguments used in our analysis are also unaffected by the modified update schedule. Moreover, the concentration arguments inherited from \citet{he2023nearly} are compatible with the normalization properties verified above and continue to apply under our modified definition of $\bar{\sigma}_{k,h}$. Consequently, Algorithm 1 preserves the nearly minimax-optimal leading-order regret established by \citet{he2023nearly}.

\section{Probability Events}
\label{probabilityevent}
In this section, we introduce several high-probability events for LSVI-UCB++. For every $(k,h)\in[K]\times[H+1]$, define the clipped optimistic and pessimistic value functions as
\[
V_{k,h}^{\mathrm{clip}}(s):=\max\{V_{k,h}(s),0\},\qquad s\in\mathcal S,
\]
and
\[
\check V_{k,h}^{\mathrm{clip}}(s):=\min\{\check V_{k,h}(s),H\},\qquad s\in\mathcal S.
\]
Since $V_{k,h}(s)\leq H$ and $\check V_{k,h}(s)\geq 0$ by construction, we have $0\leq V_{k,h}^{\mathrm{clip}}(s), \check V_{k,h}^{\mathrm{clip}}(s)\leq H$ for any $s \in \mathcal S$.

\begin{lemma}
\label{concentration}
For LSVI-UCB++, we have the following high-probability events.
\begin{itemize}
\item[(a)] Define $\cE_1$ as the event that the following inequalities hold for any $(s,a,h,k)\in \sah\times[K]$ simultaneously:
\begin{align}
\left|\hat{\wb}_{k,h}^{\top}\bphi(s,a)-\mathbb P_{s,a,h}V_{k,h+1}\right| &\leq \bar{\beta}\sqrt{\bphi(s,a)^{\top}\bSigma_{k,h}^{-1}\bphi(s,a)},\notag\\
\left|\tilde{\wb}_{k,h}^{\top}\bphi(s,a)-\mathbb P_{s,a,h}V_{k,h+1}^{2}\right| &\leq \tilde{\beta}\sqrt{\bphi(s,a)^{\top}\bSigma_{k,h}^{-1}\bphi(s,a)},\notag\\
\left|\check{\wb}_{k,h}^{\top}\bphi(s,a)-\mathbb P_{s,a,h}\check{V}_{k,h+1}\right| &\leq \bar{\beta}\sqrt{\bphi(s,a)^{\top}\bSigma_{k,h}^{-1}\bphi(s,a)}.\notag
\end{align}
Here, for sufficiently large constants $\bar c_\beta, \tilde c_\beta \geq 1$,
\[
\bar{\beta}=\bar c_\beta\left(H\sqrt{d\lambda}+\sqrt{d^3H^2\log^2\left(dHK/(\delta\lambda)\right)}\right)
\]
and
\[
\tilde{\beta}=\tilde c_\beta\left(H^2\sqrt{d\lambda}+\sqrt{d^3H^4\log^2\left(dHK/(\delta\lambda)\right)}\right).
\]
Then the event $\cE_1$ holds with probability at least $1-7\delta$.

\item[(b)] Define $\cE_2$ as the event such that, for any $(s,a,h,k)\in\sah\times[K]$, the weight vector $\hat{\wb}_{k,h}$ satisfies
\[
\left|\hat{\wb}_{k,h}^{\top}\bphi(s,a)-\mathbb P_{s,a,h}V_{k,h+1}\right|\leq\beta\sqrt{\bphi(s,a)^{\top}\bSigma_{k,h}^{-1}\bphi(s,a)},
\]
where for sufficiently large constant $c_\beta \geq 1$,
\[
\beta=c_\beta\left(H\sqrt{d\lambda}+\sqrt{d\log^2\left(1+dKH/(\delta\lambda)\right)}\right).
\]
Then the event $\cE_2$ holds with probability at least $1-8\delta$.

\item[(c)] Define $\cE_3$ as the event
\begin{align*}
\cE_3:=\Bigg\{\sum_{(k,h)\in\cJ_\tau}\left(\mathbb P_{s_h^k,a_h^k,h}-\mathbbm{1}_{s_{h+1}^k}\right)\left(V_{k,h+1}^{\mathrm{clip}}-V_{h+1}^{\pi^k}\right)\leq2\sqrt{2H^3L\log\left(\frac{KH}{\delta}\right)}\Bigg\}.
\end{align*}
Then the event $\cE_3$ holds with probability at least $1-\delta$.

\item[(d)] Define $\cE_4$ as the event
\begin{align*}
\cE_4:=\Bigg\{\sum_{(k,h)\in\cJ_\tau}(h-\tau_k+1)\left(\mathbb P_{s_h^k,a_h^k,h}-\mathbbm{1}_{s_{h+1}^k}\right)&\left(V_{k,h+1}^{\mathrm{clip}}-V_{h+1}^{\pi^k}\right)\\
&\leq2\sqrt{2H^5L\log\left(\frac{KH}{\delta}\right)}\Bigg\}.
\end{align*}
Then the event $\cE_4$ holds with probability at least $1-\delta$.

\item[(e)] Define $\mathcal E_5$ as the event
\begin{align*}
\mathcal E_5:=\Bigg\{
\sum_{(k,h)\in\mathcal J_\tau}(h-\tau_k+1)
\left(\mathbb P_{s_h^k,a_h^k,h}-\mathbbm 1_{s_{h+1}^k}\right)
&\left(V_{k,h+1}^{\mathrm{clip}}-\check V_{k,h+1}^{\mathrm{clip}}\right)
\\
&\leq 2\sqrt{2H^5L\log\left(\frac{KH}{\delta}\right)}
\Bigg\}.
\end{align*}
Then the event $\mathcal E_5$ holds with probability at least $1-\delta$.

\item[(f)] Define $\cE_6$ as the event
\begin{align*}
\cE_6:=\left\{\sum_{(k,h)\in\cJ_\tau}\mathbb V_{s_h^k,a_h^k,h}V_{h+1}^{\pi^k}\leq2H^2L+3H^3\log\left(\frac{K}{\delta}\right)\right\}.
\end{align*}
Then the event $\cE_6$ holds with probability at least $1-\delta$.

\end{itemize}
\end{lemma}
\begin{proof}[Proof of Lemma \ref{concentration}]
Parts (a) and (b) follow from the arguments of Lemmas B.1 and B.5 in \citet{he2023nearly}, respectively. As discussed in Section~\ref{app:lsvi-ucbpp-modifications}, these arguments remain valid under the adjusted-variance normalization used in Algorithm~\ref{algorithm1}.

\textbf{Retained sequence forms a martingale-difference sequence.}
We establish a common retained-sequence construction used in the proofs of parts (c), (d) and (e). Let $\mathcal F_{k,h}$ denote the $\sigma$-field generated by all samples before episode $k$, together with the trajectory in episode $k$ up to step $h$ after the action $a_h^k$ is selected and before the next state $s_{h+1}^k$ is observed. By the definition of $\tau_k$, for every $(k,h)\in[K]\times[H]$, \[\{\tau_k\leq h\}=\bigcup_{h'=1}^{h}\{\Delta_{h'}(s_{h'}^k,a_{h'}^k)>0\}\in\mathcal F_{k,h},\] since $s_{h'}^k$ and $a_{h'}^k$ are $\mathcal F_{k,h}$-measurable for every $h'\leq h$ and the true gap functions are fixed by the underlying MDP.

Order all pairs $(k,h)\in[K]\times[H]$ lexicographically, and let $(k(n),h(n))$ denote the $n$-th pair in this order. Let $\{\mathcal H_n\}_{n=0}^{KH}$ be the corresponding filtration, where \[\mathcal H_{n-1}:=\mathcal F_{k(n),h(n)},\qquad n\in[KH],\] and $\mathcal H_{KH}$ is the terminal sigma-field containing all information revealed after the last transition increment. Thus, $\mathcal H_{n-1}$ contains the information available immediately before $s_{h(n)+1}^{k(n)}$ is observed, whereas $\mathcal H_n$ contains the information available after this transition. For convenience, set $\mathcal H_{KH+1}:=\mathcal H_{KH}$. For any stopping time $\gamma$ taking values in $\{0,\ldots,KH+1\}$, let \[\mathcal H_\gamma:=\left\{A\in\mathcal H_{KH}:A\cap\{\gamma=n\}\in\mathcal H_n\text{ for every }n\in\{0,\ldots,KH+1\}\right\}\] denote the corresponding stopped sigma-field.

Define \[Z_n:=\mathbb I[\tau_{k(n)}\leq h(n)],\qquad n\in[KH].\] Then $Z_n$ is $\mathcal H_{n-1}$-measurable, and \[M:=\sum_{n=1}^{KH}Z_n=\sum_{k=1}^{K}\sum_{h=1}^{H}\mathbb I[\tau_k\leq h]=\sum_{k\in\cK_\tau}(H-\tau_k+1)\leq HL.\] Set $N_0:=0$, and for each $j\in[KH]$, define \[N_j:=\inf\left\{n\in[KH]:\sum_{i=1}^{n}Z_i=j\right\},\] with the convention $N_j=KH+1$ if the set is empty. Thus, $N_j$ is the position of the $j$-th retained pair when $j\leq M$, whereas $N_j=KH+1$ when $j>M$. Moreover, for every $n\in[KH]$, \[\{N_j=n\}=\left\{\sum_{i=1}^{n-1}Z_i=j-1\right\}\cap\{Z_n=1\}\in\mathcal H_{n-1},\] so $N_j$ is a predictable stopping time.

Let $\{\xi_n\}_{n=1}^{KH}$ be any sequence such that $\xi_n$ is $\mathcal H_n$-measurable and $\mathbb E[\xi_n\mid\mathcal H_{n-1}]=0.$ Define \[\eta_j:=\sum_{n=1}^{KH}\mathbb I[N_j=n]\xi_n,\qquad j\in[KH].\] 
Thus, $\eta_j=\xi_{N_j}$ when $j\leq M$, while $\eta_j=0$ when $j>M$. 

Fix any $A\in\mathcal H_{N_{j-1}}$. Since $N_{j-1}\leq n-1$ on $\{N_j=n\}$ and $\{N_j=n\}\in\mathcal H_{n-1}$, it follows that
\[
A\cap\{N_j=n\}=A\cap\{N_{j-1}\leq n-1\}\cap\{N_j=n\}\in\mathcal H_{n-1}.\]
Therefore,
\begin{align*}
\mathbb E[\mathbb I_A\eta_j]
=\sum_{n=1}^{KH}\mathbb E\left[\mathbb I_{A\cap\{N_j=n\}}\xi_n\right]
=\sum_{n=1}^{KH}\mathbb E\left[\mathbb I_{A\cap\{N_j=n\}}\mathbb E[\xi_n\mid\mathcal H_{n-1}]\right]
=0.
\end{align*}
Since this holds for every $A\in\mathcal H_{N_{j-1}}$,
\[
\mathbb E[\eta_j\mid\mathcal H_{N_{j-1}}]=0.
\]
Moreover, $\eta_j$ is $\mathcal H_{N_j}$-measurable. Therefore, $\{\eta_j\}_{j=1}^{KH}$ is a martingale-difference sequence with respect to the filtration $\{\mathcal H_{N_j}\}_{j=0}^{KH}$.

If $|\eta_j|\leq B$ for every $j\in[KH]$, then for any deterministic $m\in[KH]$, Lemma \ref{Hoeffding} gives
\[
\mathbb P\left(\sum_{j=1}^{m}\eta_j>B\sqrt{2m\log\left(\frac{KH}{\delta}\right)}\right)\leq\frac{\delta}{KH}.
\]
Taking a union bound over $m\in[KH]$, with probability at least $1-\delta$, we have simultaneously for all $m\in[KH]$,
\[
\sum_{j=1}^{m}\eta_j\leq B\sqrt{2m\log\left(\frac{KH}{\delta}\right)}.
\]
Hence, on this event, if $M\geq1$, evaluating the uniform bound at $m=M$ gives
\[
\sum_{j=1}^{M}\eta_j\leq B\sqrt{2M\log\left(\frac{KH}{\delta}\right)},
\]
while the case $M=0$ is immediate.

\textbf{(c)} Set \[\xi_n:=\left(\mathbb P_{s_{h(n)}^{k(n)},a_{h(n)}^{k(n)},h(n)}-\mathbbm{1}_{s_{h(n)+1}^{k(n)}}\right)\left(V_{k(n),h(n)+1}^{\mathrm{clip}}-V_{h(n)+1}^{\pi^{k(n)}}\right),\qquad n\in[KH].\] By the Markov property, \[\mathbb E[\xi_n\mid\mathcal H_{n-1}]=0.\] Let $\{\eta_j\}_{j=1}^{KH}$ be the corresponding retained sequence. Since both $V_{k,h+1}^{\mathrm{clip}}$ and $V_{h+1}^{\pi^k}$ take values in $[0,H]$, we have $|\eta_j|\leq2H$. Applying the uniform-prefix bound above with $B=2H$ and using $M\leq HL$, with probability at least $1-\delta$,
\begin{align*}
&\sum_{(k,h)\in\cJ_\tau}\left(\mathbb P_{s_h^k,a_h^k,h}-\mathbbm{1}_{s_{h+1}^k}\right)\left(V_{k,h+1}^{\mathrm{clip}}-V_{h+1}^{\pi^k}\right)=\sum_{j=1}^{M}\eta_j
\leq2\sqrt{2H^3L\log\left(\frac{KH}{\delta}\right)}.
\end{align*}
This proves part (c).

\textbf{(d)}
Set
\[
\xi_n:=Z_n\bigl(h(n)-\tau_{k(n)}+1\bigr)\left(\mathbb P_{s_{h(n)}^{k(n)},a_{h(n)}^{k(n)},h(n)}-\mathbbm{1}_{s_{h(n)+1}^{k(n)}}\right)\left(V_{k(n),h(n)+1}^{\mathrm{clip}}-V_{h(n)+1}^{\pi^{k(n)}}\right),\ n\in[KH].
\]
To verify predictability, observe that
\[
Z_n\bigl(h(n)-\tau_{k(n)}+1\bigr)=\sum_{t=1}^{h(n)}\bigl(h(n)-t+1\bigr)\mathbbm{1}[\tau_{k(n)}=t].
\]
For every $t\leq h(n)$, the event $\{\tau_{k(n)}=t\}$ is $\mathcal H_{n-1}$-measurable. Therefore, $Z_n\bigl(h(n)-\tau_{k(n)}+1\bigr)$ is $\mathcal H_{n-1}$-measurable, and the Markov property yields
\[
\mathbb E[\xi_n\mid\mathcal H_{n-1}]=0.
\]
Let $\{\eta_j\}_{j=1}^{KH}$ be the corresponding retained sequence. Since $0\leq Z_n\bigl(h(n)-\tau_{k(n)}+1\bigr)\leq H$, we have $|\eta_j|\leq2H^2$. Applying the uniform-prefix bound above with $B=2H^2$ and using $M\leq HL$, with probability at least $1-\delta$,
\begin{align*}
&\sum_{(k,h)\in\cJ_\tau}(h-\tau_k+1)\Big(\mathbb P_{s_h^k,a_h^k,h}-\mathbbm{1}_{s_{h+1}^k}\Big)\left(V_{k,h+1}^{\mathrm{clip}}-V_{h+1}^{\pi^k}\right)=\sum_{j=1}^{M}\eta_j
\leq2\sqrt{2H^5L\log\left(\frac{KH}{\delta}\right)}.
\end{align*}
This proves part (d).

\textbf{(e)} For any $n\in[KH]$, set \[\xi_n:=Z_n\bigl(h(n)-\tau_{k(n)}+1\bigr)\left(\mathbb P_{s_{h(n)}^{k(n)},a_{h(n)}^{k(n)},h(n)}-\mathbbm{1}_{s_{h(n)+1}^{k(n)}}\right)\left(V_{k(n),h(n)+1}^{\mathrm{clip}}-\check V_{k(n),h(n)+1}^{\mathrm{clip}}\right),\] and let $\{\eta_j\}_{j=1}^{KH}$ be the corresponding retained sequence. By the same argument as in part (d), $\mathbb E[\xi_n\mid\mathcal H_{n-1}]=0$ and $|\eta_j|\leq2H^2$. Therefore, with probability at least $1-\delta$,
\begin{align*}
&\sum_{(k,h)\in\cJ_\tau}(h-\tau_k+1)\left(\mathbb P_{s_h^k,a_h^k,h}-\mathbbm{1}_{s_{h+1}^k}\right)\big(V_{k,h+1}^{\mathrm{clip}}-\check V_{k,h+1}^{\mathrm{clip}}\big)=\sum_{j=1}^{M}\eta_j
\leq2\sqrt{2H^5L\log\left(\frac{KH}{\delta}\right)}.
\end{align*}
This proves part (e).

\textbf{(f)} If $L=0$, then $\cJ_\tau$ is empty and the conclusion holds immediately. We therefore consider the case $L\geq1$ and write the elements of $\cK_\tau$ in increasing order as \[k_1<k_2<\cdots<k_L.\] For each $i\in[L]$, define \[U_i:=\sum_{h=\tau_{k_i}}^{H}\mathbb V_{s_h^{k_i},a_h^{k_i},h}V_{h+1}^{\pi^{k_i}},\] and set $U_i:=0$ for $L<i\leq K$.

Let $\ell(k,h):=(k-1)H+h$ denote the lexicographic position of $(k,h)$. Set $\rho_0:=0$, and for every $i\in[K]$, define
\[
\gamma_i:=\begin{cases}\ell(k_i,\tau_{k_i})-1,&i\leq L,\\KH,&i>L,\end{cases}\qquad \rho_i:=\begin{cases}k_iH -1,&i\leq L,\\KH,&i>L.\end{cases}
\]
For every $(k,t)\in[K]\times[H]$, we have \[\{i\leq L,k_i=k,\tau_k=t\}\in\mathcal F_{k,t}=\mathcal H_{\ell(k,t)-1}.\] Moreover, \[\{i\leq L,k_i=k\}=\left\{\sum_{k'=1}^{k-1}\mathbbm{1}[\tau_{k'}\leq H]=i-1\right\}\cap\{\tau_k\leq H\}\in\mathcal F_{k,H}=\mathcal H_{kH-1}.\]
Therefore, $\gamma_i$ and $\rho_i$ are stopping times with respect to $\{\mathcal H_n\}_{n=0}^{KH}$. Moreover,
\[
\mathcal H_{\rho_{i-1}}\subseteq\mathcal H_{\gamma_i}\subseteq\mathcal H_{\rho_i},\qquad i\in[K].
\]

We first bound the conditional expectation of $U_i$. Fix any $(k,t)\in[K]\times[H]$ and, for every $h\in\{t,\ldots,H\}$, define \[\zeta_{k,h}:=r_h(s_h^k,a_h^k)+V_{h+1}^{\pi^k}(s_{h+1}^k)-V_h^{\pi^k}(s_h^k).\] By the Bellman equation and the Markov property,
\[
\mathbb E\left[\zeta_{k,h}\middle|\mathcal F_{k,h}\right]=0,\qquad
\mathbb E\left[\zeta_{k,h}^2\middle|\mathcal F_{k,h}\right]
=
\mathbb V_{s_h^k,a_h^k,h}V_{h+1}^{\pi^k}.
\]
Consequently, conditional orthogonality of the martingale differences gives
\begin{align*}
\mathbb E\left[\sum_{h=t}^{H}\mathbb V_{s_h^k,a_h^k,h}V_{h+1}^{\pi^k}\middle|\mathcal F_{k,t}\right]
&=
\mathbb E\left[\left(\sum_{h=t}^{H}\zeta_{k,h}\right)^2\middle|\mathcal F_{k,t}\right]\\
&=
\mathbb E\left[\left(\sum_{h=t}^{H}r_h(s_h^k,a_h^k)-V_t^{\pi^k}(s_t^k)\right)^2\middle|\mathcal F_{k,t}\right]\\
&\leq H^2,
\end{align*}
where the second equality follows by telescoping and using $V_{H+1}^{\pi^k}=0$. Applying this bound on each event $\{i\leq L,k_i=k,\tau_k=t\}$, and noting that $U_i=0$ for $i>L$, gives \[\mathbb E\left[U_i\middle|\mathcal H_{\gamma_i}\right]\leq H^2,\qquad i\in[K].\] This is the suffix version of the total-variance argument used in Lemma C.5 of \cite{jin2018q}.

For every $i\in[L]$, since $0\leq V_{h+1}^{\pi^{k_i}}\leq H$, \[0\leq\mathbb V_{s_h^{k_i},a_h^{k_i},h}V_{h+1}^{\pi^{k_i}}\leq H^2.\] Since the suffix contains at most $H$ steps and $U_i=0$ for $i>L$, \[0\leq U_i\leq H^3,\qquad i\in[K].\]
Since $V_{H+1}^{\pi^{k_i}}=0$ for every $i\in[L]$, the variance term corresponding to $h=H$ vanishes. Hence, for every $i\in[L]$, $U_i$ is determined by the states and actions observed up to step $H$ of episode $k_i$ and is therefore $\mathcal H_{\rho_i}$-measurable. Moreover, for $i>L$, $U_i=0$ and is thus trivially $\mathcal H_{\rho_i}$-measurable. Therefore, $U_i$ is $\mathcal H_{\rho_i}$-measurable for every $i\in[K]$. For each $i\in[K]$, define
\[
D_i:=U_i-\mathbb E\left[U_i\middle|\mathcal H_{\gamma_i}\right].
\]
Since $\mathcal H_{\gamma_i}\subseteq\mathcal H_{\rho_i}$, the conditional expectation $\mathbb E[U_i\mid\mathcal H_{\gamma_i}]$ is also $\mathcal H_{\rho_i}$-measurable. Therefore, $D_i$ is $\mathcal H_{\rho_i}$-measurable and satisfies
\[
\mathbb E[D_i\mid\mathcal H_{\gamma_i}]=0.
\]
Since $\mathcal H_{\rho_{i-1}}\subseteq\mathcal H_{\gamma_i}$, the tower property gives \[\mathbb E[D_i\mid\mathcal H_{\rho_{i-1}}]=0.\] Thus, $\{D_i\}_{i=1}^{K}$ is a martingale-difference sequence with respect to the filtration $\{\mathcal H_{\rho_i}\}_{i=0}^{K}$. Furthermore, $|D_i|\leq H^3$ and
\begin{align*}
\mathbb E\left[D_i^2\middle|\mathcal H_{\gamma_i}\right]
\leq
\mathbb E\left[U_i^2\middle|\mathcal H_{\gamma_i}\right]\leq
H^3\mathbb E\left[U_i\middle|\mathcal H_{\gamma_i}\right]
\leq H^5.
\end{align*}
The tower property therefore yields \[\mathbb E[D_i^2\mid\mathcal H_{\rho_{i-1}}]\leq H^5.\]

For any deterministic $m\in[K]$,  setting $v = H^5m$ and $c = H^3$ in Lemma \ref{Freedman} gives
\[
\mathbb P\left(
\sum_{i=1}^{m}D_i
>
2\sqrt{H^5m\log\left(\frac{K}{\delta}\right)}
+
2H^3\log\left(\frac{K}{\delta}\right)
\right)
\leq\frac{\delta}{K}.
\]
Taking a union bound over all deterministic $m\in[K]$, with probability at least $1-\delta$, the preceding inequality holds simultaneously for every $m\in[K]$. Evaluating it at $m=L$ gives
\begin{align*}
\sum_{i=1}^{L}U_i
&\leq
\sum_{i=1}^{L}\mathbb E\left[U_i\middle|\mathcal H_{\gamma_i}\right]
+
2\sqrt{H^5L\log\left(\frac{K}{\delta}\right)}
+
2H^3\log\left(\frac{K}{\delta}\right)\\
&\leq
H^2L
+
2\sqrt{H^5L\log\left(\frac{K}{\delta}\right)}
+
2H^3\log\left(\frac{K}{\delta}\right)\\
&\leq
2H^2L+3H^3\log\left(\frac{K}{\delta}\right),
\end{align*}
where the last inequality uses $2\sqrt{ab}\leq a+b$ with $a=H^2L$ and $b=H^3\log\left(\frac{K}{\delta}\right)$. Finally,
\[
\sum_{i=1}^{L}U_i
=
\sum_{k\in\cK_\tau}\sum_{h=\tau_k}^{H}
\mathbb V_{s_h^k,a_h^k,h}V_{h+1}^{\pi^k}
=
\sum_{(k,h)\in\cJ_\tau}
\mathbb V_{s_h^k,a_h^k,h}V_{h+1}^{\pi^k}.
\]
This proves the result.
\end{proof}

\section{Properties of Value Function Estimates}
\label{property}
Following Lemma B.4 and Lemma B.2 in \cite{he2023nearly}, we obtain the following two lemmas.

\begin{lemma}\label{lemma:optimistic}
On the event $\cE_1\cap\cE_2$, for any $(s,a,h,k)\in\sah\times[K]$, we have
\[
Q_{k,h}(s,a)\geq Q_h^\star(s,a)\geq\check{Q}_{k,h}(s,a),\qquad V_{k,h}(s)\geq V_h^\star(s)\geq\check{V}_{k,h}(s).
\]
\end{lemma}
In particular, on the event $\cE_1\cap\cE_2$, Lemma \ref{lemma:optimistic}, together with the nonnegativity of $\check V_{k,h}$ and the upper bound $V_{k,h}\leq H$, implies that
\[
0\leq \check V_{k,h}(s)\leq V_h^\star(s)\leq V_{k,h}(s)\leq H
\]
for every $(s,h,k)\in\mathcal S\times[H]\times[K]$. Together with the terminal identities $\check V_{k,H+1}=V_{H+1}^\star=V_{k,H+1}=0$, this implies that, on $\cE_1\cap\cE_2$,
\begin{equation}
\label{clipequal}
\check V_{k,h}^{\mathrm{clip}}(s)=\check V_{k,h}(s),\qquad V_{k,h}^{\mathrm{clip}}(s)=V_{k,h}(s)
\end{equation}
for every $(s,h,k)\in\mathcal S\times[H+1]\times[K]$.
\begin{lemma}\label{lemma:variance}
On the event $\cE_1\cap\cE_2$, for any $(h,k)\in[H]\times[K]$, the estimated variance satisfies
$$\left|\bar{\mathbb V}_{s_h^k,a_h^k,h}V_{k,h+1}-\mathbb V_{s_h^k,a_h^k,h}V_{k,h+1}\right|\leq E_{k,h}.$$
\end{lemma}

The following lemma provides a bound on the sum of the bonuses over the episodes in $\cK_h$.

\begin{lemma}[Formal Statement of Lemma \ref{lemma:sketch-sumbonus}]
\label{sumbonus}
On the event $\mathcal E_1\cap\mathcal E_2$, suppose that $\lambda=1/H^2$ and let $\iota=\log(1+K/(d\lambda H))=\log(1+T/d)$. For any $h\in[H]$ and parameters $\beta'\geq1$ and $C\geq1$, the sum of the bonuses over the episodes in $\cK_h$ can be bounded as
\begin{align*}
&\sum_{k\in\cK_h}\min\left\{\beta'\sqrt{\bphi(s_h^k,a_h^k)^\top\bSigma_{k,h}^{-1}\bphi(s_h^k,a_h^k)},C\right\}\\
&\quad\leq10d^4H^3C\iota+10\beta'd^4H^2\iota+2\beta'\sqrt{d\iota\sum_{k\in\cK_h}\left(\sigma_{k,h}^2+H\right)}.
\end{align*}
\end{lemma}

\begin{proof}[Proof of Lemma \ref{sumbonus}]
On $\mathcal E_1\cap\mathcal E_2$, Lemma \ref{lemma:variance} and the definition of $D_{k,h}$ imply that $\sigma_{k,h}^2\geq H$, and hence $\sigma_{k,h}$ is well defined and nonnegative.

Fix any step $h\in[H]$. If $\cK_h$ is empty, the conclusion holds immediately. Otherwise, write the elements of $\cK_h$ in increasing order as
\[
k_1<k_2<\cdots<k_{|\cK_h|}.
\]
Since $\beta'\geq1$, the sum of the bonuses is bounded by
\begin{align}
&\sum_{i=1}^{|\cK_h|}\min\left\{\beta'\sqrt{\bphi(s_h^{k_i},a_h^{k_i})^\top\bSigma_{k_i,h}^{-1}\bphi(s_h^{k_i},a_h^{k_i})},C\right\}\notag\\
&\leq\sum_{i=1}^{|\cK_h|}\beta'\min\left\{\sqrt{\bphi(s_h^{k_i},a_h^{k_i})^\top\bSigma_{k_i,h}^{-1}\bphi(s_h^{k_i},a_h^{k_i})},1\right\}\notag\\
&\quad+C\sum_{i=1}^{|\cK_h|}\mathbb I\left[\sqrt{\bphi(s_h^{k_i},a_h^{k_i})^\top\bSigma_{k_i,h}^{-1}\bphi(s_h^{k_i},a_h^{k_i})}\geq1\right].\label{eq:sumb1}
\end{align}

For the first term in \Cref{eq:sumb1}, for any $i\in[|\cK_h|]$, define
\[
\bSigma_i'=\lambda\Ib_d+\sum_{j=1}^{i-1}(\bar{\sigma}_{j,h}')^{-2}\bphi(s_h^{k_j},a_h^{k_j})\bphi(s_h^{k_j},a_h^{k_j})^\top
\]
and
\[
\bar{\sigma}_{i,h}'=\max\left\{\sigma_{k_i,h},\sqrt H,\sqrt{2d^3H^2}\|\bphi(s_h^{k_i},a_h^{k_i})\|_{(\bSigma_i')^{-1}}^{1/2}\right\}.
\]
We use induction to prove that, for any $i\in[|\cK_h|]$, $\bSigma_i'\preceq\bSigma_{k_i,h}$ and $\bar{\sigma}_{i,h}'\geq\bar{\sigma}_{k_i,h}$.

For $i=1$, we have
\[
\bSigma_1'=\lambda\Ib_d\preceq\bSigma_{k_1,h}.
\]
Therefore, $(\bSigma_1')^{-1}\succeq\bSigma_{k_1,h}^{-1}$, and hence
\begin{align*}
\bar{\sigma}_{1,h}'&=\max\left\{\sigma_{k_1,h},\sqrt H,\sqrt{2d^3H^2}\|\bphi(s_h^{k_1},a_h^{k_1})\|_{(\bSigma_1')^{-1}}^{1/2}\right\}\\
&\geq\max\left\{\sigma_{k_1,h},\sqrt H,\sqrt{2d^3H^2}\|\bphi(s_h^{k_1},a_h^{k_1})\|_{\bSigma_{k_1,h}^{-1}}^{1/2}\right\}\geq\bar{\sigma}_{k_1,h}.
\end{align*}

Suppose that the conclusion holds for $1,2,\ldots,i-1$. Since $\bar{\sigma}_{j,h}'\geq\bar{\sigma}_{k_j,h}$ for every $j\in[i-1]$, we have
\begin{align*}
\bSigma_i'&=\lambda\Ib_d+\sum_{j=1}^{i-1}(\bar{\sigma}_{j,h}')^{-2}\bphi(s_h^{k_j},a_h^{k_j})\bphi(s_h^{k_j},a_h^{k_j})^\top\\
&\preceq\lambda\Ib_d+\sum_{j=1}^{i-1}(\bar{\sigma}_{k_j,h})^{-2}\bphi(s_h^{k_j},a_h^{k_j})\bphi(s_h^{k_j},a_h^{k_j})^\top\preceq\bSigma_{k_i,h},
\end{align*}
where the last inequality follows because $k_1<k_2<\cdots<k_i$ and $\bSigma_{k_i,h}$ contains all covariance updates before episode $k_i$, including those corresponding to episodes $k_1,\ldots,k_{i-1}$. Consequently, $(\bSigma_i')^{-1}\succeq\bSigma_{k_i,h}^{-1}$, and thus
\begin{align*}
\bar{\sigma}_{i,h}'&=\max\left\{\sigma_{k_i,h},\sqrt H,\sqrt{2d^3H^2}\|\bphi(s_h^{k_i},a_h^{k_i})\|_{(\bSigma_i')^{-1}}^{1/2}\right\}\\
&\geq\max\left\{\sigma_{k_i,h},\sqrt H,\sqrt{2d^3H^2}\|\bphi(s_h^{k_i},a_h^{k_i})\|_{\bSigma_{k_i,h}^{-1}}^{1/2}\right\}\geq\bar{\sigma}_{k_i,h}.
\end{align*}
This completes the induction.

Since $\bSigma_i'\preceq\bSigma_{k_i,h}$, we have $(\bSigma_i')^{-1}\succeq\bSigma_{k_i,h}^{-1}$ and therefore
\begin{align}
&\sum_{i=1}^{|\cK_h|}\beta'\min\left\{\sqrt{\bphi(s_h^{k_i},a_h^{k_i})^\top\bSigma_{k_i,h}^{-1}\bphi(s_h^{k_i},a_h^{k_i})},1\right\}\notag\\
&\leq\sum_{i=1}^{|\cK_h|}\beta'\min\left\{\sqrt{\bphi(s_h^{k_i},a_h^{k_i})^\top(\bSigma_i')^{-1}\bphi(s_h^{k_i},a_h^{k_i})},1\right\}\notag\\
&\leq10\beta'd^4H^2\iota+2\beta'\sqrt{d\iota\sum_{i=1}^{|\cK_h|}\left(\sigma_{k_i,h}^2+H\right)},\label{sumb11}
\end{align}
where the last inequality follows from Lemma \ref{lemma:keysum:temp} with $\ab_i=\bphi(s_h^{k_i},a_h^{k_i})$, $\sigma_i=\sigma_{k_i,h}$, $\alpha=\sqrt H$, $\gamma=\sqrt{2d^3H^2}$, and $\hat\bSigma_i=\bSigma_i'$.

We next bound the second term in \Cref{eq:sumb1}. Let
\[
\{i_1,i_2,\ldots,i_m\}=\left\{i\in[|\cK_h|]:\sqrt{\bphi(s_h^{k_i},a_h^{k_i})^\top\bSigma_{k_i,h}^{-1}\bphi(s_h^{k_i},a_h^{k_i})}\geq1\right\},
\]
where $i_1<i_2<\cdots<i_m$. Set $\bSigma_0''=\lambda\Ib_d$, and for every $t\in[m]$, define
\[
\bSigma_t''=\lambda\Ib_d+\sum_{j=1}^t\bar{\sigma}_{k_{i_j},h}^{-2}\bphi(s_h^{k_{i_j}},a_h^{k_{i_j}})\bphi(s_h^{k_{i_j}},a_h^{k_{i_j}})^\top.
\]
Since $k_{i_1}<k_{i_2}<\cdots<k_{i_m}$, the matrix $\bSigma_{k_{i_t},h}$ contains all covariance updates corresponding to $k_{i_1},\ldots,k_{i_{t-1}}$. Therefore,
\[
\bSigma_{t-1}''\preceq\bSigma_{k_{i_t},h}.
\]
Consequently, $(\bSigma_{t-1}'')^{-1}\succeq\bSigma_{k_{i_t},h}^{-1}$, and hence
\begin{align}
&\sum_{t=1}^m\bphi(s_h^{k_{i_t}},a_h^{k_{i_t}})^\top(\bSigma_{t-1}'')^{-1}\bphi(s_h^{k_{i_t}},a_h^{k_{i_t}})\notag\\
&\geq\sum_{t=1}^m\bphi(s_h^{k_{i_t}},a_h^{k_{i_t}})^\top\bSigma_{k_{i_t},h}^{-1}\bphi(s_h^{k_{i_t}},a_h^{k_{i_t}})\geq m.\label{sumb121}
\end{align}

For every selected index $i_t$, the definition of $i_t$ gives
\[
\|\bphi(s_h^{k_{i_t}},a_h^{k_{i_t}})\|_{\bSigma_{k_{i_t},h}^{-1}}\geq1.
\]
By the definition of $\bar{\sigma}_{k,h}$, it follows that
\[
\bar{\sigma}_{k_{i_t},h}^2\geq2d^3H^2.
\]
Moreover, since $\bSigma_{t-1}''\succeq\lambda\Ib_d$, $\lambda=1/H^2$, and $\|\bphi(s_h^{k_{i_t}},a_h^{k_{i_t}})\|_2\leq1$, we have
\[
\bar{\sigma}_{k_{i_t},h}^{-2}\bphi(s_h^{k_{i_t}},a_h^{k_{i_t}})^\top(\bSigma_{t-1}'')^{-1}\bphi(s_h^{k_{i_t}},a_h^{k_{i_t}})\leq\frac{H^2}{2d^3H^2}=\frac{1}{2d^3}\leq1.
\]
In addition, by the definition of $\bar{\sigma}_{k,h}^2$, together with $\bSigma_{k,h}\succeq\lambda\Ib_d$ and $\|\bphi(s_h^k,a_h^k)\|_2\leq1$, we have
\[
\bar{\sigma}_{k,h}^2\leq\frac{5d^3H^2}{\sqrt\lambda}.
\]
Moreover, since $\bar{\sigma}_{k_{i_t},h}^2\geq H$ and $\|\bphi(s_h^{k_{i_t}},a_h^{k_{i_t}})\|_2\leq1$, we have
\[
\left\|\bar{\sigma}_{k_{i_t},h}^{-1}\bphi(s_h^{k_{i_t}},a_h^{k_{i_t}})\right\|_2^2\leq\frac{1}{H}.
\]
Therefore, since $m\leq|\cK_h|\leq K$, Lemma~\ref{Lemma:abba} gives
\begin{align*}
&\sum_{t=1}^m\bar{\sigma}_{k_{i_t},h}^{-2}\bphi(s_h^{k_{i_t}},a_h^{k_{i_t}})^\top(\bSigma_{t-1}'')^{-1}\bphi(s_h^{k_{i_t}},a_h^{k_{i_t}})\leq2d\log\left(1+\frac{m}{d\lambda H}\right)
\leq2d\iota.
\end{align*}
Combining this bound with $\bar{\sigma}_{k,h}^2\leq5d^3H^2/\sqrt{\lambda}$ and $\lambda=H^{-2}$, we obtain
\begin{align}
&\sum_{t=1}^m\bphi(s_h^{k_{i_t}},a_h^{k_{i_t}})^\top(\bSigma_{t-1}'')^{-1}\bphi(s_h^{k_{i_t}},a_h^{k_{i_t}})\notag\\
&\leq\frac{5d^3H^2}{\sqrt{\lambda}}\sum_{t=1}^m\bar{\sigma}_{k_{i_t},h}^{-2}\bphi(s_h^{k_{i_t}},a_h^{k_{i_t}})^\top(\bSigma_{t-1}'')^{-1}\bphi(s_h^{k_{i_t}},a_h^{k_{i_t}})\notag\\
&\leq10d^4H^3\iota.\label{sumb122}
\end{align}
Combining \Cref{sumb121} and \Cref{sumb122} gives
\[
m\leq10d^4H^3\iota.
\]
Substituting this bound and \Cref{sumb11} into \Cref{eq:sumb1}, we obtain
\begin{align*}
&\sum_{i=1}^{|\cK_h|}\min\left\{\beta'\sqrt{\bphi(s_h^{k_i},a_h^{k_i})^\top\bSigma_{k_i,h}^{-1}\bphi(s_h^{k_i},a_h^{k_i})},C\right\}\\
&\leq10d^4H^3C\iota+10\beta'd^4H^2\iota+2\beta'\sqrt{d\iota\sum_{i=1}^{|\cK_h|}\left(\sigma_{k_i,h}^2+H\right)}.
\end{align*}
Since $\{k_1,\ldots,k_{|\cK_h|}\}=\cK_h$, this is equivalent to
\begin{align*}
&\sum_{k\in\cK_h}\min\left\{\beta'\sqrt{\bphi(s_h^k,a_h^k)^\top\bSigma_{k,h}^{-1}\bphi(s_h^k,a_h^k)},C\right\}\\
&\leq10d^4H^3C\iota+10\beta'd^4H^2\iota+2\beta'\sqrt{d\iota\sum_{k\in\cK_h}\left(\sigma_{k,h}^2+H\right)}.
\end{align*}
This completes the proof of Lemma \ref{sumbonus}.
\end{proof}

\begin{lemma}\label{lemma:transition}
On the events $\cE_1\cap\cE_2\cap\cE_4$ in Lemma \ref{concentration}, we have
\begin{align*}
&\sum_{(k,h)\in\cJ_\tau}\mathbb P_{s_h^k,a_h^k,h}\left(V_{k,h+1}-V_{h+1}^{\pi^k}\right)\notag\\
&\leq40d^4H^6\iota+40\beta d^4H^4\iota+8\beta\sqrt{dH^3\iota\sum_{(k,h)\in\cJ_\tau}\left(\sigma_{k,h}^2+H\right)}+2\sqrt{2H^5L\log(KH/\delta)}.
\end{align*}
\end{lemma}

\begin{proof}[Proof of Lemma \ref{lemma:transition}]
For any $h'\in[H]$ and $k\in[K]$, if $k_{\mathrm{last}} \ge 1$, then we have
\begin{align}
&V_{k,h'}(s_{h'}^k)-V_{h'}^{\pi^k}(s_{h'}^k)=Q_{k,h'}(s_{h'}^k,a_{h'}^k)-Q_{h'}^{\pi^k}(s_{h'}^k,a_{h'}^k)\notag\\
&\leq\min\left\{\hat{\wb}_{k_{\mathrm{last}},h'}^\top\bphi(s_{h'}^k,a_{h'}^k)+\beta\sqrt{\bphi(s_{h'}^k,a_{h'}^k)^\top\bSigma_{k_{\mathrm{last}},h'}^{-1}\bphi(s_{h'}^k,a_{h'}^k)},H\right\}-\mathbb P_{s_{h'}^k,a_{h'}^k,h'}V_{k,h'+1}\notag\\
&\quad+\mathbb P_{s_{h'}^k,a_{h'}^k,h'}\left(V_{k,h'+1}-V_{h'+1}^{\pi^k}\right)\notag\\
&\leq\mathbb P_{s_{h'}^k,a_{h'}^k,h'}\left(V_{k,h'+1}-V_{h'+1}^{\pi^k}\right)+2\min\left\{\beta\sqrt{\bphi(s_{h'}^k,a_{h'}^k)^\top\bSigma_{k_{\mathrm{last}},h'}^{-1}\bphi(s_{h'}^k,a_{h'}^k)},H\right\}\notag\\
&\leq\mathbb P_{s_{h'}^k,a_{h'}^k,h'}\left(V_{k,h'+1}-V_{h'+1}^{\pi^k}\right)+4\min\left\{\beta\sqrt{\bphi(s_{h'}^k,a_{h'}^k)^\top\bSigma_{k,h'}^{-1}\bphi(s_{h'}^k,a_{h'}^k)},H\right\}\notag\\
&=V_{k,h'+1}(s_{h'+1}^k)-V_{h'+1}^{\pi^k}(s_{h'+1}^k)+\left(\mathbb P_{s_{h'}^k,a_{h'}^k,h'}-\mathbbm{1}_{s_{h'+1}^k}\right)\left(V_{k,h'+1}-V_{h'+1}^{\pi^k}\right)\notag\\
&\quad+4\min\left\{\beta\sqrt{\bphi(s_{h'}^k,a_{h'}^k)^\top\bSigma_{k,h'}^{-1}\bphi(s_{h'}^k,a_{h'}^k)},H\right\}.\label{eq:erroru}
\end{align}
Here, the first inequality holds due to the update rule of $Q_{k,h}(s_h^k,a_h^k)$ in Line 11 of \Cref{algorithm1} and the Bellman equation in \Cref{eq_Bellman}, the second inequality holds due to the event $\cE_2$ in Lemma \ref{concentration} and the fact that $V_{k,h'+1} = V_{k_{\mathrm{last}},h'+1}$, and the last inequality holds due to the update condition in Line 10 of \Cref{algorithm1} and Lemma \ref{lemma:det}.

\Cref{eq:erroru} also holds when $k_{\mathrm{last}}=0$, i.e., before the first policy update. Define $\hat{\wb}_{0,h}=0$ for every $h\in[H]$. By the definition of $\beta$ and $\lambda=H^{-2}$, we have $\beta\geq\sqrt d$. Moreover, $\langle\bphi(s,a),\btheta_h(\cS)\rangle=1$ and $\|\btheta_h(\cS)\|_2\leq\sqrt d$ imply $\|\bphi(s,a)\|_2\geq1/\sqrt d$, and hence
\[
\beta\|\bphi(s,a)\|_{\bSigma_{0,h}^{-1}}=\frac{\beta}{\sqrt{\lambda}}\|\bphi(s,a)\|_2\geq H.
\]
Therefore, before the first policy update,
\[
Q_{k,h}(s,a)=Q_{0,h}(s,a)=H=\min\left\{r_h(s,a)+\hat{\wb}_{0,h}^{\top}\bphi(s,a)+\beta\|\bphi(s,a)\|_{\bSigma_{0,h}^{-1}},H\right\}.
\]
Thus, the first inequality still holds. For the second inequality, since $\min\{\beta\|\bphi(s,a)\|_{\bSigma_{0,h}^{-1}},H\}=H$ and $\mathbb P_{s,a,h}V_{k,h+1}\geq0$, we have
\begin{align*}
    \min\left\{\hat{\wb}_{0,h}^{\top}\bphi(s,a)+\beta\|\bphi(s,a)\|_{\bSigma_{0,h}^{-1}},H\right\}-\mathbb P_{s,a,h}V_{k,h+1}
&=H-\mathbb P_{s,a,h}V_{k,h+1}\\
&\leq2\min\left\{\beta\|\bphi(s,a)\|_{\bSigma_{0,h}^{-1}},H\right\}.
\end{align*}

Furthermore, since the update condition has not been triggered, $\bSigma_{k,h}\succeq\bSigma_{0,h}$ and $\det(\bSigma_{k,h})<2\det(\bSigma_{0,h})$ for every $h\in[H]$. Lemma~\ref{lemma:det} therefore gives
\[
\|\bphi(s,a)\|_{\bSigma_{0,h}^{-1}}\leq\sqrt{2}\|\bphi(s,a)\|_{\bSigma_{k,h}^{-1}}\leq2\|\bphi(s,a)\|_{\bSigma_{k,h}^{-1}},
\]
which proves the third inequality. Hence, all steps in \Cref{eq:erroru} remain valid when $k_{\mathrm{last}}=0$.

For notational convenience, define
\[
\epsilon_{k,h}:=\left(\mathbb P_{s_h^k,a_h^k,h}-\mathbbm{1}_{s_{h+1}^k}\right)\left(V_{k,h+1}-V_{h+1}^{\pi^k}\right)
\]
and
\[
b_{k,h}:=\min\left\{\beta\sqrt{\bphi(s_h^k,a_h^k)^\top\bSigma_{k,h}^{-1}\bphi(s_h^k,a_h^k)},H\right\}.
\]
For every $k\in\cK_\tau$ and $h\in\{\tau_k,\ldots,H\}$, recursively applying \Cref{eq:erroru} from step $h+1$ to step $H$ and using $V_{k,H+1}=V_{H+1}^{\pi^k}=0$ gives
\[
V_{k,h+1}(s_{h+1}^k)-V_{h+1}^{\pi^k}(s_{h+1}^k)\leq\sum_{h'=h+1}^{H}\epsilon_{k,h'}+4\sum_{h'=h+1}^{H}b_{k,h'}.
\]
Consequently,
\begin{align*}
\mathbb P_{s_h^k,a_h^k,h}\left(V_{k,h+1}-V_{h+1}^{\pi^k}\right)
=V_{k,h+1}(s_{h+1}^k)-V_{h+1}^{\pi^k}(s_{h+1}^k)+\epsilon_{k,h}\leq\sum_{h'=h}^{H}\epsilon_{k,h'}+4\sum_{h'=h}^{H}b_{k,h'}.
\end{align*}
Summing the preceding inequality over $h=\tau_k,\ldots,H$ and exchanging the order of summation gives
\begin{align}
\sum_{h=\tau_k}^{H}\mathbb P_{s_h^k,a_h^k,h}\left(V_{k,h+1}-V_{h+1}^{\pi^k}\right)
\leq\sum_{h=\tau_k}^{H}(h-\tau_k+1)\epsilon_{k,h}+4\sum_{h=\tau_k}^{H}(h-\tau_k+1)b_{k,h}.\label{eq:transition-weighted}
\end{align}
On $\cE_1\cap\cE_2\cap\cE_4$, by \Cref{clipequal}, we have $V_{k,h}^{\mathrm{clip}}=V_{k,h}$. Then the definition of $\cE_4$ yields
\begin{align}
&\sum_{(k,h)\in\cJ_\tau}(h-\tau_k+1)
\left(\mathbb P_{s_h^k,a_h^k,h}-\mathbbm 1_{s_{h+1}^k}\right)
\left(V_{k,h+1}-V_{h+1}^{\pi^k}\right)\leq 2\sqrt{2H^5L\log\left(\frac{KH}{\delta}\right)}.
\label{eq:E4-original-values}
\end{align}
Summing \Cref{eq:transition-weighted} over all $k\in\cK_\tau$ and applying \Cref{eq:E4-original-values}, we obtain
\begin{align}
&\sum_{(k,h)\in\cJ_\tau}\mathbb P_{s_h^k,a_h^k,h}\left(V_{k,h+1}-V_{h+1}^{\pi^k}\right)\leq2\sqrt{2H^5L\log(KH/\delta)}+\sum_{h=1}^{H}\sum_{k\in\cK_h}4(h-\tau_k+1)b_{k,h}.\label{eq:transition-event}
\end{align}
For every $h\in[H]$, since $h-\tau_k+1\leq H$ for every $k\in\cK_h$, Lemma \ref{sumbonus} gives
\begin{align*}
\sum_{k\in\cK_h}4(h-\tau_k+1)b_{k,h}&\leq4H\sum_{k\in\cK_h}b_{k,h}\\
&\leq40d^4H^5\iota+40\beta d^4H^3\iota+8\beta H\sqrt{d\iota\sum_{k\in\cK_h}\left(\sigma_{k,h}^2+H\right)}.
\end{align*}
Summing the preceding inequality over $h\in[H]$ and applying the Cauchy--Schwarz inequality gives
\begin{align*}
&\sum_{h=1}^{H}\sum_{k\in\cK_h}4(h-\tau_k+1)\min\left\{\beta\sqrt{\bphi(s_h^k,a_h^k)^\top\bSigma_{k,h}^{-1}\bphi(s_h^k,a_h^k)},H\right\}\\
&\leq40d^4H^6\iota+40\beta d^4H^4\iota+8\beta H\sum_{h=1}^{H}\sqrt{d\iota\sum_{k\in\cK_h}\left(\sigma_{k,h}^2+H\right)}\\
&\leq40d^4H^6\iota+40\beta d^4H^4\iota+8\beta\sqrt{dH^3\iota\sum_{(k,h)\in\cJ_\tau}\left(\sigma_{k,h}^2+H\right)}.
\end{align*}
Substituting this bound into \Cref{eq:transition-event} completes the proof.
\end{proof}

In addition, for the cumulative difference between the optimistic estimate $V_{k,h}(s)$ and the pessimistic estimate $\check{V}_{k,h}(s)$ over the random suffixes, we have the following lemma.

\begin{lemma}\label{lemma:transition1}
On the events $\cE_1\cap\cE_2\cap\cE_5$ in Lemma \ref{concentration}, we have
\begin{align*}
\sum_{(k,h)\in\cJ_\tau}&\mathbb P_{s_h^k,a_h^k,h}\left(V_{k,h+1}-\check{V}_{k,h+1}\right)\leq80d^4H^6\iota+40\left(\beta+\bar{\beta}\right)d^4H^4\iota\notag\\
&\quad +8\left(\beta+\bar{\beta}\right)\sqrt{dH^3\iota\sum_{(k,h)\in\cJ_\tau}\left(\sigma_{k,h}^2+H\right)}+2\sqrt{2H^5L\log(KH/\delta)}.
\end{align*}
\end{lemma}

\begin{proof}[Proof of Lemma \ref{lemma:transition1}]
For each step $h'\in[H]$ and episode $k\in[K]$, if $k_{\mathrm{last}} \geq 1$, we have
\begin{align}
&V_{k,h'}(s_{h'}^k)-\check{V}_{k,h'}(s_{h'}^k)\notag\\
&\leq Q_{k,h'}(s_{h'}^k,a_{h'}^k)-\check{Q}_{k,h'}(s_{h'}^k,a_{h'}^k)\notag\\
&\leq \min\left\{r_{h'}(s_{h'}^k,a_{h'}^k)+\hat{\wb}_{k_{\mathrm{last}},h'}^\top\bphi(s_{h'}^k,a_{h'}^k)+\beta\sqrt{\bphi(s_{h'}^k,a_{h'}^k)^\top\bSigma_{k_{\mathrm{last}},h'}^{-1}\bphi(s_{h'}^k,a_{h'}^k)},H\right\}\notag\\
&\quad-\max\left\{r_{h'}(s_{h'}^k,a_{h'}^k)+\check{\wb}_{k_{\mathrm{last}},h'}^\top\bphi(s_{h'}^k,a_{h'}^k)-\bar{\beta}\sqrt{\bphi(s_{h'}^k,a_{h'}^k)^\top\bSigma_{k_{\mathrm{last}},h'}^{-1}\bphi(s_{h'}^k,a_{h'}^k)},0\right\}\notag\\
&\leq \min\left\{r_{h'}(s_{h'}^k,a_{h'}^k)+\mathbb P_{s_{h'}^k,a_{h'}^k,h'}V_{k,h'+1}+2\beta\sqrt{\bphi(s_{h'}^k,a_{h'}^k)^\top\bSigma_{k_{\mathrm{last}},h'}^{-1}\bphi(s_{h'}^k,a_{h'}^k)},H\right\}\notag\\
&\quad-\max\left\{r_{h'}(s_{h'}^k,a_{h'}^k)+\mathbb P_{s_{h'}^k,a_{h'}^k,h'}\check{V}_{k,h'+1}-2\bar{\beta}\sqrt{\bphi(s_{h'}^k,a_{h'}^k)^\top\bSigma_{k_{\mathrm{last}},h'}^{-1}\bphi(s_{h'}^k,a_{h'}^k)},0\right\}\notag\\
&\leq \mathbb P_{s_{h'}^k,a_{h'}^k,h'}\left(V_{k,h'+1}-\check{V}_{k,h'+1}\right)+2\min\left\{\beta\sqrt{\bphi(s_{h'}^k,a_{h'}^k)^\top\bSigma_{k_{\mathrm{last}},h'}^{-1}\bphi(s_{h'}^k,a_{h'}^k)},H\right\}\notag\\
&\quad+2\min\left\{\bar{\beta}\sqrt{\bphi(s_{h'}^k,a_{h'}^k)^\top\bSigma_{k_{\mathrm{last}},h'}^{-1}\bphi(s_{h'}^k,a_{h'}^k)},H\right\}\notag\\
&\leq V_{k,h'+1}(s_{h'+1}^k)-\check{V}_{k,h'+1}(s_{h'+1}^k)+\left(\mathbb P_{s_{h'}^k,a_{h'}^k,h'}-\mathbbm{1}_{s_{h'+1}^k}\right)\left(V_{k,h'+1}-\check{V}_{k,h'+1}\right) +\notag\\
&4\min\Big\{\beta\sqrt{\bphi(s_{h'}^k,a_{h'}^k)^\top\bSigma_{k,h'}^{-1}\bphi(s_{h'}^k,a_{h'}^k)},H\Big\}+4\min\Big\{\bar{\beta}\sqrt{\bphi(s_{h'}^k,a_{h'}^k)^\top\bSigma_{k,h'}^{-1}\bphi(s_{h'}^k,a_{h'}^k)},H\Big\}.\label{eq:009}
\end{align}
Here, the first inequality follows from
\[
\check{V}_{k,h'}(s_{h'}^k)=\max_{a}\check{Q}_{k,h'}(s_{h'}^k,a)
\geq \check{Q}_{k,h'}(s_{h'}^k,a_{h'}^k).
\]
The second inequality follows from the update rules of $Q_{k,h}$ and $\check{Q}_{k,h}$ in Lines 11--12 of Algorithm 1. The third inequality follows from the events $\mathcal E_1$ and $\mathcal E_2$ in Lemma \ref{concentration} and the fact that $V_{k,h'+1} = V_{k_{\mathrm{last}},h'+1}$ and $\check V_{k,h'+1} = \check V_{k_{\mathrm{last}},h'+1}$ . The fourth inequality uses $0\leq r_{h'}(s_{h'}^k,a_{h'}^k)\leq1$, $0\leq V_{k,h'+1},\check V_{k,h'+1}\leq H$, and the clipping inequalities \[\min\{x+2y,H\}\leq x+2\min\{y,H\},\qquad \max\{x-2y,0\}\geq x-2\min\{y,H\},\] which hold for every $x\in[0,2H]$ and $y\geq0$.
The last inequality follows from the update condition in Line 10 of Algorithm 1 and Lemma \ref{lemma:det}.

\Cref{eq:009} also holds when $k_{\mathrm{last}}=0$, i.e., before the first policy update. Define $\hat{\wb}_{0,h}=\check{\wb}_{0,h}=0$ for every $h\in[H]$. By the definitions of $\beta$ and $\bar{\beta}$ and $\lambda=H^{-2}$, we have $\beta,\bar{\beta}\geq\sqrt d$. Moreover, $\langle\bphi(s,a),\btheta_h(\cS)\rangle=1$ and $\|\btheta_h(\cS)\|_2\leq\sqrt d$ imply $\|\bphi(s,a)\|_2\geq1/\sqrt d$, and hence
\[
\beta\|\bphi(s,a)\|_{\bSigma_{0,h}^{-1}}\geq H,\qquad \bar{\beta}\|\bphi(s,a)\|_{\bSigma_{0,h}^{-1}}\geq H.
\]
Therefore, before the first policy update,
\[
Q_{k,h}(s,a)=Q_{0,h}(s,a)=\min\left\{r_h(s,a)+\hat{\wb}_{0,h}^{\top}\bphi(s,a)+\beta\|\bphi(s,a)\|_{\bSigma_{0,h}^{-1}},H\right\}=H
\]
and
\[
\check Q_{k,h}(s,a)=\check Q_{0,h}(s,a)=\max\left\{r_h(s,a)+\check{\wb}_{0,h}^{\top}\bphi(s,a)-\bar{\beta}\|\bphi(s,a)\|_{\bSigma_{0,h}^{-1}},0\right\}=0,
\]
where the last equality uses $0\leq r_h(s,a)\leq H$. Thus, the first inequality in \Cref{eq:009} follows from the same greedy-action argument, and the second inequality holds with equality. Moreover, since $0\leq V_{k,h+1},\check V_{k,h+1}\leq H$, we have
\[
\left|\hat{\wb}_{0,h}^{\top}\bphi(s,a)-\mathbb P_{s,a,h}V_{k,h+1}\right|\leq H\leq\beta\|\bphi(s,a)\|_{\bSigma_{0,h}^{-1}}
\]
and
\[
\left|\check{\wb}_{0,h}^{\top}\bphi(s,a)-\mathbb P_{s,a,h}\check V_{k,h+1}\right|\leq H\leq\bar{\beta}\|\bphi(s,a)\|_{\bSigma_{0,h}^{-1}},
\]
which proves the third inequality. The fourth inequality then follows from the same clipping inequalities used above. Finally, since the update condition has not been triggered,
\[
\bSigma_{k,h}\succeq\bSigma_{0,h},\qquad \det(\bSigma_{k,h})<2\det(\bSigma_{0,h}),
\]
and Lemma~\ref{lemma:det} gives
\[
\|\bphi(s,a)\|_{\bSigma_{0,h}^{-1}}\leq\sqrt{2}\|\bphi(s,a)\|_{\bSigma_{k,h}^{-1}}\leq2\|\bphi(s,a)\|_{\bSigma_{k,h}^{-1}}.
\]
This proves the last step in \Cref{eq:009} for $k_{\mathrm{last}}=0$.

For notational convenience, define
\[
\check{\epsilon}_{k,h}:=\left(\mathbb P_{s_h^k,a_h^k,h}-\mathbbm{1}_{s_{h+1}^k}\right)\left(V_{k,h+1}-\check{V}_{k,h+1}\right),
\]
\[
b_{k,h}:=\min\left\{\beta\sqrt{\bphi(s_h^k,a_h^k)^\top\bSigma_{k,h}^{-1}\bphi(s_h^k,a_h^k)},H\right\},
\]
and
\[
\check{b}_{k,h}:=\min\left\{\bar{\beta}\sqrt{\bphi(s_h^k,a_h^k)^\top\bSigma_{k,h}^{-1}\bphi(s_h^k,a_h^k)},H\right\}.
\]
For every $k\in\cK_\tau$ and $h\in\{\tau_k,\ldots,H\}$, recursively applying \Cref{eq:009} from step $h+1$ to step $H$ and using $V_{k,H+1}=\check{V}_{k,H+1}=0$ gives
\[
V_{k,h+1}(s_{h+1}^k)-\check{V}_{k,h+1}(s_{h+1}^k)\leq\sum_{h'=h+1}^{H}\check{\epsilon}_{k,h'}+4\sum_{h'=h+1}^{H}\left(b_{k,h'}+\check{b}_{k,h'}\right).
\]
Consequently,
\begin{align*}
\mathbb P_{s_h^k,a_h^k,h}\left(V_{k,h+1}-\check{V}_{k,h+1}\right)
&=V_{k,h+1}(s_{h+1}^k)-\check{V}_{k,h+1}(s_{h+1}^k)+\check{\epsilon}_{k,h}\\
&\leq\sum_{h'=h}^{H}\check{\epsilon}_{k,h'}+4\sum_{h'=h}^{H}\left(b_{k,h'}+\check{b}_{k,h'}\right).
\end{align*}
Summing the preceding inequality over $h=\tau_k,\ldots,H$ and exchanging the order of summation gives
\begin{align}
&\sum_{h=\tau_k}^{H}\mathbb P_{s_h^k,a_h^k,h}\left(V_{k,h+1}-\check{V}_{k,h+1}\right)\leq\sum_{h=\tau_k}^{H}(h-\tau_k+1)\check{\epsilon}_{k,h}+4\sum_{h=\tau_k}^{H}(h-\tau_k+1)\left(b_{k,h}+\check{b}_{k,h}\right).\label{eq:transition1-weighted}
\end{align}
On $\cE_1\cap\cE_2\cap\cE_5$, by \Cref{clipequal}, we have $\check V_{k,h}^{\mathrm{clip}}=\check V_{k,h}, V_{k,h}^{\mathrm{clip}}=V_{k,h}$. Then the definition of $\cE_5$ yields
\begin{align}
&\sum_{(k,h)\in\cJ_\tau}(h-\tau_k+1)
\left(\mathbb P_{s_h^k,a_h^k,h}-\mathbbm 1_{s_{h+1}^k}\right)
\left(V_{k,h+1}-\check V_{k,h+1}\right)\leq 2\sqrt{2H^5L\log\left(\frac{KH}{\delta}\right)}.
\label{eq:E5-original-values}
\end{align}
Summing \Cref{eq:transition1-weighted} over all $k\in\cK_\tau$ and applying \Cref{eq:E5-original-values}, we obtain
\begin{align}
&\sum_{(k,h)\in\cJ_\tau}\mathbb P_{s_h^k,a_h^k,h}\left(V_{k,h+1}-\check V_{k,h+1}\right)\notag\\
&\leq 2\sqrt{2H^5L\log\left(\frac{KH}{\delta}\right)}
+\sum_{h=1}^{H}\sum_{k\in\cK_h}4(h-\tau_k+1)b_{k,h}
+\sum_{h=1}^{H}\sum_{k\in\cK_h}4(h-\tau_k+1)\check b_{k,h}.
\label{eq:transition1-event}
\end{align}
For every $h\in[H]$, Lemma \ref{sumbonus} and $h-\tau_k+1\leq H$ give
\begin{align*}
\sum_{k\in\cK_h}4(h-\tau_k+1)b_{k,h}\leq40d^4H^5\iota+40\beta d^4H^3\iota+8\beta H\sqrt{d\iota\sum_{k\in\cK_h}\left(\sigma_{k,h}^2+H\right)}
\end{align*}
and
\begin{align*}
\sum_{k\in\cK_h}4(h-\tau_k+1)\check{b}_{k,h}\leq40d^4H^5\iota+40\bar{\beta}d^4H^3\iota+8\bar{\beta}H\sqrt{d\iota\sum_{k\in\cK_h}\left(\sigma_{k,h}^2+H\right)}.
\end{align*}
Summing the preceding two inequalities over $h\in[H]$ and applying the Cauchy--Schwarz inequality gives
\begin{align*}
&\sum_{h=1}^{H}\sum_{k\in\cK_h}4(h-\tau_k+1)b_{k,h}+\sum_{h=1}^{H}\sum_{k\in\cK_h}4(h-\tau_k+1)\check{b}_{k,h}\\
&\leq80d^4H^6\iota+40\left(\beta+\bar{\beta}\right)d^4H^4\iota+8\left(\beta+\bar{\beta}\right)\sqrt{dH^3\iota\sum_{(k,h)\in\cJ_\tau}\left(\sigma_{k,h}^2+H\right)}.
\end{align*}
Substituting this bound into \Cref{eq:transition1-event} completes the proof.
\end{proof}

\section{Proof of \texorpdfstring{\Cref{regret}}{Theorem 3.1}}
\label{regretproof}
We first prove Lemma \ref{expectedrmain}, which connects the expected regret with the value difference at the first realized positive-gap step.
\begin{lemma}[Restatement of Lemma \ref{lemma:sketch-first-gap}]
\label{expectedrmain}
For LSVI-UCB++ with $K$ episodes and $T=HK$ steps, the expected regret satisfies
\begin{align*}
\mathbb E\left[\operatorname{Regret}(T)\right]&=\mathbb E\left[\sum_{k=1}^K\sum_{h=1}^{H}\Delta_h(s_h^k,a_h^k)\right]=\mathbb E\left[\sum_{k\in\cK_\tau}\left(V_{\tau_k}^\star-V_{\tau_k}^{\pi^k}\right)(s_{\tau_k}^k)\right].
\end{align*}
\end{lemma}
\begin{proof}[Proof of Lemma \ref{expectedrmain}]
For any fixed episode $k\in[K]$, the performance-difference identity gives
\begin{align}
\left(V_1^\star-V_1^{\pi^k}\right)(s_1^k)=\mathbb E\left[\sum_{h=1}^{H}\Delta_h(s_h^k,a_h^k)\middle|\mathcal F_{k,1}\right].\label{eq:expected-regret-gap}
\end{align}
Indeed, for any $h\in[H]$, since $a_h^k=\pi_h^k(s_h^k)$, we have
\begin{align*}
\left(V_h^\star-V_h^{\pi^k}\right)(s_h^k)&=V_h^\star(s_h^k)-Q_h^\star(s_h^k,a_h^k)+\left(Q_h^\star-Q_h^{\pi^k}\right)(s_h^k,a_h^k)\\
&=\Delta_h(s_h^k,a_h^k)+\mathbb E\left[\left(V_{h+1}^\star-V_{h+1}^{\pi^k}\right)(s_{h+1}^k)\middle|\mathcal F_{k,h}\right],
\end{align*}
where the second equality follows from the Bellman equation in \Cref{eq_Bellman}. Recursively applying this equality from $h=1$ to $h=H$ proves \Cref{eq:expected-regret-gap}.

We next express the cumulative gaps using the first realized positive-gap step $\tau_k$. By the definition of $\tau_k$, on the event $\{\tau_k\leq H\}$, we have $\Delta_h(s_h^k,a_h^k)=0$ for every $h<\tau_k$. On the event $\{\tau_k=H+1\}$, we have $\Delta_h(s_h^k,a_h^k)=0$ for every $h\in[H]$. Therefore, pathwise,
\begin{align}
\sum_{h=1}^{H}\Delta_h(s_h^k,a_h^k)=\mathbb I[\tau_k\leq H]\sum_{h=\tau_k}^{H}\Delta_h(s_h^k,a_h^k).\label{eq:gap-random-suffix}
\end{align}

For any fixed $t\in[H]$, recursively applying the preceding Bellman decomposition from step $t$ to step $H$ gives
\begin{align*}
\left(V_t^\star-V_t^{\pi^k}\right)(s_t^k)=\mathbb E\left[\sum_{h=t}^{H}\Delta_h(s_h^k,a_h^k)\middle|\mathcal F_{k,t}\right].
\end{align*}
Since $\{\tau_k=t\}\in\mathcal F_{k,t}$, multiplying the preceding equality by $\mathbb I[\tau_k=t]$ and taking expectations yields
\begin{align*}
\mathbb E\left[\mathbb I[\tau_k=t]\left(V_t^\star-V_t^{\pi^k}\right)(s_t^k)\right]=\mathbb E\left[\mathbb I[\tau_k=t]\sum_{h=t}^{H}\Delta_h(s_h^k,a_h^k)\right].
\end{align*}
Summing the preceding equality over $t\in[H]$ gives
\begin{align*}
\mathbb E\left[\mathbb I[\tau_k\leq H]\left(V_{\tau_k}^\star-V_{\tau_k}^{\pi^k}\right)(s_{\tau_k}^k)\right]=\mathbb E\left[\mathbb I[\tau_k\leq H]\sum_{h=\tau_k}^{H}\Delta_h(s_h^k,a_h^k)\right].
\end{align*}
Combining this equality with \Cref{eq:gap-random-suffix}, we obtain
\begin{align}
\mathbb E\left[\sum_{h=1}^{H}\Delta_h(s_h^k,a_h^k)\right]=\mathbb E\left[\mathbb I[\tau_k\leq H]\left(V_{\tau_k}^\star-V_{\tau_k}^{\pi^k}\right)(s_{\tau_k}^k)\right].\label{eq:episode-first-gap}
\end{align}

Finally, combining \Cref{eq:expected-regret-gap} and \Cref{eq:episode-first-gap} and summing over all episodes gives
\begin{align*}
\mathbb E\left[\operatorname{Regret}(T)\right]&=\mathbb E\left[\sum_{k=1}^K\left(V_1^\star-V_1^{\pi^k}\right)(s_1^k)\right]\\
&=\mathbb E\left[\sum_{k=1}^K\sum_{h=1}^{H}\Delta_h(s_h^k,a_h^k)\right]\\
&=\mathbb E\left[\sum_{k=1}^K\mathbb I[\tau_k\leq H]\left(V_{\tau_k}^\star-V_{\tau_k}^{\pi^k}\right)(s_{\tau_k}^k)\right]\\
&=\mathbb E\left[\sum_{k\in\cK_\tau}\left(V_{\tau_k}^\star-V_{\tau_k}^{\pi^k}\right)(s_{\tau_k}^k)\right].
\end{align*}
This completes the proof.
\end{proof}

We then prove Lemma \ref{LEMMA: TOTAL-ESTIMATE-VARIANCE}, which bounds the sum of the estimated variances over the random suffix index set $\cJ_\tau$.

\begin{lemma}
\label{LEMMA: TOTAL-ESTIMATE-VARIANCE}
On the event $\cE_1\cap\cE_2\cap\cE_4\cap\cE_5\cap\cE_6$, the sum of the estimated variances over $\cJ_\tau$ is upper bounded as
\[
\sum_{(k,h)\in\cJ_\tau}\sigma_{k,h}^2\leq O\left(H^2L+d^{10}H^{9}\iota_1^3\right).
\]
\end{lemma}

\begin{proof}[Proof of Lemma \ref{LEMMA: TOTAL-ESTIMATE-VARIANCE}]
According to the definition of $\sigma_{k,h}^2$ in \Cref{eq:variance}, we have
\begin{align}
&\sum_{(k,h)\in\cJ_\tau}\sigma_{k,h}^2=\sum_{(k,h)\in\cJ_\tau}\left(\bar{\mathbb V}_{s_h^k,a_h^k,h}V_{k,h+1}+E_{k,h}+D_{k,h}+H\right)\notag\\
&\leq H^2L+\underbrace{\sum_{(k,h)\in\cJ_\tau}\left(\bar{\mathbb V}_{s_h^k,a_h^k,h}V_{k,h+1}-\mathbb V_{s_h^k,a_h^k,h}V_{k,h+1}\right)}_{I_1}+\underbrace{\sum_{(k,h)\in\cJ_\tau}E_{k,h}}_{I_2}\notag\\
&\quad+\underbrace{\sum_{(k,h)\in\cJ_\tau}D_{k,h}}_{I_3}+\underbrace{\sum_{(k,h)\in\cJ_\tau}\left(\mathbb V_{s_h^k,a_h^k,h}V_{k,h+1}-\mathbb V_{s_h^k,a_h^k,h}V_{h+1}^{\pi^k}\right)}_{I_4}+\underbrace{\sum_{(k,h)\in\cJ_\tau}\mathbb V_{s_h^k,a_h^k,h}V_{h+1}^{\pi^k}}_{I_5}.\label{eq:011}
\end{align}
Here, we used
\[
\sum_{(k,h)\in\cJ_\tau}H=H|\cJ_\tau|=H\sum_{k\in\cK_\tau}(H-\tau_k+1)\leq H^2L.
\]

For the term $I_1$, according to Lemma \ref{lemma:variance}, it is upper bounded by
\begin{align}
I_1&=\sum_{(k,h)\in\cJ_\tau}\left(\bar{\mathbb V}_{s_h^k,a_h^k,h}V_{k,h+1}-\mathbb V_{s_h^k,a_h^k,h}V_{k,h+1}\right)\leq\sum_{(k,h)\in\cJ_\tau}E_{k,h}=I_2.\label{eq:012}
\end{align}

For the term $I_2$, by the definition of $E_{k,h}$, we have
\begin{align}
I_2&=\sum_{h=1}^{H}\sum_{k\in\cK_h}\bigg(\min\left\{\tilde{\beta}\sqrt{\bphi(s_h^k,a_h^k)^\top\bSigma_{k,h}^{-1}\bphi(s_h^k,a_h^k)},H^2\right\}\notag\\
&\quad+\min\left\{2H\bar{\beta}\sqrt{\bphi(s_h^k,a_h^k)^\top\bSigma_{k,h}^{-1}\bphi(s_h^k,a_h^k)},H^2\right\}\bigg)\notag\\
&\leq20d^4H^6\iota+\left(10\tilde{\beta}+20\bar{\beta}\right)d^4H^4\iota+\left(2\tilde{\beta}+4\bar{\beta}\right)\sqrt{dH^3\iota\sum_{(k,h)\in\cJ_\tau}\left(\sigma_{k,h}^2+H\right)},\label{eq:013}
\end{align}
where the inequality follows by applying Lemma \ref{sumbonus} to $\cK_h$ for every $h\in[H]$, summing the resulting inequalities over $h\in[H]$, and applying the Cauchy--Schwarz inequality.

For the term $I_3$, by the definition of $D_{k,h}$, we have
\begin{align}
I_3&=\sum_{(k,h)\in\cJ_\tau}\min\left\{4d^3H^2\left(\hat{\wb}_{k,h}^\top\bphi(s_h^k,a_h^k)-\check{\wb}_{k,h}^\top\bphi(s_h^k,a_h^k)+2\bar{\beta}\|\bphi(s_h^k,a_h^k)\|_{\bSigma_{k,h}^{-1}}\right),d^3H^3\right\}\notag\\
&\leq\sum_{(k,h)\in\cJ_\tau}\min\left\{4d^3H^2\left(\mathbb P_{s_h^k,a_h^k,h}\left(V_{k,h+1}-\check{V}_{k,h+1}\right)+4\bar{\beta}\|\bphi(s_h^k,a_h^k)\|_{\bSigma_{k,h}^{-1}}\right),d^3H^3\right\}\notag\\
&\leq4d^3H^2\sum_{(k,h)\in\cJ_\tau}\mathbb P_{s_h^k,a_h^k,h}\left(V_{k,h+1}-\check{V}_{k,h+1}\right)\notag\\
&\quad+\sum_{h=1}^{H}\sum_{k\in\cK_h}\min\left\{16d^3H^2\bar{\beta}\sqrt{\bphi(s_h^k,a_h^k)^\top\bSigma_{k,h}^{-1}\bphi(s_h^k,a_h^k)},d^3H^3\right\}\notag\\
&\leq4d^3H^2\sum_{(k,h)\in\cJ_\tau}\mathbb P_{s_h^k,a_h^k,h}\left(V_{k,h+1}-\check{V}_{k,h+1}\right)+160d^7H^7\iota+160\bar{\beta}d^7H^5\iota\notag\\
&\quad+32d^3H^2\bar{\beta}\sqrt{dH\iota\sum_{(k,h)\in\cJ_\tau}\left(\sigma_{k,h}^2+H\right)}\notag\\
&\leq480d^7H^8\iota+320\left(\beta+\bar{\beta}\right)d^7H^6\iota+64d^3H^2\left(\beta+\bar{\beta}\right)\sqrt{dH^3\iota\sum_{(k,h)\in\cJ_\tau}\left(\sigma_{k,h}^2+H\right)}\notag\\
&\quad+16d^3\sqrt{H^9L\log(KH/\delta)}.\label{eq:014}
\end{align}
Here, the first inequality follows from the event $\cE_1$ in Lemma \ref{concentration}. The second inequality follows from Lemma \ref{lemma:optimistic}, which gives $V_{k,h+1}(s)\geq V_{h+1}^\star(s)\geq\check{V}_{k,h+1}(s)$. The third inequality follows by applying Lemma \ref{sumbonus} to $\cK_h$ for every $h\in[H]$ and then applying the Cauchy--Schwarz inequality. The last inequality follows from Lemma \ref{lemma:transition1} and the facts that $H\geq1$ and $\bar{\beta}\leq\beta+\bar{\beta}$.

For the term $I_4$, we have
\begin{align}
I_4&=\sum_{(k,h)\in\cJ_\tau}\Big(\mathbb P_{s_h^k,a_h^k,h}\left(V_{k,h+1}\right)^2-\big(\mathbb P_{s_h^k,a_h^k,h}V_{k,h+1}\big)^2-\mathbb P_{s_h^k,a_h^k,h}(V_{h+1}^{\pi^k})^2+\left(\mathbb P_{s_h^k,a_h^k,h}V_{h+1}^{\pi^k}\right)^2\Big)\notag\\
&\leq\sum_{(k,h)\in\cJ_\tau}\left(\mathbb P_{s_h^k,a_h^k,h}\left(V_{k,h+1}\right)^2-\mathbb P_{s_h^k,a_h^k,h}\left(V_{h+1}^{\pi^k}\right)^2\right)\notag\\
&\leq2H\sum_{(k,h)\in\cJ_\tau}\mathbb P_{s_h^k,a_h^k,h}\left(V_{k,h+1}-V_{h+1}^{\pi^k}\right)\notag\\
&\leq80d^4H^7\iota+80\beta d^4H^5\iota+16H\beta\sqrt{dH^3\iota\sum_{(k,h)\in\cJ_\tau}\big(\sigma_{k,h}^2+H\big)}+8\sqrt{H^7L\log(KH/\delta)}.\label{eq:015}
\end{align}
Here, the first inequality follows because Lemma \ref{lemma:optimistic} gives $V_{h+1}^{\pi^k}(s')\leq V_{h+1}^\star(s')\leq V_{k,h+1}(s')$. The second inequality follows from $0\leq V_{k,h+1}(s'),V_{h+1}^{\pi^k}(s')\leq H$. The last inequality follows from Lemma \ref{lemma:transition}.

By the event $\cE_6$ in Lemma \ref{concentration}, the term $I_5$ satisfies
\begin{align}
I_5\leq2H^2L+3H^3\log(K/\delta).\label{eq:016}
\end{align}

Let
\[
S_\tau:=\sum_{(k,h)\in\cJ_\tau}\sigma_{k,h}^2,
\qquad
\bar S_\tau:=\sum_{(k,h)\in\cJ_\tau}\left(\sigma_{k,h}^2+H\right).
\]
Since $|\cJ_\tau|\leq HL$, we have
\[
\bar S_\tau=S_\tau+H|\cJ_\tau|\leq S_\tau+H^2L.
\]
Substituting \Cref{eq:012}, \Cref{eq:013}, \Cref{eq:014}, \Cref{eq:015}, and \Cref{eq:016} into \Cref{eq:011}, and absorbing lower-order terms, gives
\begin{align}
S_\tau &\leq O\bigg(H^2L+d^7H^8\iota+d^7H^6(\beta+\bar{\beta})\iota+d^4H^4(\tilde{\beta}+\bar{\beta})\iota+d^3\sqrt{H^9L\log(KH/\delta)}\notag\\
&\qquad+H^3\log(K/\delta)+\Big(\tilde{\beta}+\bar{\beta}+d^3H^2(\beta+\bar{\beta})+H\beta\Big)\sqrt{dH^3\iota\left(S_\tau+H^2L\right)}\bigg).\label{eq:variance-self-bound}
\end{align}
Since $\lambda=1/H^2$, the definitions of $\beta$, $\tilde{\beta}$, and $\bar{\beta}$ imply
\[
\beta=O\left(\sqrt d\,\iota_1\right),\qquad \bar{\beta}=O\left(d^{3/2}H\iota_1\right),\qquad \tilde{\beta}=O\left(d^{3/2}H^2\iota_1\right).
\]
Moreover, $\iota\leq O(\iota_1)$ and $\log(KH/\delta)\leq O(\iota_1)$. Therefore,
\[
(\tilde{\beta}+\bar{\beta})\sqrt{dH^3\iota}\leq O\left(d^2H^{\frac{7}{2}}\iota_1^{\frac{3}{2}}\right),
\]
\[
d^3H^2(\beta+\bar{\beta})\sqrt{dH^3\iota}\leq O\left(d^5H^{\frac{9}{2}}\iota_1^{\frac{3}{2}}\right),
\]
and
\[
H\beta\sqrt{dH^3\iota}\leq O\left(dH^{\frac{5}{2}}\iota_1^{\frac{3}{2}}\right).
\]
Hence the last term in \Cref{eq:variance-self-bound} is bounded by
\[
O\left(d^5H^{\frac{9}{2}}\iota_1^{\frac{3}{2}}\sqrt{S_\tau+H^2L}\right)\leq O\left(d^5H^{\frac{9}{2}}\iota_1^{\frac{3}{2}}\sqrt{S_\tau}+d^5H^{\frac{11}{2}}\iota_1^{\frac{3}{2}}\sqrt L\right).
\]
Using $2\sqrt{ab}\leq a+b$, we have
\[
d^3\sqrt{H^9L\iota_1}\leq O\left(H^2L+d^6H^7\iota_1\right)
\]
and
\[
d^5H^{\frac{11}{2}}\iota_1^{\frac{3}{2}}\sqrt L\leq O\left(H^2L+d^{10}H^9\iota_1^3\right).
\]
The remaining deterministic terms in \Cref{eq:variance-self-bound} are also bounded by $O(d^{10}H^9\iota_1^3)$. Consequently,
\[
S_\tau\leq O\left(H^2L+d^{10}H^9\iota_1^3+d^5H^{\frac{9}{2}}\iota_1^{\frac{3}{2}}\sqrt{S_\tau}\right).
\]
Using the fact that $x\leq a\sqrt{x}+b$ implies $x\leq a^2+2b$, we conclude that
\[
S_\tau\leq O\left(H^2L+d^{10}H^9\iota_1^3\right).
\]
This completes the proof.
\end{proof}

With the lemma established above, we next bound the cumulative optimistic error at the first realized positive-gap steps.

\begin{lemma}[Formal Statement of Lemma \ref{lemma:sketch-global-error}]\label{lemma:global-error}
On the event $\bigcap_{i=1}^6\cE_i$, we have
\begin{align}
&\sum_{k\in\cK_\tau}\left(Q_{k,\tau_k}(s_{\tau_k}^k,a_{\tau_k}^k)-Q_{\tau_k}^{\pi^k}(s_{\tau_k}^k,a_{\tau_k}^k)\right)\leq O\left(d^6H^5\iota_1^3+d\sqrt{H^3\iota_1^3L}\right).\label{eq:global-upper}
\end{align}
Moreover,
\begin{align}
\sum_{k\in\cK_\tau}\left(Q_{k,\tau_k}(s_{\tau_k}^k,a_{\tau_k}^k)-Q_{\tau_k}^{\pi^k}(s_{\tau_k}^k,a_{\tau_k}^k)\right)\leq O\left(\frac{d^2H^3\iota_1^3}{\dmin}+d^6H^5\iota_1^3\right).\label{eq:global-self-bound}
\end{align}
\end{lemma}

\begin{proof}[Proof of Lemma \ref{lemma:global-error}]
For every $k\in\cK_\tau$, the definition of $\tau_k$ gives
\[
\Delta_{\tau_k}(s_{\tau_k}^k,a_{\tau_k}^k)\geq\dmin.
\]
Since $\pi^k$ is greedy with respect to $Q_{k,h}$, we have
\[
Q_{k,\tau_k}(s_{\tau_k}^k,a_{\tau_k}^k)=V_{k,\tau_k}(s_{\tau_k}^k).
\]
Moreover, since $a_{\tau_k}^k=\pi_{\tau_k}^k(s_{\tau_k}^k)$, we have
\[
Q_{\tau_k}^{\pi^k}(s_{\tau_k}^k,a_{\tau_k}^k)=V_{\tau_k}^{\pi^k}(s_{\tau_k}^k).
\]
Therefore, by Lemma \ref{lemma:optimistic},
\begin{align}
Q_{k,\tau_k}(s_{\tau_k}^k,a_{\tau_k}^k)-Q_{\tau_k}^{\pi^k}(s_{\tau_k}^k,a_{\tau_k}^k)
&=V_{k,\tau_k}(s_{\tau_k}^k)-V_{\tau_k}^{\pi^k}(s_{\tau_k}^k)\notag\\
&\geq V_{\tau_k}^\star(s_{\tau_k}^k)-V_{\tau_k}^{\pi^k}(s_{\tau_k}^k)\notag\\
&\geq \Delta_{\tau_k}(s_{\tau_k}^k,a_{\tau_k}^k)\notag\\
&\geq \dmin. \label{eq:global-lower-single}
\end{align}
Here, the second inequality follows from
\[
V_{\tau_k}^{\pi^k}(s_{\tau_k}^k)
=Q_{\tau_k}^{\pi^k}(s_{\tau_k}^k,a_{\tau_k}^k)
\leq Q_{\tau_k}^\star(s_{\tau_k}^k,a_{\tau_k}^k),
\]
and hence
\[
V_{\tau_k}^\star(s_{\tau_k}^k)-V_{\tau_k}^{\pi^k}(s_{\tau_k}^k)
\geq V_{\tau_k}^\star(s_{\tau_k}^k)-Q_{\tau_k}^\star(s_{\tau_k}^k,a_{\tau_k}^k)
=\Delta_{\tau_k}(s_{\tau_k}^k,a_{\tau_k}^k).
\]
The last inequality follows from the definition of \(\tau_k\) and \(\dmin\).

Summing \Cref{eq:global-lower-single} over all $k\in\cK_\tau$ gives
\begin{align}
\dmin L\leq\sum_{k\in\cK_\tau}\left(Q_{k,\tau_k}(s_{\tau_k}^k,a_{\tau_k}^k)-Q_{\tau_k}^{\pi^k}(s_{\tau_k}^k,a_{\tau_k}^k)\right).\label{eq:global-lower}
\end{align}

For notational convenience, fix an arbitrary dummy action $a^\dagger\in\mathcal A$, set $a_{H+1}^k:=a^\dagger$ for every $k\in[K]$, and define
\[
Q_{k,H+1}(s,a):=Q_{H+1}^{\pi^k}(s,a):=0
\]
for every $(s,a)\in\mathcal S\times\mathcal A$. Under these conventions, together with $V_{k,H+1}=V_{H+1}^{\pi^k}=0$, we have
\[
V_{k,h+1}(s_{h+1}^k)=Q_{k,h+1}(s_{h+1}^k,a_{h+1}^k),\qquad V_{h+1}^{\pi^k}(s_{h+1}^k)=Q_{h+1}^{\pi^k}(s_{h+1}^k,a_{h+1}^k)
\]
for every $k\in[K]$ and $h\in[H]$. Therefore, by \Cref{eq:erroru}, for every $k\in[K]$ and $h\in[H]$, we have
\begin{align}
&Q_{k,h}(s_h^k,a_h^k)-Q_h^{\pi^k}(s_h^k,a_h^k)\notag\\
&\leq Q_{k,h+1}(s_{h+1}^k,a_{h+1}^k)-Q_{h+1}^{\pi^k}(s_{h+1}^k,a_{h+1}^k)+\left(\mathbb P_{s_h^k,a_h^k,h}-\mathbbm{1}_{s_{h+1}^k}\right)\left(V_{k,h+1}-V_{h+1}^{\pi^k}\right)\notag\\
&\quad+4\min\left\{\beta\sqrt{\bphi(s_h^k,a_h^k)^\top\bSigma_{k,h}^{-1}\bphi(s_h^k,a_h^k)},H\right\}.
\label{eq:tele}
\end{align}
For every $k\in\cK_\tau$, summing \Cref{eq:tele} over $h=\tau_k,\tau_k+1,\ldots,H$ gives
\begin{align*}
&Q_{k,\tau_k}(s_{\tau_k}^k,a_{\tau_k}^k)-Q_{\tau_k}^{\pi^k}(s_{\tau_k}^k,a_{\tau_k}^k)\notag\\
&\leq\sum_{h=\tau_k}^{H}\left(\mathbb P_{s_h^k,a_h^k,h}-\mathbbm{1}_{s_{h+1}^k}\right)(V_{k,h+1}-V_{h+1}^{\pi^k})+\sum_{h=\tau_k}^{H}4\min\Big\{\beta\sqrt{\bphi(s_h^k,a_h^k)^\top\bSigma_{k,h}^{-1}\bphi(s_h^k,a_h^k)},H\Big\}.
\end{align*}
Summing the preceding inequality over all $k\in\cK_\tau$ yields
\begin{align}
&\sum_{k\in\cK_\tau}\left(Q_{k,\tau_k}(s_{\tau_k}^k,a_{\tau_k}^k)-Q_{\tau_k}^{\pi^k}(s_{\tau_k}^k,a_{\tau_k}^k)\right)\notag\\
&\leq\sum_{(k,h)\in\cJ_\tau}\left(\mathbb P_{s_h^k,a_h^k,h}-\mathbbm{1}_{s_{h+1}^k}\right)\left(V_{k,h+1}-V_{h+1}^{\pi^k}\right)\notag\\
&\quad+\sum_{h=1}^{H}\sum_{k\in\cK_h}4\min\left\{\beta\sqrt{\bphi(s_h^k,a_h^k)^\top\bSigma_{k,h}^{-1}\bphi(s_h^k,a_h^k)},H\right\}.\label{eq:global-error-decomposition}
\end{align}
For the first term on the right-hand side of \Cref{eq:global-error-decomposition}, on the event $\bigcap_{i=1}^6\cE_i$, \Cref{clipequal} gives $V_{k,h}^{\mathrm{clip}}=V_{k,h}$. Therefore, the event $\cE_3$ in Lemma \ref{concentration} yields
\begin{align}
\sum_{(k,h)\in\cJ_\tau}\left(\mathbb P_{s_h^k,a_h^k,h}-\mathbbm{1}_{s_{h+1}^k}\right)\left(V_{k,h+1}-V_{h+1}^{\pi^k}\right)\leq2\sqrt{2H^3L\log(KH/\delta)}.\label{eq:global-martingale}
\end{align}
For the second term on the right-hand side of \Cref{eq:global-error-decomposition}, applying Lemma \ref{sumbonus} to $\cK_h$ for every $h\in[H]$ gives
\begin{align}
&\sum_{h=1}^{H}\sum_{k\in\cK_h}4\min\left\{\beta\sqrt{\bphi(s_h^k,a_h^k)^\top\bSigma_{k,h}^{-1}\bphi(s_h^k,a_h^k)},H\right\}\notag\\
&\leq40d^4H^5\iota+40\beta d^4H^3\iota+8\beta\sum_{h=1}^{H}\sqrt{d\iota\sum_{k\in\cK_h}\left(\sigma_{k,h}^2+H\right)}\notag\\
&\leq40d^4H^5\iota+40\beta d^4H^3\iota+8\beta\sqrt{dH\iota\sum_{(k,h)\in\cJ_\tau}\left(\sigma_{k,h}^2+H\right)},\label{eq:global-bonus}
\end{align}
where the second inequality follows from the Cauchy--Schwarz inequality.

By Lemma \ref{LEMMA: TOTAL-ESTIMATE-VARIANCE},
\[
\sum_{(k,h)\in\cJ_\tau}\left(\sigma_{k,h}^2+H\right)=\sum_{(k,h)\in\cJ_\tau}\sigma_{k,h}^2+H|\cJ_\tau|\leq O\left(H^2L+d^{10}H^{9}\iota_1^3\right).
\]
Since $\lambda=1/H^2$, we have $\beta=O(\sqrt d\,\iota_1)$, $\iota\leq O(\iota_1)$, and $\log(KH/\delta)\leq O(\iota_1)$. Substituting the preceding bound into \Cref{eq:global-martingale} and \Cref{eq:global-bonus}, and then substituting these two inequalities into \Cref{eq:global-error-decomposition}, gives
\begin{align*}
    \sum_{k\in\cK_\tau}\left(Q_{k,\tau_k}(s_{\tau_k}^k,a_{\tau_k}^k)-Q_{\tau_k}^{\pi^k}(s_{\tau_k}^k,a_{\tau_k}^k)\right) &\leq O\left(d^6H^5\iota_1^3+\beta\sqrt{dH^3\iota_1L}\right) \\
    &\leq O\left(d^6H^5\iota_1^3+d\sqrt{H^3\iota_1^3L}\right),
\end{align*}

which proves \Cref{eq:global-upper}.

Combining \Cref{eq:global-lower} and the equation above, we have
\begin{align*}
&\sum_{k\in\cK_\tau}\left(Q_{k,\tau_k}(s_{\tau_k}^k,a_{\tau_k}^k)-Q_{\tau_k}^{\pi^k}(s_{\tau_k}^k,a_{\tau_k}^k)\right)\\
&\leq O\left(d^6H^5\iota_1^3+\beta \sqrt{\frac{dH^3\iota_1}{\dmin}\sum_{k\in\cK_\tau}\left(Q_{k,\tau_k}(s_{\tau_k}^k,a_{\tau_k}^k)-Q_{\tau_k}^{\pi^k}(s_{\tau_k}^k,a_{\tau_k}^k)\right)}\right).
\end{align*}
Using the fact that $x\leq a\sqrt{x}+b$ implies $x\leq a^2+2b$, we obtain
\[
\sum_{k\in\cK_\tau}\left(Q_{k,\tau_k}(s_{\tau_k}^k,a_{\tau_k}^k)-Q_{\tau_k}^{\pi^k}(s_{\tau_k}^k,a_{\tau_k}^k)\right)\leq O\left(\frac{d\beta^2H^3\iota_1}{\dmin}+d^6H^5\iota_1^3\right).
\]
This proves \Cref{eq:global-self-bound} and completes the proof.
\end{proof}

On the event $\bigcap_{i=1}^6\cE_i$, Lemma \ref{lemma:optimistic} gives
\begin{align*}
\sum_{k\in\cK_\tau}\left(V_{\tau_k}^\star-V_{\tau_k}^{\pi^k}\right)(s_{\tau_k}^k)&\leq\sum_{k\in\cK_\tau}\left(V_{k,\tau_k}-V_{\tau_k}^{\pi^k}\right)(s_{\tau_k}^k)\\
&=\sum_{k\in\cK_\tau}\left(Q_{k,\tau_k}(s_{\tau_k}^k,a_{\tau_k}^k)-Q_{\tau_k}^{\pi^k}(s_{\tau_k}^k,a_{\tau_k}^k)\right)\\
&\leq O\left(\frac{d\beta^2H^3\iota_1}{\dmin}+d^6H^5\iota_1^3\right),
\end{align*}
where the last inequality follows from Lemma \ref{lemma:global-error}.

Let $\delta=1/(19T)$. By Lemma \ref{concentration}, the event $\cE:=\bigcap_{i=1}^6\cE_i$ holds with probability at least $1-19\delta=1-1/T$. Therefore, by Lemma \ref{expectedrmain} and the fact that 
$$\sum_{k\in\cK_\tau}\left(V_{\tau_k}^\star-V_{\tau_k}^{\pi^k}\right)(s_{\tau_k}^k)\leq HL \leq HK= T,$$
we have
\begin{align*}
\mathbb E\left[\operatorname{Regret}(T)\right]
&=\mathbb E\left[\sum_{k\in\cK_\tau}\left(V_{\tau_k}^\star-V_{\tau_k}^{\pi^k}\right)(s_{\tau_k}^k)\right] \notag\\
&\leq O\left(\frac{d\beta^2H^3\iota_1}{\dmin}+d^6H^5\iota_1^3\right)+T\mathbb P(\cE^c) \notag\\
&\leq O\left(\frac{d\beta^2H^3\iota_1}{\dmin}+d^6H^5\iota_1^3\right) \\
&\leq O\left(\left(\frac{d^2H^3}{\dmin}+d^6H^5\right)\log^3\left(1+dHK\right)\right) \notag.
\end{align*}
Therefore,
\[
\mathbb E\left[\operatorname{Regret}(T)\right]=\widetilde O\left(\frac{d^2H^3}{\dmin}+d^6H^5\right).
\]
This completes the proof.

\section{Proof of \texorpdfstring{\Cref{ellipsoidlower}}{Theorem 3.2}}
\label{sec: ellipsoidal-lower-bound}

In this section, we prove the lower bound for persistent ellipsoidal optimistic algorithms. For a fixed episode budget $K$, the structural conditions are parameterized by the regularization, confidence radius, regression-weight bounds, covariance staleness, and calibration-failure probability.

For the fixed-initial-state instances used below, the standard performance-difference telescoping argument---the same argument used at the beginning of the proof of Lemma~\ref{expectedrmain}---gives
\begin{align}\mathbb E[\operatorname{Regret}(T)] = \mathbb E\!\left[\sum_{k=1}^K\sum_{h=1}^H \Delta_h(s_h^k,a_h^k)\right]. \label{eq:persistent-performance-difference}\end{align}
We only use this identity after taking expectation; realized action counts are not identified pathwise with pseudo-regret.

\subsection{Persistent Ellipsoidal Optimism}

\begin{definition}[Persistent ellipsoidal optimistic algorithm]
\label{def:persistent-ellipsoid}
Fix an episode budget $K$, a regularization parameter $\lambda>0$, confidence radius $\beta>0$, deterministic constants $0<\alpha_-\leq\alpha_+<\infty$, covariance-staleness constant $\kappa\geq1$, and calibration-failure level $\delta_{\mathrm{cal}}\in(0,1/16]$. An algorithm satisfies persistent ellipsoidal optimism with parameters $(\lambda,\beta,\alpha_-,\alpha_+,\kappa,\delta_{\mathrm{cal}})$ if the following conditions hold over the first $K$ episodes.

\textbf{(i) Predictable, immutable regression weights.}
For every episode $k$ and step $h$,
\[\Lambda_{k,h}^{\mathrm{all}} = \lambda I_d+\sum_{\ell<k}\alpha_{\ell,h}\phi(s_h^\ell,a_h^\ell)\phi(s_h^\ell,a_h^\ell)^\top,\]
where $\alpha_{\ell,h}$ is chosen before $s_{h+1}^\ell$ is observed, remains unchanged thereafter, and satisfies
\[\alpha_-\leq\alpha_{\ell,h}\leq\alpha_+ \qquad\text{on every sample path.}\]
The weighting rule may otherwise depend arbitrarily on the past history and on the algorithm's internal randomization.

\textbf{(ii) Complete episode-level updates and bounded covariance staleness.}
At each genuine update index $u$, corresponding to an episode in which the algorithm recomputes its value estimates, the algorithm uses the current cumulative covariance
\[\Lambda_{u,h}=\Lambda_{u,h}^{\mathrm{all}}\]
and performs the backward update over all steps. Every historical regression target at step $h$ is evaluated with the continuation value $V_{u,h+1}$ from this same update. Values used for decision making in an episode are fixed before the random transitions of that episode occur and remain fixed throughout the episode. Let $\tau(k)\leq k$ be the latest genuine update whose values are used in episode $k$, with $\tau(k)=0$ if no genuine update has yet occurred, and set $\Lambda_{0,h}=\lambda I_d$. We require
\begin{align}\Lambda_{\tau(k),h} \preceq \Lambda_{k,h}^{\mathrm{all}} \preceq \kappa\Lambda_{\tau(k),h} \qquad\forall k,h. \label{eq:persistent-covariance-growth}\end{align}

\textbf{(iii) Bellman ridge predictor.}
At each genuine update $u$, the algorithm follows one of two reward conventions throughout the run: 
\begin{align}
\widehat{\mathcal B}_{u,h}(s,a)
&=
 r_h(s,a)+\phi(s,a)^\top\Lambda_{u,h}^{-1}
 \sum_{\ell<u}\alpha_{\ell,h}\phi(s_h^\ell,a_h^\ell)V_{u,h+1}(s_{h+1}^\ell),
\label{eq:persistent-predictor-outside-reward}
\\
\widehat{\mathcal B}_{u,h}(s,a)
&=
\phi(s,a)^\top\Lambda_{u,h}^{-1}
\sum_{\ell<u}\alpha_{\ell,h}\phi(s_h^\ell,a_h^\ell)
\bigl(r_h(s_h^\ell,a_h^\ell)+V_{u,h+1}(s_{h+1}^\ell)\bigr).
\label{eq:persistent-predictor-inside-reward}
\end{align}
Under the first convention, $r_h(s,a)$ is added outside the regression term and satisfies the assumptions in Section~\ref{prelim}. Under the second convention, the reward is included in the regression target, as in prior work such as \citet{jin2020provably}, where the reward is assumed to be linear in the feature map: for every $h\in[H]$, there exists $\eta_h\in\mathbb R^d$ with $\|\eta_h\|_2\leq\sqrt d$ such that $r_h(s,a)=\phi(s,a)^\top\eta_h$ for all $(s,a)\in\mathcal S\times\mathcal A$.

\textbf{(iv) Persistent ellipsoidal bonus and one fixed maintenance rule.}
At each genuine update,
\begin{align*}
\widetilde Q_{u,h}(s,a) = \min\!\left\{\widehat{\mathcal B}_{u,h}(s,a) +\beta\|\phi(s,a)\|_{\Lambda_{u,h}^{-1}},\ H\right\}. 
\end{align*}
Before interaction, the algorithm fixes one of the two rules
\begin{align}
Q_{u,h}&=\widetilde Q_{u,h}
&&\text{(newest-candidate rule)},
\notag\\
Q_{u,h}(s,a)&=
\min_{v\leq u:\,v\text{ genuine}}\widetilde Q_{v,h}(s,a)
&&\text{(historical-minimum rule)}.
\label{eq:persistent-maintained-Q}
\end{align}
Initially $Q_{0,h}(s,a)=H$, and
\[V_{u,h}(s)=\max_a Q_{u,h}(s,a), \qquad V_{u,H+1}(s)=0.\]
Under the historical-minimum rule, the maintained $Q$- and $V$-values are therefore pointwise nonincreasing across genuine updates.

\textbf{(v) Fixed deterministic greedy tie-breaking.}
The finite action set is indexed. At every step,
\begin{align}a_h^k = \min\arg\max_{a}Q_{\tau(k),h}(s_h^k,a). \label{eq:persistent-greedy-tie}\end{align}
Thus the smallest indexed maximizing action is always selected. The algorithm may still randomize its predictable weights or update decisions.

\textbf{(vi) Model-wise full-domain calibration and optimism.}
For each fixed linear MDP $M$, when the algorithm is run on $M$ for $K$ episodes, there exists an event $\mathcal E_{\mathrm{cal}}^M$ satisfying
\begin{align}\mathbb P_M(\mathcal E_{\mathrm{cal}}^M)\geq1-\delta_{\mathrm{cal}}. \label{eq:persistent-calibration-probability}\end{align}
On $\mathcal E_{\mathrm{cal}}^M$, simultaneously at every genuine update $u$ within the first $K$ episodes, every step $h$, and every state-action pair $(s,a)$, including unreachable states,
\begin{align}
&\left|
\widehat{\mathcal B}_{u,h}(s,a)
-
\bigl(r_h(s,a)+\mathbb P_{s,a,h}V_{u,h+1}\bigr)
\right|
\leq
\beta\|\phi(s,a)\|_{\Lambda_{u,h}^{-1}},
\label{eq:persistent-calibration}
\\
&Q_{u,h}(s,a)\geq Q_h^\star(s,a),
\qquad
0\leq Q_{u,h}(s,a),V_{u,h}(s)\leq H.\notag
\end{align}
The event may depend on $M$; no simultaneous event over different MDPs is required. The probability includes all algorithm randomization. The algorithm has only ordinary online interaction and receives no additional observations from unreachable states.
\end{definition}

When the same algorithm with the same parameters $(\lambda,\beta,\alpha_-,\alpha_+,\kappa,\delta_{\mathrm{cal}})$ is applied to every member of the hard-instance family constructed in Section~\ref{subsec:persistent-hard-family}, the state and action spaces, feature map, rewards, horizon, initial state, and episode budget are also held fixed across the family; only the unknown transition kernel changes. Consequently, on the same realized history, the algorithm uses the same conditional action, update, weighting, and internal-randomization kernels under every family member.

\begin{remark}[Full-domain calibration]
\label{rem:persistent-global-calibration}
The proof of \Cref{thm:persistent-structural-lower-bound} below applies the calibration condition at unreachable states (see Lemma \ref{lem:persistent-probe}), but does not rely on any observations from those states. This is compatible with the usual ellipsoidal confidence analysis, because the confidence event controls the estimated regression parameter and therefore its prediction at any fixed feature vector.

For any bounded function $V$, define
\[
w_h(V):=\int_{\mathcal S}V(s')\,d\btheta_h(s')\in\mathbb R^d,
\]
so that
\[
\mathbb P_{s,a,h}V=\phi(s,a)^\top w_h(V)
\qquad\text{for all }(s,a).
\]
Accordingly, the population regression parameter is $w_h(V)$ under the outside-reward convention in \Cref{eq:persistent-predictor-outside-reward}, and $\eta_h+w_h(V)$ under the inside-reward convention in \Cref{eq:persistent-predictor-inside-reward}. More generally, if $\widehat b$ is the corresponding ridge estimator, $b$ is its population target, and
\[
\|\widehat b-b\|_{\Lambda}\leq\beta,
\]
then, for any feature vector $\phi$,
\[
|\phi^\top(\widehat b-b)|
\leq
\|\widehat b-b\|_{\Lambda}\|\phi\|_{\Lambda^{-1}}
\leq
\beta\|\phi\|_{\Lambda^{-1}}.
\]
Thus, the same parameter-space confidence event controls any fixed state-action feature vector, regardless of whether the corresponding state is reachable. For the outside-reward convention, the resulting full-domain Bellman confidence bound with radius $\beta$ is established in part~(b) of \Cref{concentration}. For the inside-reward convention, under the linear-reward assumption, the analogous full-domain Bellman confidence bound is standard for LSVI-style algorithms; see, e.g., \citep[Lemma~B.4]{jin2020provably}. Hence, the calibration condition in \Cref{eq:persistent-calibration} does not depend on observations from unreachable states, which are fixed before interaction and never queried by the algorithm.
\end{remark}
For the remainder of the proof, let
\begin{align*}A_0:=9, \qquad m:=d-6, \qquad H_0:=\left\lfloor\frac H4\right\rfloor, \qquad \mathcal T:=\{3,\ldots,H_0\}, \qquad R_t:=H-t.\end{align*}
Define
\begin{align}
N_b
&:=
\left\lceil
\frac{(\kappa-1)\lambda+4096\kappa\beta^2}{\alpha_-}
\right\rceil,
\qquad
n_s
:=
\left\lceil
\frac{1024\kappa\beta^2}{\alpha_-\dmin^2}
\right\rceil,
\label{eq:persistent-burnin-counts}
\\
K_{\mathrm{burn}}
&:=
\left\lceil
\max\{8N_b,\ 8Hn_s,\ 32H\log128\}
\right\rceil,
\label{eq:persistent-burnin-budget}
\\
\epsilon_t
&:=
\frac{A_0\dmin}{R_t\sqrt m},
\qquad
n_t^{\mathrm{info}}
:=
\frac{m}{2^{19}\epsilon_t^2}
=
\frac{m^2R_t^2}{2^{19}A_0^2\dmin^2}.
\label{eq:persistent-information-threshold}
\end{align}
Finally, define
\begin{align*}
\chi_0:=2\left(\frac{256}{567}\right)^2=\frac{131072}{321489},
\qquad
c_{\mathrm{cal}}
:=
\frac{127}{256}-\chi_0-\frac1{64}-\delta_{\mathrm{cal}}.
\end{align*}

\begin{theorem}[Formal statement of \Cref{ellipsoidlower}]
\label{thm:persistent-structural-lower-bound}
For any feature dimension $d\geq15$, horizon $H\geq16$, minimum gap $\dmin\in(0,H/16]$, episode number $K$, and any algorithm satisfying Definition~\ref{def:persistent-ellipsoid} with parameters $(\lambda,\beta,\alpha_-,\alpha_+,\kappa,\delta_{\mathrm{cal}})$ such that $c_{\mathrm{cal}}>0$, if
\begin{align}K \geq K_{\mathrm{burn}} + \left\lceil \frac{8H}{c_{\mathrm{cal}}}\max_{t\in\mathcal T}n_t^{\mathrm{info}} \right\rceil, \label{eq:persistent-main-budget}\end{align}
then there exists a member $M_{\bsigma}$ of the hard-instance family constructed in Section~\ref{subsec:persistent-hard-family} that is a legal $d$-dimensional linear MDP with fixed initial state $s_{\mathrm{init}}$ and minimum positive gap exactly $\dmin$, and satisfies
\begin{align}\mathbb E_{M_{\bsigma}}[\operatorname{Regret}(T)] \geq 2^{-33}\frac{d^2H^3}{\dmin}. \label{eq:persistent-class-final}\end{align}

Moreover, the sufficient budget in \Cref{eq:persistent-main-budget} scales as
\begin{align}O\!\left( 1+ \frac{\kappa\lambda}{\alpha_-} + \frac{\kappa\beta^2}{\alpha_-}\left(1+\frac{H}{\dmin^2}\right) + \frac{d^2H^3}{c_{\mathrm{cal}}\dmin^2} +H \right). \label{eq:persistent-main-budget-order}\end{align}

\end{theorem}

\subsection{A Hard-Instance Family}
\label{subsec:persistent-hard-family}

Let
\[\mathfrak S:=\{-1,+1\}^{m\times\mathcal T}.\]
For every sign array $\bsigma=(\sigma_{i,t})\in\mathfrak S$, we define an MDP $M_{\bsigma}$. All family members share the same known state space, action space, feature map, rewards, horizon, and initial state; only the unknown transition kernel depends on $\bsigma$.

Choose orthonormal coordinates
\[e_0,e_1,\ldots,e_m,e_\star,e_{\mathrm{init}},e_w,e_g,e_b\in\mathbb R^d.\]
There are exactly $m+6=d$ coordinates. Define the countable rational simplex
\[\mathcal V_{\mathbb Q} := \left\{v\in\mathbb Q^m:\ v_i\geq0,\ \sum_{i=1}^m v_i=1\right\}.\]
The state space, equipped with the discrete sigma-field, is
\[\mathcal S = \{s_{\mathrm{init}},w,g,b,z_s,x_\star,x_1,\ldots,x_m\} \cup \{z_v:v\in\mathcal V_{\mathbb Q}\}, \qquad \mathcal A=\{1,2\}.\]
The feature map is
\begin{align*}
&\phi(x_i,1)=e_0,
&&\phi(x_i,2)=e_i,
\qquad i\in[m],
\notag\\
&\phi(x_\star,1)=e_0,
&&\phi(x_\star,2)=e_\star,
\notag\\
&\phi(z_s,a)=e_0,
&&\phi(z_v,a)=\sum_{i=1}^m v_i e_i,
\notag\\
&\phi(s_{\mathrm{init}},a)=e_{\mathrm{init}},
&&\phi(w,a)=e_w,
\qquad
\phi(g,a)=e_g,
\qquad
\phi(b,a)=e_b.
\end{align*}
At every step, the reward equals one in $g$ and zero elsewhere. Equivalently,
\[r_h(s,a)=\phi(s,a)^\top e_g.\]

Every episode starts at $s_{\mathrm{init}}$. At step $1$, independently of the action,
\[P_1(b\mid s_{\mathrm{init}},a)=\frac14, \qquad P_1(w\mid s_{\mathrm{init}},a)=\frac34.\]
At every other step, $s_{\mathrm{init}}$ moves to $b$. For $2\leq h\leq H_0-1$, set
\begin{align*}
P_h(z_s\mid w,a)&=\frac1{2H},
\notag\\
P_h(x_i\mid w,a)&=P_h(x_\star\mid w,a)=\frac1{2(m+1)H},
\qquad i\in[m],
\notag\\
P_h(w\mid w,a)&=1-\frac1H.
\end{align*}
Outside these steps, $w$ moves to $b$. The states $g$ and $b$ are absorbing.

For every $t\in\mathcal T$, define
\begin{align*}
p_{0,t} := \frac12+\epsilon_t+\frac{\dmin}{R_t}, \qquad p_{i,t} := \frac12+\sigma_{i,t}\epsilon_t, \qquad p_{\star,t} := \frac12+\epsilon_t.
\end{align*}
At step $t$, action $1$ at every $x_i$ and $x_\star$, and both actions at $z_s$, move to $g$ with probability $p_{0,t}$. Action $2$ at $x_i$ moves to $g$ with probability $p_{i,t}$, while action $2$ at $x_\star$ moves to $g$ with probability $p_{\star,t}$. Both actions at $z_v$ move to $g$ with probability
\[\sum_{i=1}^m v_i p_{i,t}.\]
In every case the remaining probability goes to $b$. At steps outside $\mathcal T$, both actions at $x_i,x_\star,z_s,z_v$ move to $b$.

Every $z_v$ is unreachable from $s_{\mathrm{init}}$: no transition distribution of a reachable state assigns mass to such a probe. The countable probe family, including all features and transition kernels, is fixed before interaction.

\begin{lemma}[Validity, exact gap, and exogenous reach probabilities]
\label{lem:persistent-hard-validity}
Every $M_{\bsigma}$ is a valid $d$-dimensional linear MDP satisfying the normalization in Section~\ref{prelim}. Its minimum positive gap is exactly $\dmin$. Moreover, for every $t\in\mathcal T$ and $i\in[m]$, the reach probabilities are independent of the algorithm and of $\bsigma$ and satisfy
\begin{align}q_{s,t}:=\mathbb P(s_t=z_s)\geq\frac1{4H}, \qquad q_{i,t}:=\mathbb P(s_t=x_i)\geq\frac1{8mH}. \label{eq:persistent-reach-probability}\end{align}
All features observed along an actual trajectory are one-hot in the displayed basis; hence every sampled covariance matrix is diagonal in that basis.
\end{lemma}

\begin{proof}
For $t\in\mathcal T$, define the coordinate transition measures corresponding to $e_0,e_i,e_\star$ as the Bernoulli mixtures of $\delta_g$ and $\delta_b$ with success probabilities $p_{0,t},p_{i,t},p_{\star,t}$, respectively. Outside $\mathcal T$, set these coordinate measures to $\delta_b$. The coordinates $e_{\mathrm{init}},e_w,e_g,e_b$ are assigned the transition distributions specified above. Then the feature $\sum_i v_i e_i$ at $z_v$ produces exactly the stated convex-combination kernel. Every coordinate transition measure is a probability measure, so for every measurable set $C$,
\[\|\btheta_h(C)\|_2^2\leq d.\]
Also, $\|\phi(z_v,a)\|_2=\|v\|_2\leq\|v\|_1=1$, every other feature has norm one, and the reward parameter $e_g$ has norm one. Thus every family member is a legal linear MDP.

Since $m\geq9$, $A_0=9$, and $R_t\geq3H/4$,
\[R_t\epsilon_t=\frac{A_0\dmin}{\sqrt m}\leq3\dmin, \qquad \epsilon_t\leq\frac{3\dmin}{R_t}\leq\frac14.\]
Therefore $p_{i,t}\in[1/4,3/4]$, $p_{\star,t}\leq3/4$, and
\[p_{0,t}\leq\frac12+\frac{4\dmin}{R_t}\leq\frac56,\]
so all displayed probabilities are valid. Since
\[V_{t+1}^\star(g)=R_t, \qquad V_{t+1}^\star(b)=0,\]
the gap of action $2$ at $x_i$ is
\begin{align}R_t(p_{0,t}-p_{i,t}) = \dmin+(1-\sigma_{i,t})R_t\epsilon_t \in[\dmin,7\dmin], \label{eq:persistent-xi-gap}\end{align}
whereas the gap of action $2$ at the reachable state $x_\star$ is exactly
\[R_t(p_{0,t}-p_{\star,t})=\dmin.\]
All remaining state-action-step triples have zero gap. Hence the minimum positive gap is exactly $\dmin$ for every member of the family.

Finally, for $t\in\mathcal T$,
\[q_{s,t} = \frac34\left(1-\frac1H\right)^{t-3}\frac1{2H}, \qquad q_{i,t} = \frac34\left(1-\frac1H\right)^{t-3}\frac1{2(m+1)H}.\]
Since $t\leq H_0\leq H/4$, Bernoulli's inequality gives
\[\left(1-\frac1H\right)^{t-3}\geq\frac34.\]
Thus $q_{s,t}\geq9/(32H)\geq1/(4H)$. Also, $m\geq9$ implies $9m\geq4(m+1)$, so $q_{i,t}\geq1/(8mH)$. These transitions are action-independent and use fresh environment randomness, proving the claimed exogeneity. The only non-one-hot features belong to the unreachable probes $z_v$.
\end{proof}

\subsection{Global Calibration Gives a Joint Error Budget}
\label{subsec:persistent-joint-error}

Fix a family member $M_{\bsigma}$, a genuine update $u$, and an active step $t\in\mathcal T$. Write
\begin{align*}D_i := e_i^\top\Lambda_{u,t}e_i, \qquad G := V_{u,t+1}(g), \qquad C := V_{u,t+1}(b), \qquad \Gamma := G-C,\end{align*}
and define the actual Bellman-prediction error in hard direction $i$ by
\begin{align*}\xi_i := \widehat{\mathcal B}_{u,t}(x_i,2) - \bigl(C+\Gamma p_{i,t}\bigr). \end{align*}
\begin{lemma}[Joint positive and negative error-energy bounds]
\label{lem:persistent-probe}
On $\mathcal E_{\mathrm{cal}}^{M_{\bsigma}}$, at every genuine update $u$ and every $t\in\mathcal T$,
\begin{align}\sum_{i=1}^mD_i(\xi_i)_+^2\leq\beta^2, \qquad \sum_{i=1}^mD_i(\xi_i)_-^2\leq\beta^2. \label{eq:persistent-signed-energy}\end{align}
Consequently,
\begin{align}
\sum_{i=1}^mD_i\xi_i^2\leq2\beta^2. 
\label{eq:persistent-total-energy}\end{align}
\end{lemma}

\begin{proof}
The only reachable sample whose feature has an $e_i$ component is action $2$ at $x_i$, and its reward is zero. Hence, under either reward convention in Definition~\ref{def:persistent-ellipsoid}, for every rational-simplex probe $z_v$,
\begin{align*}
\widehat{\mathcal B}_{u,t}(z_v,a) - \mathbb P_{z_v,a,t}V_{u,t+1} = \sum_{i=1}^m v_i\xi_i. 
\end{align*}
By Lemma~\ref{lem:persistent-hard-validity}, the sampled covariance is diagonal in the displayed basis, so
\[\|\phi(z_v,a)\|_{\Lambda_{u,t}^{-1}}^2 = \sum_{i=1}^m\frac{v_i^2}{D_i}.\]
Global calibration at the unreachable probe gives
\begin{align}\left|\sum_{i=1}^m v_i\xi_i\right| \leq \beta\sqrt{\sum_{i=1}^m\frac{v_i^2}{D_i}} \qquad \forall v\in\mathcal V_{\mathbb Q}. \label{eq:persistent-simplex-calibration}\end{align}
Both sides are continuous in $v$, so \Cref{eq:persistent-simplex-calibration} holds for every point of the real probability simplex.

Let
\[S_+:=\sum_iD_i(\xi_i)_+^2, \qquad L_+:=\sum_iD_i(\xi_i)_+.\]
If $L_+=0$, the positive-part claim is immediate. Otherwise choose
\[v_i=\frac{D_i(\xi_i)_+}{L_+}.\]
Then \Cref{eq:persistent-simplex-calibration} gives
\[\frac{S_+}{L_+} \leq \beta\frac{\sqrt{S_+}}{L_+},\]
which implies $S_+\leq\beta^2$. Applying the same argument to the negative parts proves the second inequality in \Cref{eq:persistent-signed-energy}; summing them gives \Cref{eq:persistent-total-energy}.
\end{proof}

\subsection{Burn-in and Maintained-Value Comparisons}
\label{subsec:persistent-burnin}

For every $t\in\mathcal T$, let $\mathcal F_t$ be the event that among the first $K_{\mathrm{burn}}$ completed episodes, at least $N_b$ episodes take the transition $s_{\mathrm{init}}\to b$ at step $1$, and at least $n_s$ episodes visit $z_s$ at step $t$.

\begin{lemma}[Anchor counts]
\label{lem:persistent-anchor-counts}
For every $M_{\bsigma}$ and every $t\in\mathcal T$,
\[\mathbb P_{M_{\bsigma}}(\mathcal F_t^c)\leq\frac1{64}.\]
\end{lemma}

\begin{proof}
The first count is $\operatorname{Binomial}(K_{\mathrm{burn}},1/4)$. The second is binomial with success probability $q_{s,t}\geq1/(4H)$ by Lemma~\ref{lem:persistent-hard-validity}. For a sum $X$ of independent Bernoulli variables with mean $\mu$, the standard multiplicative Chernoff bound gives
\[\mathbb P(X\leq\mu/2)\leq e^{-\mu/8}.\]
Because $K_{\mathrm{burn}}\geq8N_b$ and $K_{\mathrm{burn}}\geq8Hn_s$, the two failure probabilities are at most $e^{-K_{\mathrm{burn}}/32}$ and $e^{-K_{\mathrm{burn}}/(32H)}$, respectively. By \Cref{eq:persistent-burnin-budget}, each is at most $1/128$. A union bound proves the claim. No independence between the two counts is required, and we do not union-bound over $t$.
\end{proof}

\begin{lemma}[Cemetery and shared-direction bounds]
\label{lem:persistent-anchor}
Fix $M_{\bsigma}$ and $t\in\mathcal T$. On
$\mathcal F_t\cap\mathcal E_{\mathrm{cal}}^{M_{\bsigma}}$, for every episode $k>K_{\mathrm{burn}}$, the latest update $u=\tau(k)$ is genuine and
\begin{align}V_{u,t+1}(b)&\leq\frac{R_t}{64}, & V_{u,t+1}(g)&\geq R_t, & V_{u,t+1}(g)-V_{u,t+1}(b)&\geq\frac{63R_t}{64}. \label{eq:persistent-cemetery-bounds}\end{align}
Moreover, the maintained score of the shared action, which is identical at every $x_i$, satisfies
\begin{align}Q_{u,t}(x_i,1) \leq C+\Gamma p_{0,t}+\frac{\dmin}{16} <\frac{7H}{8}. \label{eq:persistent-shared-score}\end{align}
\end{lemma}

\begin{proof}
Every episode whose first transition is $s_{\mathrm{init}}\to b$ supplies an $e_b$ sample at every step $h\geq2$. Thus, for $k>K_{\mathrm{burn}}$,
\[e_b^\top\Lambda_{k,h}^{\mathrm{all}}e_b \geq \lambda+\alpha_-N_b.\]
If $u=\tau(k)=0$, \Cref{eq:persistent-covariance-growth} would imply
\[\lambda+\alpha_-N_b\leq\kappa\lambda,\]
contradicting the definition of $N_b$. Hence $u$ is genuine. Again using \Cref{eq:persistent-covariance-growth},
\begin{align}e_b^\top\Lambda_{u,h}e_b \geq \frac{\lambda+\alpha_-N_b}{\kappa} \geq \lambda+4096\beta^2. \label{eq:persistent-eb-covariance}\end{align}
All historical responses in the $e_b$ direction equal the current continuation value at $b$, and every corresponding reward is zero. Under either reward convention,
\[\widehat{\mathcal B}_{u,h}(b,a) = \frac{W_{u,h,b}}{\lambda+W_{u,h,b}}V_{u,h+1}(b) \leq V_{u,h+1}(b),\]
where $W_{u,h,b}$ is the total historical weight in direction $e_b$. By \Cref{eq:persistent-eb-covariance}, the bonus is at most $1/64$. Since either maintenance rule gives a maintained score no larger than the current candidate,
\[V_{u,h}(b)\leq V_{u,h+1}(b)+\frac1{64}.\]
Backward induction from $V_{u,H+1}(b)=0$ yields
\[V_{u,t+1}(b)\leq\frac{R_t}{64}.\]
Optimism gives $V_{u,t+1}(g)\geq V_{t+1}^\star(g)=R_t$, proving \Cref{eq:persistent-cemetery-bounds}.

At step $t$, every $z_s$ sample contributes weight at least $\alpha_-$ in direction $e_0$. Hence
\[e_0^\top\Lambda_{u,t}e_0 \geq \frac{\lambda+\alpha_-n_s}{\kappa} \geq \frac{1024\beta^2}{\dmin^2},\]
so the current bonus in direction $e_0$ is at most $\dmin/32$. By calibration, the current candidate for action $1$ at $x_i$ is therefore at most
\[C+\Gamma p_{0,t}+\frac{\dmin}{16}.\]
The maintained score is no larger than this candidate under either maintenance rule, giving the first inequality in \Cref{eq:persistent-shared-score}. Finally, using $p_{0,t}\leq5/6$, $G\leq H$, $C\leq H/64$, and $\dmin/16\leq H/256$,
\[C+\Gamma p_{0,t}+\frac{\dmin}{16} \leq \left(\frac56+\frac1{64}+\frac1{256}\right)H = \frac{655}{768}H <\frac{7H}{8}.\]
\end{proof}

The historical-minimum maintenance rule requires one additional comparison because the currently maintained value of an action can have been created at an earlier genuine update.

\begin{lemma}[Frozen-direction comparison]
\label{lem:persistent-frozen-direction}
Fix a sample path, episode $k$, and $t\in\mathcal T$, and let $u=\tau(k)\geq1$. Suppose the shared score
\[q_s:=Q_{u,t}(x_i,1)\]
is below $H$. Under either maintenance rule, if the prescribed greedy action at $x_i$ is action $1$, equivalently if $Q_{u,t}(x_i,2)\leq q_s$, then
\begin{align}\widehat{\mathcal B}_{u,t}(x_i,2) + \frac{\beta}{\sqrt{D_i}} \leq q_s. \label{eq:persistent-frozen-raw-score}\end{align}
\end{lemma}

\begin{proof}
For the newest-candidate rule, a maintained score no larger than $q_s<H$ cannot be upper-clipped at $H$, so it equals its raw current candidate, and \Cref{eq:persistent-frozen-raw-score} follows immediately.

For the historical-minimum rule, write
\[q_i:=Q_{u,t}(x_i,2)\leq q_s<H.\]
Because the initial score equals $H$, the value $q_i$ is attained by some genuine historical candidate; let $v\leq u$ be an update attaining it. From update $v$ until the current episode, the bad maintained score remains exactly $q_i$: a strictly smaller intervening candidate would contradict its current value. During the same period the shared maintained score is at least its current value $q_s\geq q_i$, because maintained values are pointwise nonincreasing. Therefore every intervening visit to $x_i$ selects action $1$, including at equality by the smallest-index rule \Cref{eq:persistent-greedy-tie}.

No new sample in direction $e_i$ has therefore been collected since update $v$. The probes $z_v$ are unreachable, and no other reachable feature has an $e_i$ component. Thus the covariance and immutable historical weights in direction $e_i$ are unchanged. Under historical-minimum maintenance, continuation values are pointwise nonincreasing across genuine updates; all samples in direction $e_i$ have zero reward. Consequently, under either reward convention,
\[\widehat{\mathcal B}_{u,t}(x_i,2) \leq \widehat{\mathcal B}_{v,t}(x_i,2).\]
The bonus in direction $e_i$ is unchanged. Since the witness candidate at $v$ equals $q_i<H$, it was not clipped, so its raw value equals $q_i$. Hence the current raw score is at most $q_i\leq q_s$, proving \Cref{eq:persistent-frozen-raw-score}.
\end{proof}

\subsection{Calibration Forces Exploration or Fine Estimation}
\label{subsec:persistent-exploration-estimation}

Fix $k>K_{\mathrm{burn}}$ and $t\in\mathcal T$, and let $u=\tau(k)$. At the information set after the algorithm's pre-episode update and internal randomization but before the current episode's environment transitions, define
\begin{align*}
\mathcal I_{k,t} := \left\{i\in[m]:Q_{u,t}(x_i,2)>Q_{u,t}(x_i,1)\right\}, \qquad r_t(k):=|\mathcal I_{k,t}|. 
\end{align*}
A visit to an index in $\mathcal I_{k,t}$ selects action $2$ with probability one, while outside this set the smallest-index rule selects action $1$.

Define an auxiliary sign estimator $\widehat\bsigma_t(k)\in\{-1,+1\}^m$ from the same pre-episode information. If $u=0$ or
\[\Gamma:=V_{u,t+1}(g)-V_{u,t+1}(b)\leq0,\]
set every estimated sign to $+1$. Otherwise define, for every $i$,
\begin{align}
\widehat p_i
:=
\frac{\widehat{\mathcal B}_{u,t}(x_i,2)-V_{u,t+1}(b)}{\Gamma},
\qquad
\widehat\sigma_{i,t}(k)
:=
\begin{cases}
+1,&\widehat p_i\geq1/2,\\
-1,&\widehat p_i<1/2.
\end{cases}
\label{eq:persistent-sign-estimator}
\end{align}
This estimator is defined on the entire history space and is used only in the proof; it does not modify the algorithm.

\begin{lemma}[Direct energy charge for a nonexplorable sign error]
\label{lem:persistent-energy-charge}
On $\mathcal F_t\cap\mathcal E_{\mathrm{cal}}^{M_{\bsigma}}$, every misclassified index $i\notin\mathcal I_{k,t}$ satisfies
\begin{align}D_i\xi_i^2 \geq \beta^2 \left( \frac{\epsilon_t} {\epsilon_t+\dmin/R_t+\dmin/(16\Gamma)} \right)^2. \label{eq:persistent-mistake-energy}\end{align}
Consequently, the number of such misclassified nonexplorable indices is at most
\begin{align}2\left(1+\frac{67\sqrt m}{63A_0}\right)^2 \leq \chi_0m. \label{eq:persistent-mistake-count}\end{align}
\end{lemma}

\begin{proof}
On the stated event, Lemma~\ref{lem:persistent-anchor} gives $\Gamma>0$ and a shared score below $H$. For $i\notin\mathcal I_{k,t}$, Lemma~\ref{lem:persistent-frozen-direction} and \Cref{eq:persistent-shared-score} imply
\begin{align}\xi_i+b_i \leq (1-\sigma_{i,t})\Gamma\epsilon_t+\zeta_t, \qquad b_i:=\frac{\beta}{\sqrt{D_i}}, \qquad \zeta_t:=\Gamma\frac{\dmin}{R_t}+\frac{\dmin}{16}. \label{eq:persistent-raw-error-comparison}\end{align}

If $\sigma_{i,t}=+1$ is misclassified, then $\widehat p_i<1/2$, so $\xi_i=-x$ for some $x>\Gamma\epsilon_t$. \Cref{eq:persistent-raw-error-comparison} gives $b_i\leq x+\zeta_t$, and therefore
\[\frac{|\xi_i|}{b_i} \geq \frac{x}{x+\zeta_t} \geq \frac{\Gamma\epsilon_t}{\Gamma\epsilon_t+\zeta_t}.\]
If $\sigma_{i,t}=-1$ is misclassified, then $\xi_i=x\geq\Gamma\epsilon_t$. \Cref{eq:persistent-raw-error-comparison} gives
\[b_i\leq2\Gamma\epsilon_t+\zeta_t-x\leq\Gamma\epsilon_t+\zeta_t,\]
and the same ratio bound follows. Since
\[D_i\xi_i^2 = \beta^2\left(\frac{|\xi_i|}{b_i}\right)^2,\]
we obtain \Cref{eq:persistent-mistake-energy}.

By \Cref{eq:persistent-cemetery-bounds}, $\Gamma\geq63R_t/64$. Hence
\[\frac{\dmin/R_t+\dmin/(16\Gamma)}{\epsilon_t} \leq \frac{67\sqrt m}{63A_0}.\]
Each nonexplorable mistake therefore consumes at least the reciprocal-square fraction of $\beta^2$ displayed in \Cref{eq:persistent-mistake-energy}. Applying the total energy budget \Cref{eq:persistent-total-energy} yields the first bound in \Cref{eq:persistent-mistake-count}. Finally, $m\geq9$ and $A_0=9$ imply
\[\frac{2}{m}\left(1+\frac{67\sqrt m}{63A_0}\right)^2 = 2\left(\frac1{\sqrt m}+\frac{67}{567}\right)^2 \leq 2\left(\frac13+\frac{67}{567}\right)^2 = \chi_0.\]
\end{proof}

\begin{lemma}[Exploration--estimation relation]
\label{lem:persistent-exploration-estimation}
For every family member $M_{\bsigma}$ and every $k>K_{\mathrm{burn}}$,
\begin{align*}
\mathbb E_{M_{\bsigma}}\!\left[ \operatorname{Ham}(\widehat\bsigma_t(k),\bsigma_t) \right]
\leq
\mathbb E_{M_{\bsigma}}[r_t(k)]
+
m\left(\chi_0+\frac1{64}+\delta_{\mathrm{cal}}\right).
\end{align*}
\end{lemma}

\begin{proof}
On $\mathcal F_t\cap\mathcal E_{\mathrm{cal}}^{M_{\bsigma}}$, Lemma~\ref{lem:persistent-anchor} ensures that the estimator in \Cref{eq:persistent-sign-estimator} uses a positive $\Gamma$. There can be at most $r_t(k)$ sign errors on explorable indices, while Lemma~\ref{lem:persistent-energy-charge} bounds the number of errors outside $\mathcal I_{k,t}$ by $\chi_0m$. Thus
\[\operatorname{Ham}(\widehat\bsigma_t(k),\bsigma_t) \leq r_t(k)+\chi_0m\]
on this event. The Hamming loss is always at most $m$. Using Lemma~\ref{lem:persistent-anchor-counts} and \Cref{eq:persistent-calibration-probability}, without any independence assumption,
\[
\mathbb E_{M_{\bsigma}}\!\left[
\operatorname{Ham}(\widehat\bsigma_t(k),\bsigma_t)
\right]
\leq
\mathbb E_{M_{\bsigma}}[r_t(k)]
+
m\left(\chi_0+\frac1{64}+\delta_{\mathrm{cal}}\right).
\]
\end{proof}

\subsection{Information Lower Bound and Completion of the Proof}
\label{subsec:persistent-information-completion}

For any random quantity $X$ and event $E$ that may depend on the family index $\bsigma$, define
\begin{align*}
\overline{\mathbb E}[X]
&:=
\mathbb E_{\bsigma\sim\mathrm{Unif}(\mathfrak S)}
\!\left[\mathbb E_{M_{\bsigma}}[X]\right]
=
\frac1{|\mathfrak S|}
\sum_{\bsigma\in\mathfrak S}
\mathbb E_{M_{\bsigma}}[X],
\notag\\
\overline{\mathbb P}(E)
&:=
\mathbb E_{\bsigma\sim\mathrm{Unif}(\mathfrak S)}
\!\left[\mathbb P_{M_{\bsigma}}(E)\right]
=
\frac1{|\mathfrak S|}
\sum_{\bsigma\in\mathfrak S}
\mathbb P_{M_{\bsigma}}(E).
\end{align*}
For $i\in[m]$, $t\in\mathcal T$, and $k\in[K]$, define
\begin{align*}
N_{i,t,2}(k) := \sum_{\ell=1}^k\mathbb I\{s_t^\ell=x_i,\ a_t^\ell=2\}, \qquad S_t(k) := \frac1{|\mathfrak S|} \sum_{\bsigma\in\mathfrak S} \sum_{i=1}^m \mathbb E_{M_{\bsigma}}[N_{i,t,2}(k)]. 
\end{align*}

\begin{lemma}[Information lower bound for the sign estimator]
\label{lem:persistent-information}
For any estimator of $\bsigma_t$ measurable with respect to the information available before the current episode's environment transitions,
\begin{align}\overline{\mathbb E}\! \left[\operatorname{Ham}(\widehat\bsigma_t(k),\bsigma_t)\right] \geq \frac m2 -2\sqrt2\,\epsilon_t\sqrt{mS_t(k-1)}. \label{eq:persistent-information-lower}\end{align}
In particular, if $S_t(k-1)\leq n_t^{\mathrm{info}}$, then
\begin{align}\overline{\mathbb E}\! \left[\operatorname{Ham}(\widehat\bsigma_t(k),\bsigma_t)\right] \geq \frac{127m}{256}. \label{eq:persistent-information-threshold-consequence}\end{align}
\end{lemma}

\begin{proof}
Fix an index $(i,t)$ and fix all signs except $\sigma_{i,t}$. Let $P_-$ and $P_+$ denote the distributions of the pre-episode history under the two models with $\sigma_{i,t}=-1$ and $\sigma_{i,t}=+1$, respectively. The reachable transition kernels of these two models differ only at $(x_i,2,t)$. They also differ at some unreachable probes, but those states have identically zero visit counts.

Conditional on any fixed realized history, the algorithm's action, update, weight, and internal-randomization kernels are the same under the two models. Therefore, the chain rule for relative entropy yields
\begin{align*}
\operatorname{KL}(P_-\|P_+)
&=
\mathbb E_-\!\left[N_{i,t,2}(k-1)\right]
\operatorname{kl}\!\left(\frac12-\epsilon_t,\frac12+\epsilon_t\right)
\leq
32\epsilon_t^2\,
\mathbb E_-\!\left[N_{i,t,2}(k-1)\right].
\end{align*}
Indeed, $\epsilon_t\leq1/4$ by Lemma~\ref{lem:persistent-hard-validity}, and
$\operatorname{kl}(p,q)\leq(p-q)^2/[q(1-q)]\leq32\epsilon_t^2$ for the two displayed Bernoulli parameters. The history may be augmented with the pre-episode update and internal randomization used to form the estimator; these are common conditional kernels and contribute zero conditional relative entropy.

By Pinsker's inequality, any binary test of $\sigma_{i,t}$ satisfies
\[P_-(\widehat\sigma_{i,t}=+1) + P_+(\widehat\sigma_{i,t}=-1) \geq 1- 4\epsilon_t \sqrt{\mathbb E_-[N_{i,t,2}(k-1)]}.\]
Average this inequality over all pairs of sign arrays that differ only in coordinate $(i,t)$. The mean count on the negative half-family is at most twice the mean over the full family, so Jensen's inequality gives
\[\overline{\mathbb P}(\widehat\sigma_{i,t}\neq\sigma_{i,t}) \geq \frac12 -2\sqrt2\,\epsilon_t \sqrt{\overline{\mathbb E}[N_{i,t,2}(k-1)]}.\]
Summing over $i$ and applying Cauchy--Schwarz proves \Cref{eq:persistent-information-lower}. If $S_t(k-1)\leq m/(2^{19}\epsilon_t^2)$, then
\[2\sqrt2\,\epsilon_t\sqrt{mS_t(k-1)} \leq \frac{m}{256},\]
which yields \Cref{eq:persistent-information-threshold-consequence}.
\end{proof}

\begin{lemma}[Constant sampling rate below the information threshold]
\label{lem:persistent-sampling-rate}
For every $k>K_{\mathrm{burn}}$, if $S_t(k-1)\leq n_t^{\mathrm{info}}$, then
\begin{align*}
S_t(k)-S_t(k-1) \geq \frac{c_{\mathrm{cal}}}{8H}. 
\end{align*}
\end{lemma}

\begin{proof}
Average Lemma~\ref{lem:persistent-exploration-estimation} over the uniform prior on $\mathfrak S$ and combine it with Lemma~\ref{lem:persistent-information}. Under the assumed information threshold,
\[
\overline{\mathbb E}[r_t(k)]
\geq
\left(
\frac{127}{256}
-\chi_0
-\frac1{64}
-\delta_{\mathrm{cal}}
\right)m
=
c_{\mathrm{cal}}m.
\]
Condition on the information after the current pre-episode update and before environment transitions. The set $\mathcal I_{k,t}$ is then fixed. The root and waiting-state transitions are fresh and policy-independent. For every $i\in\mathcal I_{k,t}$, a visit to $x_i$ at step $t$ occurs with probability at least $1/(8mH)$ by \Cref{eq:persistent-reach-probability}, and action $2$ is then selected with probability one by \Cref{eq:persistent-greedy-tie}. Therefore
\[
S_t(k)-S_t(k-1)
\geq
\frac1{8mH}\overline{\mathbb E}[r_t(k)]
\geq
\frac{c_{\mathrm{cal}}}{8H}.
\]
This argument does not condition on the calibration event.
\end{proof}

\begin{proof}[Proof of Theorem~\ref{thm:persistent-structural-lower-bound}]
Lemma~\ref{lem:persistent-hard-validity} already proves that every member of the family is a legal $d$-dimensional linear MDP with fixed initial state and minimum positive gap exactly $\dmin$.

Fix $t\in\mathcal T$. Suppose, toward a contradiction, that
\[S_t(K)<n_t^{\mathrm{info}}.\]
By monotonicity of $S_t(\cdot)$, the premise of Lemma~\ref{lem:persistent-sampling-rate} then holds for every
$k=K_{\mathrm{burn}}+1,\ldots,K$. Hence
\[
S_t(K)
\geq
\frac{c_{\mathrm{cal}}}{8H}(K-K_{\mathrm{burn}})
\geq
n_t^{\mathrm{info}},
\]
where the last inequality is exactly \Cref{eq:persistent-main-budget}, a contradiction. Thus
\[S_t(K)\geq n_t^{\mathrm{info}} \qquad\forall t\in\mathcal T.\]

Every count in $N_{i,t,2}$ corresponds to a positive gap of at least $\dmin$ by \Cref{eq:persistent-xi-gap}. Using the expected performance-difference identity \Cref{eq:persistent-performance-difference} and averaging uniformly over the finite family,
\begin{align*}
\overline{\mathbb E}[\operatorname{Regret}(T)]
&\geq
\dmin\sum_{t\in\mathcal T}S_t(K)
\geq
\frac{m^2}{2^{19}A_0^2\dmin}
\sum_{t\in\mathcal T}R_t^2.
\end{align*}
For $d\geq15$, $m=d-6\geq d/2$. For $H\geq16$,
\[|\mathcal T| = \left\lfloor\frac H4\right\rfloor-2 \geq \frac H{16}, \qquad R_t\geq\frac{3H}{4}.\]
Since $A_0=9$,
\begin{align*}\overline{\mathbb E}[\operatorname{Regret}(T)] &\geq \frac{9d^2H^3}{2^{29}A_0^2\dmin} = \frac{d^2H^3}{9\cdot2^{29}\dmin} \geq 2^{-33}\frac{d^2H^3}{\dmin}.\end{align*}
Therefore at least one fixed family member satisfies \Cref{eq:persistent-class-final}.

Finally, \Cref{eq:persistent-main-budget-order} follows directly from
\Cref{eq:persistent-burnin-counts}--\Cref{eq:persistent-information-threshold} and the bounds
$m\leq d$ and $R_t\leq H$. In particular,
\[
\frac{8H}{c_{\mathrm{cal}}}\max_{t\in\mathcal T}n_t^{\mathrm{info}}
=
\frac{m^2H}{2^{16}A_0^2c_{\mathrm{cal}}\dmin^2}
\max_{t\in\mathcal T}R_t^2
=
O\!\left(\frac{d^2H^3}{c_{\mathrm{cal}}\dmin^2}\right).
\]
This completes the proof.
\end{proof}

\subsection{Application to LSVI-UCB++}
\label{subsec:persistent-lsvi-ucbpp}

We now verify Definition~\ref{def:persistent-ellipsoid} and the budget condition in Theorem~\ref{thm:persistent-structural-lower-bound} for the modified LSVI-UCB++ in Algorithm~1. We use the confidence guarantees already proved in Lemma~\ref{concentration} and Lemma~\ref{lemma:optimistic}; no new concentration theorem is needed here. We instantiate the deterministic tie-breaking in Algorithm~1 by choosing the smallest indexed maximizing action. This convention does not affect the upper-bound analysis.

The weighted covariance update of Algorithm~1 is
\[\bSigma_{k+1,h} = \bSigma_{k,h} + \bar\sigma_{k,h}^{-2} \bphi(s_h^k,a_h^k)\bphi(s_h^k,a_h^k)^\top,\]
where
\begin{align}\bar\sigma_{k,h}^2 = \max\!\left\{ \sigma_{k,h}^2, H, 2d^3H^2\|\bphi(s_h^k,a_h^k)\|_{\bSigma_{k,h}^{-1}} \right\}. \label{eq:lsviucbpp-class-weight-definition}\end{align}
Thus
\[\alpha_{k,h}=\bar\sigma_{k,h}^{-2}, \qquad \Lambda_{k,h}^{\mathrm{all}}=\bSigma_{k,h}.\]
The quantity $\bar\sigma_{k,h}$ is determined before $s_{h+1}^k$ is observed, so the weights are predictable and remain immutable thereafter.

At a genuine determinant-doubling update, Algorithm~1 recomputes all steps in a single backward pass while keeping the same previous update index fixed throughout that pass. Its optimistic Bellman predictor is
\[\widehat{\mathcal B}_{u,h}(s,a) = r_h(s,a)+\hat\wb_{u,h}^\top\bphi(s,a),\]
which is exactly \Cref{eq:persistent-predictor-outside-reward}, and its current candidate is
\[\min\!\left\{ r_h(s,a)+\hat\wb_{u,h}^\top\bphi(s,a) +\beta\|\bphi(s,a)\|_{\bSigma_{u,h}^{-1}}, H \right\}.\]
Because each genuine update takes the pointwise minimum with the previously maintained optimistic estimate, the maintained $Q$-function is precisely the historical-minimum rule in \Cref{eq:persistent-maintained-Q}.

While an update index $u$ remains in use, $\bSigma_{k,h}\succeq\bSigma_{u,h}$ and the determinant-doubling condition has not yet fired, so
\[\det(\bSigma_{k,h})<2\det(\bSigma_{u,h}).\]
Every eigenvalue of
$\bSigma_{u,h}^{-1/2}\bSigma_{k,h}\bSigma_{u,h}^{-1/2}$
is at least one, while their product is smaller than two. Hence every eigenvalue is smaller than two and
\begin{align}\bSigma_{u,h} \preceq \bSigma_{k,h} \preceq 2\bSigma_{u,h}. \label{eq:lsviucbpp-class-kappa}\end{align}
Therefore \Cref{eq:persistent-covariance-growth} holds with $\kappa=2$.

With the standard choice $\lambda=H^{-2}$, we have
$\bSigma_{k,h}\succeq H^{-2}I_d$, and hence
\[\|\bphi(s_h^k,a_h^k)\|_{\bSigma_{k,h}^{-1}}\leq H.\]
The clipped empirical-variance term in $\sigma_{k,h}^2$ is at most $H^2$, the two terms defining $E_{k,h}$ are each at most $H^2$, $D_{k,h}\leq d^3H^3$, and the additive term is $H$. Thus
\[\sigma_{k,h}^2\leq H^2+2H^2+d^3H^3+H\leq5d^3H^3.\]
The leverage term in \Cref{eq:lsviucbpp-class-weight-definition} is at most $2d^3H^3$, while the outer maximum always gives $\bar\sigma_{k,h}^2\geq H$. Consequently,
\begin{align}H \leq \bar\sigma_{k,h}^2 \leq 5d^3H^3, \qquad \alpha_-:=\frac1{5d^3H^3}, \qquad \alpha_+:=\frac1H. \label{eq:lsviucbpp-class-global-weights}\end{align}
These bounds are pathwise and do not rely on a high-probability event.

It remains to verify modelwise full-domain calibration and optimism. Set
\begin{align*}
\delta_{\mathrm{src}}:=\frac1{19HK}, \qquad \delta_{\mathrm{cal}}:=\frac1{16}, \qquad L_K:=1+\log(1+19dH^4K^2), 
\end{align*}
and choose the radii according to \Cref{radii}, with sufficiently large fixed universal constants. Equivalently, for suitable numerical constants $c_\beta,\bar c_\beta,\tilde c_\beta$,
\begin{align*}
\beta
&=c_\beta\sqrt d\,L_K,
\\
\bar\beta
&=\bar c_\beta\left[\sqrt d+d^{3/2}H\log(19dH^4K^2)\right],
\notag\\
\tilde\beta
&=\tilde c_\beta\left[H\sqrt d+d^{3/2}H^2\log(19dH^4K^2)\right].
\end{align*}
Lemma~\ref{concentration} gives
\[
\mathbb P(\mathcal E_1\cap\mathcal E_2)
\geq
1-15\delta_{\mathrm{src}}
\geq
1-\delta_{\mathrm{cal}}
\]
for $H\geq16$. On $\mathcal E_2$, the Bellman confidence inequality with radius $\beta$ holds simultaneously for every $(s,a,h,k)\in\mathcal S\times\mathcal A\times[H]\times[K]$, including the unreachable probe states. Lemma~\ref{lemma:optimistic} gives global optimism on $\mathcal E_1\cap\mathcal E_2$. Hence, for every fixed linear MDP, we may take
\[\mathcal E_{\mathrm{cal}}^M := \mathcal E_1\cap\mathcal E_2.\]
Together with \Cref{eq:lsviucbpp-class-global-weights} and \Cref{eq:lsviucbpp-class-kappa}, this shows that Algorithm~1 satisfies Definition~\ref{def:persistent-ellipsoid} with parameters
\[
\left(
H^{-2},
\beta,
\frac1{5d^3H^3},
\frac1H,
2,
\frac1{16}
\right).
\]
No simultaneous event over different MDPs is required.

For $\delta_{\mathrm{cal}}=1/16$, the numerical inequality
\[
\chi_0+\frac1{64}+\frac1{16}
=
\frac{131072}{321489}+\frac5{64}
<
\frac{63}{128}
\]
implies
\[
c_{\mathrm{cal}}
>
\frac{127}{256}-\frac{63}{128}
=
\frac1{256}.
\]
Hence the information term in \Cref{eq:persistent-main-budget} is at most
\[
2048H\max_{t\in\mathcal T}n_t^{\mathrm{info}}.
\]

Substituting $\lambda=H^{-2}$, $\kappa=2$, and \Cref{eq:lsviucbpp-class-global-weights} into \Cref{eq:persistent-burnin-counts} gives
\begin{align*}
N_b
&\leq
1+5d^3H+40960d^3H^3\beta^2,
\\
n_s
&\leq
1+\frac{10240d^3H^3\beta^2}{\dmin^2}.
\end{align*}
Together with the information term in \Cref{eq:persistent-information-threshold}, these bounds imply that there is a universal constant $C_0>0$ such that \Cref{eq:persistent-main-budget} holds whenever
\begin{align}K \geq C_0d^3H^3\beta^2\left(1+\frac{H}{\dmin^2}\right). \label{eq:lsviucbpp-class-budget-beta}\end{align}
Since $\beta^2=O(dL_K^2)$ for the fixed confidence-radius constants, and for $x=dHK\geq1$ we have $1+19dH^4K^2\leq20(1+x)^4$, it follows that $L_K\leq C[1+\log(1+dHK)]$ for an absolute constant $C$. Hence there is another universal constant $C_1>0$ such that the sufficient condition
\begin{align*}
K \geq C_1d^4H^3 \left(1+\frac{H}{\dmin^2}\right) \left[1+\log(1+dHK)\right]^2 
\end{align*}
implies \Cref{eq:lsviucbpp-class-budget-beta}. The occurrence of $K$ inside the logarithm is harmless and can be removed explicitly as follows.

\begin{lemma}[Removing the implicit logarithmic budget]
\label{lem:persistent-remove-implicit-K}
For $A,B_1\geq1$,
\[K\geq64A[\log(eAB_1)]^2 \quad\Longrightarrow\quad K\geq A[\log(B_1K)]^2.\]
\end{lemma}
\begin{proof}
Let $s=\log(eAB_1)\geq1$ and $K_\star=64As^2$. Then
\[\log(B_1K_\star) =s-1+\log64+2\log s \leq8s,\]
so $K_\star\geq A[\log(B_1K_\star)]^2$. Moreover, $K/[\log(B_1K)]^2$ is increasing whenever $\log(B_1K)>2$, which holds for every $K\geq K_\star$.
\end{proof}

Let
\[A := C_2d^4H^3\left(1+\frac{H}{\dmin^2}\right), \qquad B_1:=2edH,\]
where $C_2$ is a sufficiently large universal constant absorbing the fixed confidence-radius constants and the comparison between $L_K$ and $1+\log(1+dHK)$. Then
\begin{align}
K \geq 64A[\log(eAB_1)]^2 
\label{eq:lsviucbpp-class-budget-explicit}
\end{align}
is an explicit polynomial-logarithmic sufficient condition for \Cref{eq:persistent-main-budget}. Therefore Theorem~\ref{thm:persistent-structural-lower-bound} gives the following direct specialization.

\begin{corollary}[Persistent-ellipsoidal lower bound for LSVI-UCB++]
\label{cor:persistent-lsviucbpp-lower}
Let $d\geq15$, $H\geq16$, and $\dmin\in(0,H/16]$. Run Algorithm~1 with $\lambda=H^{-2}$, the confidence radii in \Cref{radii} with $\delta=1/(19HK)$, sufficiently large fixed universal constants, and smallest-index tie-breaking. If $K$ satisfies \Cref{eq:lsviucbpp-class-budget-explicit}, then there exists a $d$-dimensional linear MDP with fixed initial state and minimum positive gap exactly $\dmin$ such that
\begin{align*}
\mathbb E[\operatorname{Regret}(T)] \geq 2^{-33}\frac{d^2H^3}{\dmin}.
\end{align*}
\end{corollary}
The lower bound in Corollary~\ref{cor:persistent-lsviucbpp-lower} matches the $d^2H^3/\dmin$ dependence of the leading term in \Cref{regret} up to logarithmic factors.

\section{Proof of \texorpdfstring{\Cref{lower}}{Theorem 3.3}}
\label{lowerproof}
For fixed integers $d,H\geq1$, let $\mathfrak M_{d,H}$ denote the class of horizon-$H$ linear MDPs with ambient feature dimension $d$, deterministic known rewards in $[0,1]$, and the standard linear-MDP norm bounds.

For probability measures $P$ and $Q$ with $P\ll Q$, define
\[\operatorname{KL}(P\|Q):=\int \log\!\left(\frac{\mathrm dP}{\mathrm dQ}\right)\mathrm dP.\]
For $p,q\in[0,1]$, define
\[\operatorname{kl}(p,q):=p\log\!\left(\frac pq\right)+(1-p)\log\!\left(\frac{1-p}{1-q}\right),\]
with the usual conventions at the boundary of $[0,1]$.

\begin{proof}
Fix any possibly randomized learning algorithm $\mathsf{Alg}$ and parameters $d$, $H$, $\dmin$, and $K$ satisfying the conditions of \Cref{lower}.
Let
\[d':=\left\lfloor\frac{d-3}{2}\right\rfloor,\qquad H':=\left\lfloor\frac H2\right\rfloor.\]
We construct a family of linear MDPs indexed by
\[\Sigma:=\{-1,+1\}^{[d']\times\{2,\ldots,H'\}}.\]
For each $\sigma\in\Sigma$, let $\mathbb Q_\sigma^K$ denote the probability law of the complete interaction history through episode $K$ generated by $\mathsf{Alg}$ on the corresponding MDP $\mathcal M_\sigma$, and let $\mathbb E_\sigma$ denote the corresponding expectation.

\textbf{Construction of the hard-instance family.}
Consider the state space
\[\mathcal S:=\{w,g,b,x_1,\ldots,x_{d'}\}\]
and the action space $\{1,2\}$. The initial state is $w$ in every episode. The state $w$ is a waiting state, $x_1,\ldots,x_{d'}$ are decision states, and $g$ and $b$ are absorbing states. For every $t\in[H]$ and $a\in\{1,2\}$, define the deterministic known rewards by
\[r_t(g,a)=1,\qquad r_t(s,a)=0\quad\text{for every }s\neq g.\]
Thus, all rewards are deterministic and belong to $[0,1]$.

For every $t\in[H'-1]$ and $a\in\{1,2\}$, let
\[P_t^\sigma(x_i\mid w,a)=\frac{1}{d'H},\quad i\in[d'],\qquad P_t^\sigma(w\mid w,a)=1-\frac1H.\]
For every $t\in\{H',\ldots,H\}$ and $a\in\{1,2\}$, let $P_t^\sigma(b\mid w,a)=1$. Moreover, for every $t\in[H]$ and $a\in\{1,2\}$, let
\[P_t^\sigma(g\mid g,a)=1,\qquad P_t^\sigma(b\mid b,a)=1.\]
For every $i\in[d']$ and $t\in\{2,\ldots,H'\}$, define
\[P_t^\sigma(g\mid x_i,1)=\frac12,\qquad P_t^\sigma(b\mid x_i,1)=\frac12,\]
and
\[P_t^\sigma(g\mid x_i,2)=\frac12+\sigma_{i,t}\frac{\dmin}{H-t},\qquad P_t^\sigma(b\mid x_i,2)=\frac12-\sigma_{i,t}\frac{\dmin}{H-t}.\]
At every other step, both actions at every $x_i$ transition deterministically to $b$. Since $t\leq H'$ and $\dmin\leq H/4$, we have $\dmin/(H-t)\leq1/2$, so all transition probabilities are valid.

If the process enters $g$ after the transition at step $t$, it receives reward $1$ at steps $t+1,\ldots,H$. Hence, for every $i\in[d']$ and $t\in\{2,\ldots,H'\}$,
\[Q_t^\star(x_i,1)=\frac{H-t}{2},\qquad Q_t^\star(x_i,2)=\frac{H-t}{2}+\sigma_{i,t}\dmin.\]
Therefore, when $\sigma_{i,t}=-1$, action $1$ is optimal and action $2$ has gap $\dmin$, while when $\sigma_{i,t}=+1$, action $2$ is optimal and action $1$ has gap $\dmin$. All other state-action-step triples have zero gap. Thus, every $\mathcal M_\sigma$ has minimum positive gap $\dmin$.

\textbf{Verification of the linear-MDP structure.}
Let $\{e_j\}_{j=1}^d$ be the standard basis of $\mathbb R^d$ and define
\[\bphi(x_i,1)=e_{2i-1},\qquad \bphi(x_i,2)=e_{2i},\qquad i\in[d'],\]
and
\[\bphi(w,a)=e_{2d'+1},\qquad \bphi(g,a)=e_{2d'+2},\qquad \bphi(b,a)=e_{2d'+3},\qquad a\in\{1,2\}.\]
Since $2d'+3\leq d$, this defines a feature map in $\mathbb R^d$. More explicitly, $2d'+3=d$ when $d$ is odd and $2d'+3=d-1$ when $d$ is even; in the latter case, the last coordinate is set to zero in the feature map and all linear parameters. Clearly, $\|\bphi(s,a)\|_2=1$.

For each step $t$, define the coordinates of the vector-valued signed measure $\btheta_t^\sigma$ by
\[(\theta_t^\sigma)^{(2i-1)}(\cdot)=P_t^\sigma(\cdot\mid x_i,1),\quad (\theta_t^\sigma)^{(2i)}(\cdot)=P_t^\sigma(\cdot\mid x_i,2),\qquad i\in[d'],\]
\[(\theta_t^\sigma)^{(2d'+1)}(\cdot)=P_t^\sigma(\cdot\mid w,1),\quad (\theta_t^\sigma)^{(2d'+2)}(\cdot)=P_t^\sigma(\cdot\mid g,1),\qquad (\theta_t^\sigma)^{(2d'+3)}(\cdot)=P_t^\sigma(\cdot\mid b,1),\]
and set every unused coordinate to the zero measure. Since the two actions have identical transitions at $w$, $g$, and $b$,
\[P_t^\sigma(\cdot\mid s,a)=\langle\bphi(s,a),\btheta_t^\sigma(\cdot)\rangle\]
for every $(s,a,t)$. Moreover, for every measurable $B\subseteq\mathcal S$, each nonzero coordinate of $\theta_t^\sigma(B)$ lies in $[0,1]$, and there are at most $2d'+3\le d$ such coordinates. Hence $\|\theta_t^\sigma(B)\|_2^2\le 2d'+3\le d$.

Therefore, $\mathcal M_\sigma\in\mathfrak M_{d,H}$ for every $\sigma\in\Sigma$.

\textbf{Reach probabilities and regret decomposition.}
For every $i\in[d']$ and $t\in\{2,\ldots,H'\}$, because both actions have identical transition kernels at $w$, the probability of reaching $x_i$ at step $t$ is independent of the algorithm and of $\sigma$, and equals
\[p_{i,t}:=\left(1-\frac1H\right)^{t-2}\frac{1}{d'H}.\]
Since $t\leq H'$, Bernoulli's inequality gives
\[\left(1-\frac1H\right)^{t-2}\geq1-\frac{t-2}{H}\geq\frac12.\]
Since $d'\leq d/2$, it follows that
\[p_{i,t}\geq\frac{1}{2d'H}\geq\frac{1}{dH}.\]

For $a\in\{1,2\}$, define
\[N_{i,t,a}(K):=\sum_{k=1}^K\mathbb I\{s_t^k=x_i,\ a_t^k=a\},\qquad N_{i,t}(K):=N_{i,t,1}(K)+N_{i,t,2}(K).\]
The transition draws in different episodes are independent, and the probability of reaching $x_i$ at step $t$ is always $p_{i,t}$. Therefore,
\[N_{i,t}(K)\sim\operatorname{Binomial}(K,p_{i,t})\]
under every $\mathcal M_\sigma$.

By the expected performance-difference identity and the gap structure above,
\[\mathbb E_\sigma[\operatorname{Regret}(T)]=\dmin\sum_{i=1}^{d'}\sum_{t=2}^{H'}\left(\mathbb I\{\sigma_{i,t}=-1\}\mathbb E_\sigma[N_{i,t,2}(K)]+\mathbb I\{\sigma_{i,t}=+1\}\mathbb E_\sigma[N_{i,t,1}(K)]\right).\]

\textbf{A two-instance lower bound.}
Fix any $i\in[d']$ and $t\in\{2,\ldots, \lfloor H/4\rfloor\}$, and fix all entries of $\sigma$ other than $\sigma_{i,t}$. Let $\sigma^-$ and $\sigma^+$ denote the two resulting sign vectors obtained by setting $\sigma_{i,t}=-1$ and $\sigma_{i,t}=+1$, respectively. The MDPs $\mathcal M_{\sigma^-}$ and $\mathcal M_{\sigma^+}$ agree everywhere except at $(x_i,2,t)$, where
\[P_t^{\sigma^-}(g\mid x_i,2)=\frac12-\frac{\dmin}{H-t},\qquad P_t^{\sigma^+}(g\mid x_i,2)=\frac12+\frac{\dmin}{H-t}.\]

We claim that
\[\max\left\{\mathbb E_{\sigma^-}[N_{i,t,2}(K)],\mathbb E_{\sigma^+}[N_{i,t,1}(K)]\right\}\geq\frac{(H-t)^2}{256\dmin^2}\log K.\]
Suppose, toward a contradiction, that
\[\mathbb E_{\sigma^-}[N_{i,t,2}(K)]<\frac{(H-t)^2}{256\dmin^2}\log K,\qquad \mathbb E_{\sigma^+}[N_{i,t,1}(K)]<\frac{(H-t)^2}{256\dmin^2}\log K.\]

Define
\[\mathcal L_K:=\left\{N_{i,t,2}(K)\leq\frac{p_{i,t}K}{2}\right\}.\]
Under $\mathcal M_{\sigma^-}$, Markov's inequality gives
\[\mathbb Q_{\sigma^-}^K(\mathcal L_K^c)\leq\frac{2\mathbb E_{\sigma^-}[N_{i,t,2}(K)]}{p_{i,t}K}<\frac{(H-t)^2\log K}{128p_{i,t}K\dmin^2}\leq\frac{dH^3\log K}{128K\dmin^2}.\]
The assumed regime $K\geq(dH^3/\dmin^2)^2$ implies
\[\frac{dH^3}{\dmin^2}\leq\sqrt K.\]
Moreover, $\dmin\leq H/4$, $d\geq7$, and $H\geq8$ imply $K \geq (dH^3/\dmin^2)^2\geq 896^2 > 2^{16}$, so the assumed regime in particular ensures that $K$ is large enough that $\log K\leq K^{1/4}$. Hence,
\[\mathbb Q_{\sigma^-}^K(\mathcal L_K^c)\leq\frac{1}{128K^{1/4}}<\frac14,\]
and therefore
\[\mathbb Q_{\sigma^-}^K(\mathcal L_K)\geq\frac34.\]

Under $\mathcal M_{\sigma^+}$, if $\mathcal L_K$ occurs and $N_{i,t}(K)>3p_{i,t}K/4$, then
\[N_{i,t,1}(K)=N_{i,t}(K)-N_{i,t,2}(K)>\frac{p_{i,t}K}{4}.\]
Thus, by a multiplicative Chernoff bound and Markov's inequality,
\begin{align*}
\mathbb Q_{\sigma^+}^K(\mathcal L_K)
&\leq \mathbb Q_{\sigma^+}^K\left(N_{i,t}(K)\leq\frac{3p_{i,t}K}{4}\right)+\mathbb Q_{\sigma^+}^K\left(N_{i,t,1}(K)>\frac{p_{i,t}K}{4}\right)\\
&\leq\exp\left(-\frac{p_{i,t}K}{32}\right)+\frac{4\mathbb E_{\sigma^+}[N_{i,t,1}(K)]}{p_{i,t}K}\\
&<\exp\left(-\frac{K}{32dH}\right)+\frac{dH^3\log K}{64K\dmin^2}.
\end{align*}
Using again $dH^3/\dmin^2\leq\sqrt K$ and $\log K\leq K^{1/4}$,
\[\frac{dH^3\log K}{64K\dmin^2}\leq\frac{1}{64K^{1/4}}.\]
Furthermore, since $dH^3/\dmin^2\geq16dH$ and $\sqrt K\geq dH^3/\dmin^2$,
\[\frac{K}{32dH}\geq\frac{\sqrt K}{2}.\]
For every $K\geq 896^2$, $\frac{\sqrt K}{2}\geq\frac14\log K+\log64$, and hence
\[\exp\left(-\frac{K}{32dH}\right)\leq\frac{1}{64K^{1/4}}.\]
Consequently,
\[\mathbb Q_{\sigma^+}^K(\mathcal L_K)\leq\frac{1}{32K^{1/4}}\leq K^{-1/4}.\]

Although the histories and selected actions may have different distributions under the two MDPs, given any fixed history before an action is selected, the algorithm uses the same conditional action-selection kernel under both MDPs. Hence, the action-selection factors cancel when the likelihood ratio of the same realized history is formed. Since the initial state and reward functions are also identical and the two transition kernels differ only at $(x_i,2,t)$, the chain rule for relative entropy gives
\[\operatorname{KL}(\mathbb Q_{\sigma^-}^K\|\mathbb Q_{\sigma^+}^K)=\mathbb E_{\sigma^-}[N_{i,t,2}(K)]\operatorname{kl}\left(\frac12-\frac{\dmin}{H-t},\frac12+\frac{\dmin}{H-t}\right).\]

Since $\mathcal L_K$ is determined by the complete interaction history, the data-processing inequality gives
\[\operatorname{KL}(\mathbb Q_{\sigma^-}^K\|\mathbb Q_{\sigma^+}^K)\geq\operatorname{kl}\left(\mathbb Q_{\sigma^-}^K(\mathcal L_K),\mathbb Q_{\sigma^+}^K(\mathcal L_K)\right).\]
For $p,q\in[0,1]$,
\[\operatorname{kl}(p,q)\geq p\log\left(\frac1q\right)-\log2.\]
Using $\mathbb Q_{\sigma^-}^K(\mathcal L_K)\geq3/4$ and $\mathbb Q_{\sigma^+}^K(\mathcal L_K)\leq K^{-1/4}$, we obtain
\[\operatorname{KL}(\mathbb Q_{\sigma^-}^K\|\mathbb Q_{\sigma^+}^K)\geq\frac{3}{16}\log K-\log2.\]

On the other hand, using $\operatorname{kl}(p,q)\leq(p-q)^2/[q(1-q)]$ and $\dmin/(H-t)\leq1/3$,
\[\operatorname{kl}\left(\frac12-\frac{\dmin}{H-t},\frac12+\frac{\dmin}{H-t}\right)\leq\frac{32\dmin^2}{(H-t)^2}.\]
Together with the assumed upper bound on $\mathbb E_{\sigma^-}[N_{i,t,2}(K)]$, this gives
\[\operatorname{KL}(\mathbb Q_{\sigma^-}^K\|\mathbb Q_{\sigma^+}^K)<\frac18\log K.\]
Since the assumed regime implies $K>2^{16}$, we have
\[\frac{3}{16}\log K-\log2>\frac18\log K,\]
which is a contradiction. Therefore,
\[\max\left\{\mathbb E_{\sigma^-}[N_{i,t,2}(K)],\mathbb E_{\sigma^+}[N_{i,t,1}(K)]\right\}\geq\frac{(H-t)^2}{256\dmin^2}\log K.\]

\textbf{Averaging over the hard-instance family.}
Fix $i\in[d']$ and $t\in\{2,\ldots,\lfloor H/4\rfloor\}$. The set $\Sigma$ can be partitioned into exactly
\[2^{d'(H'-1)-1}\]
disjoint pairs $(\sigma^-,\sigma^+)$ which agree on every coordinate except $(i,t)$ and satisfy $\sigma^-_{i,t}=-1$ and $\sigma^+_{i,t}=+1$. For every such pair, the preceding two-instance bound implies
\[\mathbb E_{\sigma^-}[N_{i,t,2}(K)]+\mathbb E_{\sigma^+}[N_{i,t,1}(K)]\geq\frac{(H-t)^2}{256\dmin^2}\log K.\]
Therefore,
\begin{align*}
&\frac{1}{|\Sigma|}\sum_{\sigma\in\Sigma}\left(\mathbb I\{\sigma_{i,t}=-1\}\mathbb E_\sigma[N_{i,t,2}(K)]+\mathbb I\{\sigma_{i,t}=+1\}\mathbb E_\sigma[N_{i,t,1}(K)]\right)\\
&\qquad\geq\frac{2^{d'(H'-1)-1}}{2^{d'(H'-1)}}\frac{(H-t)^2}{256\dmin^2}\log K
=\frac{(H-t)^2}{512\dmin^2}\log K.
\end{align*}

Averaging the regret decomposition over all $\sigma\in\Sigma$ and then applying the preceding bound to every $(i,t)$ yields
\begin{align*}
&\frac{1}{|\Sigma|}\sum_{\sigma\in\Sigma}\mathbb E_\sigma[\operatorname{Regret}(T)]\\
&=\dmin\sum_{i=1}^{d'}\sum_{t=2}^{H'}\frac{1}{|\Sigma|}\sum_{\sigma\in\Sigma}\left(\mathbb I\{\sigma_{i,t}=-1\}\mathbb E_\sigma[N_{i,t,2}(K)]+\mathbb I\{\sigma_{i,t}=+1\}\mathbb E_\sigma[N_{i,t,1}(K)]\right)\\
&\geq\dmin\sum_{i=1}^{d'}\sum_{t=2}^{\lfloor H/4\rfloor}\frac{1}{|\Sigma|}\sum_{\sigma\in\Sigma}\left(\mathbb I\{\sigma_{i,t}=-1\}\mathbb E_\sigma[N_{i,t,2}(K)]+\mathbb I\{\sigma_{i,t}=+1\}\mathbb E_\sigma[N_{i,t,1}(K)]\right)\\
&\geq\frac{\log K}{512\dmin}\sum_{i=1}^{d'}\sum_{t=2}^{\lfloor H/4\rfloor}(H-t)^2.
\end{align*}
Since $t\leq \lfloor H/4\rfloor$,
\[(H-t)^2\geq\frac{9H^2}{16}.\]
Therefore,
\[\frac{1}{|\Sigma|}\sum_{\sigma\in\Sigma}\mathbb E_\sigma[\operatorname{Regret}(T)]\geq\frac{\log K}{512\dmin}\,d'(\lfloor \frac H4\rfloor-1)\frac{9H^2}{16}.\]
For $d\geq7$ and $H\geq8$,
\[d'=\left\lfloor\frac{d-3}{2}\right\rfloor\geq\frac d4,\qquad \left\lfloor\frac H4\right\rfloor-1\geq\frac H{11}.\]
Hence,
\[\frac{1}{|\Sigma|}\sum_{\sigma\in\Sigma}\mathbb E_\sigma[\operatorname{Regret}(T)]\geq\frac{1}{40050}\frac{dH^3}{\dmin}\log K.\]
Since the maximum is at least the average,
\[\max_{\sigma\in\Sigma}\mathbb E_\sigma[\operatorname{Regret}(T)]\geq\frac{1}{40050}\frac{dH^3}{\dmin}\log K.\]
Thus, there exists $\sigma^\star\in\Sigma$ such that
\[
\mathbb E_{\sigma^\star}[\operatorname{Regret}(T)]
\geq\frac{1}{40050}\frac{dH^3}{\dmin}\log K.
\]
Since $\mathcal M_{\sigma^\star}\in\mathfrak M_{d,H}$ and has minimum positive gap $\dmin$, this completes the proof of \Cref{lower} with $c=1/40050$.
\end{proof}
\end{document}

%% file: math_commands.tex
\usepackage{amsmath,amsfonts,bm}

\def\eqref#1{equation~\ref{#1}}

\def\1{\bm{1}}

\DeclareMathAlphabet{\mathsfit}{\encodingdefault}{\sfdefault}{m}{sl}
\SetMathAlphabet{\mathsfit}{bold}{\encodingdefault}{\sfdefault}{bx}{n}

\newcommand{\R}{\mathbb{R}}

\DeclareMathOperator*{\argmax}{arg\,max}
\DeclareMathOperator*{\argmin}{arg\,min}

\let\ab\allowbreak